\documentclass[conference,compsoc]{IEEEtran}
\usepackage[nocompress]{cite}
\usepackage{amsmath}
\usepackage{algorithm}
\usepackage{algorithmic}
\usepackage{array}
\usepackage{url}
\usepackage{hyperref}
\usepackage{pifont}
\usepackage{amsfonts}
\usepackage{amsmath}
\usepackage{amsthm}
\usepackage{subfigure}
\usepackage{graphicx}
\usepackage{xcolor}
\usepackage{caption}
\usepackage{multirow}
\usepackage[normalem]{ulem}
\useunder{\uline}{\ul}{}
\definecolor{celestialblue}{rgb}{0.29, 0.59, 0.82}
\definecolor{cerulean}{rgb}{0.0, 0.48, 0.65}
\definecolor{cadmiumorange}{rgb}{0.93, 0.53, 0.18}
\DeclareMathOperator*{\argmax}{arg\,max}

\newtheorem{thm}{Theorem}

\newtheorem{corollary}{Corollary}

\usepackage{tikz}
\newcommand*\circled[1]{\tikz[baseline=(char.base)]{
            \node[shape=circle,draw,inner sep=0pt] (char) {#1};}}

\begin{document}

\title{Node-aware Bi-smoothing: \\Certified Robustness against Graph Injection Attacks}

\author{\IEEEauthorblockN{Yuni Lai\IEEEauthorrefmark{1}, Yulin Zhu\IEEEauthorrefmark{1}, Bailin Pan\IEEEauthorrefmark{2}, Kai Zhou\IEEEauthorrefmark{1}}
	\IEEEauthorblockA{\IEEEauthorrefmark{1} Department of Computing, The Hong Kong Polytechnic University}
	\IEEEauthorblockA{\IEEEauthorrefmark{2} Department of Applied Mathematics, The Hong Kong Polytechnic University}
	\IEEEauthorblockA{csylai@comp.polyu.edu.hk, yulinzhu@polyu.edu.hk, 23061138g@connect.polyu.hk, kaizhou@polyu.edu.hk}
}


\maketitle

\thispagestyle{plain}
\pagestyle{plain}
\begin{abstract} 

Deep Graph Learning (DGL) has emerged as a crucial technique across various domains. However, recent studies have exposed vulnerabilities in DGL models, such as susceptibility to evasion and poisoning attacks. While empirical and provable robustness techniques have been developed to defend against graph modification attacks (GMAs), the problem of certified robustness against graph injection attacks (GIAs) remains largely unexplored. To bridge this gap, we introduce the \textit{node-aware bi-smoothing} framework, which is the \textit{first} certifiably robust approach for general node classification tasks against GIAs. Notably, the proposed node-aware bi-smoothing scheme is model-agnostic and is applicable for both evasion and poisoning attacks. Through rigorous theoretical analysis, we establish the certifiable conditions of our smoothing scheme. We also explore the practical implications of our node-aware bi-smoothing schemes in two contexts: as an empirical defense approach against real-world GIAs and in the context of recommendation systems. Furthermore, we extend two state-of-the-art certified robustness frameworks to address node injection attacks and compare our approach against them. Extensive evaluations demonstrate the effectiveness of our proposed certificates.\footnote{https://anonymous.4open.science/r/NodeAwareSmoothing-35E6/}

\end{abstract}

\section{Introduction}
\label{Sec:Intro}

Deep Graph Learning (DGL), particularly Graph Neural Networks (GNNs), has established itself as the dominant approach for graph learning tasks. DGL has consistently demonstrated outstanding performance across various applications, such as recommender systems, community detection, link prediction in social networks, network intrusion detection, and anomaly detection in financial networks~\cite{zhang2019graph}. 
Many of these applications are critical for ensuring system security, such as node classification in anomaly detection, which helps prevent money laundering~\cite{li2020flowscope} and financial fraud~\cite{ma2021comprehensive}. Consequently, ensuring the trustworthiness of those DGL models is of paramount importance.


Indeed, extensive research has been dedicated to studying the adversarial robustness of DGL against attacks. Specifically, various graph adversarial attacks~\cite{zugner2018adversarial,zugner2018adversarial1,liu2022towards} have been proposed to assess the vulnerability of DGL models. In response, different defense mechanisms are explored, resulting in robust DGL models such as Pro-GNN~\cite{jin2020graph}, RobustGCN\cite{zhu2019robust}, and GCNGuard~\cite{zhang2020gnnguard}. However, despite the effectiveness of defense models, their robustness is often compromised by the relentless development of new attack techniques~\cite{mujkanovic2022defenses}. The ongoing research on attacks and defense for DGL has resulted in a highly competitive status.


To end such an arms race between attack and defense,
there is a growing interest in developing provable defense methods that offer certified robustness \cite{li2023sok}. 
Specifically, certifiably robust models \cite{bojchevski2020efficient,wang2021certified,jia2020certified,schuchardt2023localized} can provide the theoretical guarantee that their predictions would stay unchanged as long as the amount of input perturbations is within a certain range. For instance, a smoothed GNN-based classifier~\cite{bojchevski2020efficient} can achieve a \textit{certified accuracy} of $60\%$ (meaning that $60\%$ of the test nodes are \textit{guaranteed} to be correctly classified)  when faced with an attack involving the \textit{arbitrary} deletion of up to $10$ edges from the graph.
Overall, certified robustness can significantly enhance the trustworthiness of DGL models in deployment.

Despite the significant progress in achieving certified robustness of models in computer vision~\cite{li2022double,fischer2021scalable,jia2022almost,levine2021deep,levine2020robustness,scholten2023hierarchical} and graph learning~\cite{bojchevski2020efficient,wang2021certified,jia2020certified,schuchardt2023localized}, \textit{there is a notable gap regarding the certified robustness of DGL against a novel and significant form of attack known as the Graph Injection Attack (GIA)}. Unlike the commonly investigated Graph Manipulation Attack (GMA), which allows the attacker to modify the existing structure of the graph, GIA involves the injection of new nodes (along with associated edges) into the original graph. \textit{Exploring the certified robustness of DGL against GIA is of great importance for several reasons.} 

\begin{figure}[t]
    \centering
    \includegraphics[width=0.48\textwidth,height=2.2cm]{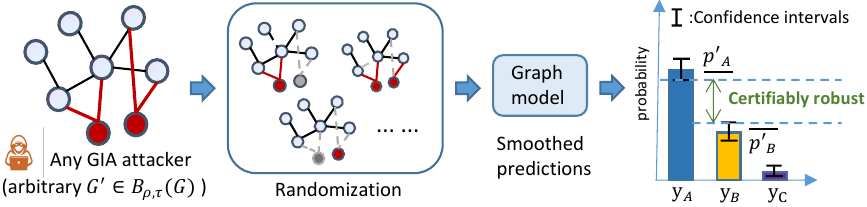}
    \caption{Certified Robustness via Node-aware Smoothing.}
    \label{fig:RandomizedSmoothing_framwork}
\end{figure}

Firstly, unlike in GMA where the attacker requires control over the entire graph, in GIA, the attacker only needs to control the newly injected malicious nodes. Consequently, GIA presents a less demanding threat model, making it a more  realistic threat. Secondly, it is shown that GIA is a more powerful and stealthy attack compared to GMA~\cite{sun2020non,zou2021tdgia,tao2021single,chen2022understanding,tao2023adversarial,ju2023let}. Notably, black-box GIAs such as $G^2A2C$~\cite{ju2023let} have successfully doubled the misclassification rate with only one node injected and one edge inserted.
Lastly, GIA is a type of attack that is particularly prevalent in recommender systems~\cite{zhang2021reverse,huang2021data,fan2023adversarial,guo2023targeted}, where adversaries can easily create fake accounts to engage with items, deliberately damaging recommendations intended for genuine users.
\textit{The practical deployment, high stealthiness, and power of GIA} underscore the urgent need to investigate the provable robustness of DGL under such attacks.

\textbf{Is adapting from existing methods sufficient?}
Our initial attempt is to adapt two existing and prevalent certifying frameworks to tackle the task of certifying DGL against GIA. We emphasize that these two certifying schemes have achieved state-of-the-art performance in their respective tasks, which include image classification and node classification in graphs.
Specifically:
\begin{itemize}
    \item[1)] Bagging-based certifying scheme~\cite{jia2021intrinsic}. Such a scheme is designed to certify image classifiers against sample insertion or deletion attacks. We can extend it to our task by regarding each node as an independent sample without accounting for the graph structure. 
    \item[2)] Randomized-smoothing certifying scheme~\cite{bojchevski2020efficient}. It is designed for DGL models against GMA. To extend it for GIA, we can pre-inject several isolated nodes in the clean graph and then certify how many edges can be injected from these nodes.
\end{itemize}

Both schemes operate by adding carefully crafted random noise to an input graph, resulting in a collection of randomized graphs, which are subsequently classified by a base classifier. The final classification result for the graph is then determined by a \textit{majority vote} among the classifications obtained from the randomized graphs. Leveraging the Neyman-Pearson lemma~\cite{Neyman1992}, these schemes offer verifiable classification margins when dealing with perturbed data. This is achieved by assessing the \textit{probability of overlap}, or likelihood ratio, between the randomized graphs originating from the clean graph and the perturbed graph. The underlying intuition is that after the randomization, if the perturbed graph and clean graph are identical, they should yield the same prediction. Consequently, a higher overlap probability corresponds to a wider certifiable radius, indicating a greater tolerance for perturbation levels.

Nevertheless, the above two adapted approaches have their own limitations, leading to poor certification performances (as shown in Table.~\ref{tab:poison_CA} and Figure.~\ref{fig:cer_poison}). The adapted bagging-based certifying scheme \textit{did not properly define the perturbation space} (limitation \circled{1}), resulting in an \textit{exaggerated} threat model where the attacker is unnecessarily too strong.  In particular, the scheme completely neglects the graph structure information, resulting in no constraints on the number of connected edges for each injected node (in practice, this will make GIA easily detectable). Consequently, the certification performance against GIA is significantly compromised. 
Although the randomized-smoothing certifying scheme can explicitly consider graph structure and restrict the added number of edges per node, \textit{it suffers from an extremely low probability of random sample overlap under GIAs} (limitation \circled{2}), resulting in an inadequate certification performance.
These limitations underscore the necessity and the challenges of developing novel certifying approaches that can effectively leverage the graph structure while increasing the overlap probability in order to provide a more effective certificate against GIA.


\textbf{Our solutions.}  
We propose a novel \textbf{node-aware} \textbf{bi-smoothing} scheme to explicitly address the above limitations. Specifically, to address limitation \circled{1},  we fully consider the practical constraint for GIA that each injected node can only connect to a few edges to ensure attack unnoticeability, leading to more accurate perturbation space and improved certification performances.
Our solution to this is a nontrivial generalization of the sparsity-aware certificate~\cite{bojchevski2020efficient} to certify against node injection perturbation.
Furthermore, to increase the sample overlap probability under GIA (limitation \circled{2}), our bi-smoothing scheme will \textit{randomly deletes nodes and edges simultaneously}. 
More specifically, we show that increasing the probability of deleting all inserted edges is essential for improving the overlap probability. 
Considering that the potential perturbed edges are concentrated around the injected nodes, node deletion enables a significantly higher probability of removing all perturbed edges originating from the injected nodes. 
Overall, by introducing node-aware bi-smoothing, we can model a more realistic and restricted attacker, and increase the chance of deleting all the perturbed edges from an injected node.

We offer a rigorous theoretical analysis to establish the validity of our robustness certificate. Additionally, we show the versatility of our framework by demonstrating its effectiveness not only against evasion attacks but also against poisoning attacks. Nevertheless, to enhance the certification performance specifically for poisoning attacks, we introduce a variant called \textbf{node-aware-exclude}. This variant excludes isolated nodes from the prediction process after randomization, thereby improving the overall prediction quality.

Our comprehensive evaluation shows that the proposed node-aware bi-smoothing framework can significantly improve the certification performances of the baselines (i.e., direct adaptations). For instance, in certain cases, our scheme has shown improvements of $760\%$ and $530\%$ in terms of the average certifiable radius (ACR). When arbitrarily injecting $10$ nodes with a maximum of $5$ edges per node, our node-aware and node-aware-exclude 
approaches achieve certified accuracies of $35\%$ and $55\%$ respectively. In contrast, the two direct adapted baselines yield $0\%$ certified accuracy.

\textbf{Practical Implications.} We further investigate the \textit{practical implications} of our proposed node-aware bi-smoothing schemes from two perspectives. Firstly, we explore the application of node-aware bi-smoothing schemes as an empirical defense approach to protect against an actual GIA. We compare our smoothed classifier with other state-of-the-art robust GNN models. Remarkably, the experimental results demonstrate that our model not only achieves competitive empirical accuracy but also provides certified accuracy, which is a distinctive advantage over other empirically robust models.
Secondly, we examine the potential of applying the smoothing schemes to recommendation systems, where GIAs are commonly encountered. Specifically, we treat the recommender system as a multi-label node classification task, where it predicts $K$ items for each user (node).    
To assess the effectiveness of our model, we evaluate the certified number of overlap items between the predicted and the ground truth items. Notably, our model demonstrates superior certified performance compared to the baseline method~\cite{jia2023pore} specifically designed for recommender systems.

We summarize the main contributions as follows:
\begin{itemize}
    \item We address the challenging task of achieving certified robustness against the \textit{graph injection attack}, which is a highly powerful and stealthy form of attack compared to the graph manipulation attack. Our primary technical advancement is a novel \textit{node-aware bi-smoothing} scheme, which is essential to achieve enhanced certified robustness.
    \item Our node-aware bi-smoothing scheme is highly versatile. It is model-agnostic and is applicable to both evasion and poisoning attacks. Furthermore, with minimum modification, our scheme can also provide certification for recommender systems, where graph injection attacks are commonly observed.
    \item 
    In addition, we demonstrate that our node-aware bi-smoothing scheme can be used as a practical defense strategy. Notably, our defense method achieves comparable empirical robustness to state-of-the-art robust models under actual graph injection attacks, while offering theoretical robustness guarantees.
    \item We conduct extensive experiments to validate the effectiveness of our schemes. The results show that our schemes can significantly improve the certified robustness against graph injection attacks compared to strong baseline methods.
\end{itemize}

\textbf{Organization.} We provide the preliminaries in Section~\ref{Sec:Backg}, and state our problem in Section~\ref{Sec:Problem}. In Section~\ref{Sec:Method}, we propose our smoothing scheme and the theoretical guarantees. Then, we illustrate the practical implementation in Section~\ref{Sec:Implem}, and the experimental results in Section~\ref{Sec:Results}. The related work is presented in Section~\ref{Sec:RelatedW}. Finally, we discuss the limitation in Section~\ref{Sec:Limit} and conclude in Section~\ref{Sec:Conc}.

\section{Background}
\label{Sec:Backg}
Our primary goal is to develop a certifiably robust classifier for the node classification task against graph injection attacks. In the following, we introduce the necessary preliminaries and background to tackle this task.

\subsection{Node Classification in Graphs}
A graph with $n$ nodes is represented as $G=(\mathcal{V},\mathcal{E},X)\in \mathbb{G}$, where $\mathcal{V}=\{v_1, \cdots, v_n\}$ is the set of nodes, $\mathcal{E}=\{e_{ij}=(v_i,v_j)\}$ is the set of edges with each edge $e_{ij}$ linking $v_i$ and $v_j$, and $X\in \mathbb{R}^{n\times d}$ are node features with dimension $d$. The graph structure of $G$ can also be encoded by adjacency matrix $A\in \{0,1\}^{n\times n}$ with $A_{ij} = 1$ if $e_{ij} \in \mathcal{E}$ and $A_{ij} = 0$ if $e_{ij} \notin \mathcal{E}$. Some of the nodes are associated with a label $y\in\mathcal{Y}=\{1,\cdots, C\}$. 
The task of node classification is to predict the missing node labels. To this end, a graph-based classifier $f:\mathbb{G}\rightarrow \{1,\cdots, C\}^n$ takes graph $G$ as input and is used to predict the labels. 
We consider both the inductive and transductive settings. Specifically, in the transductive setting, the model trained on graph $G$ can only make predictions for the current nodes in $G$, while the inductive classifier can make predictions for graphs with new nodes. 

\subsection{Graph Injection Attack}
Graph injection attack (GIA) is a type of attack that manipulates the structure of a graph by injecting malicious nodes with carefully crafted features to degrade the performance of node classification. For instance, a representative example HAOGIA~\cite{chen2022understanding} first formulates an adversarial objective function comprising two key components: an attack objective and an unnoticeable objective. The attack objective guides the attacker in accomplishing their malicious goal, while the unnoticeable objective aids the attacker in evading detection by defenders. Subsequently, the method utilizes gradient descent to iterate and locate the optimal edges and node features that maximize the objective function. 
We note that a GIA can occur at test time (i.e., evasion attack~\cite{chen2022understanding,ju2023let,chen2023single}) and training time (i.e., poisoning attack~\cite{sun2020non,zou2021tdgia,chen2023imperceptible,zhang2021reverse,huang2021data,fan2023adversarial,guo2023targeted}). 
Specifically, the former will manipulate the testing data to disrupt a trained model, while the latter will manipulate the training data causing changed model parameters. Our proposed scheme is applicable to both evasion and poisoning attacks.


A common requirement for GIA is to remain stealthy such that the injected malicious nodes can be spotted by detectors.
To ensure stealthiness, many existing attacks impose a constraint that the degree of an injected node should not exceed the average degree of the clean graph~\cite{sun2020non,tao2021single,chen2023imperceptible}. In particular, attack methods like $G$-NIA\cite{tao2021single}, $G^2A2C$~\cite{ju2023let}, and $G^2$-SNIA~\cite{chen2023single} 
adopt a default strategy of inserting a single malicious node with a single edge.
\textit{This motivates us to investigate certified robustness against such constrained while more realistic graph injection attacks.}

\subsection{Randomized Smoothing}


To defend against attacks, numerous approaches have been proposed, especially empirical approaches, such as adversarial training~\cite{dai2019adversarial}, robust models~\cite{kipf2016semi,jin2020graph,zhu2019robust,zhang2020gnnguard}, and anomaly detectors~\cite{zhang2020gcn}. Additionally, certified approaches provide provable robustness within a predefined perturbation set, representing the attacker's capabilities.

 A mainstream technique to achieve certified robustness is \textit{randomized smoothing}~\cite{bojchevski2020efficient,wang2021certified,jia2020certified,schuchardt2023localized}. It provides probabilistic certified robustness by adding random noise to the input samples.
One representative smoothing scheme designed for graph data is the sparsity-aware smoothing method~\cite{bojchevski2020efficient}. It offers an $l_0$-ball guarantee for graph modification attacks (GMA). This guarantee specifies the maximum number of edges that can be added or deleted among existing nodes while maintaining consistent predictions. It achieves this by creating a smoothed classifier through randomization, which involves randomly deleting or adding edges to the input samples. 
Based on this randomization, denoted as $\phi(\cdot)$, it constructs a smoothed classifier:
\begin{equation}
\label{eqn:smooth_g0}
    g_v(G):=\argmax_{y\in \{1,\cdots,C\}}\mathbb{P}(f_v(\phi(G))=y),
\end{equation}
where $f:\mathbb{G}\rightarrow \{1,\cdots, C\}^n$ returned the class $y$ given a randomized graph $\phi(G)$, and smoothed classifier $g(\cdot)$ returns the ``majority votes'' of the base classifier $f(\cdot)$. Then it estimates the probability of the model’s output given the perturbed graph based on the prediction on the clean graph and the sample overlap probability. The smoothed classifier has a certified consistent prediction if the top class probability $p'_A$ is larger than the runner-up class $p'_B$ under any perturbed graph in the perturbation set. 

\section{Problem Statement}
\label{Sec:Problem}
In this section, we present a formal definition of the threat model and outline our defense goal. 
\subsection{Threat model}
We consider an attacker whose goal is to degrade the performance of the node classification performance of a classifier. To achieve this, the attacker is allowed to inject $\rho$ nodes $\tilde{\mathcal{V}}=\{\tilde{v}_1,\cdots,\tilde{v}_{\rho}\}$ with arbitrary node features $\tilde{X}\in \mathbb{R}^{\rho \times d}$ into the graph. Let $\tilde{\mathcal{E}}$ denote the inserted edges from $\tilde{\mathcal{V}}$. To ensure stealthiness of attacks and constrain the attacker's ability, we assume that each injected node $\tilde{v}$ can connect at most $\tau$ edges. That is the degree of node $\tilde{v}$, denoted as $\delta(\tilde{v})$, is less than $\tau$. 

The attack causes a perturbation to the original graph $G$. Specifically, we define the node injection perturbation set as $B_{\rho,\tau}(G)$:
\begin{align}
\label{eqn:pertb_ball}
    B_{\rho,\tau}(G)&:=\{G'(\mathcal{V}',\mathcal{E}',X')|\mathcal{V}'=\mathcal{V}\cup\tilde{\mathcal{V}}, \mathcal{E}'=\mathcal{E}\cup \tilde{\mathcal{E}}, \nonumber\\ 
    &X'=X \cup \tilde{X}, |\tilde{\mathcal{V}}| \leq \rho,\delta(\tilde{v})\leq\tau,\forall \tilde{v} \in \tilde{\mathcal{V}}\}
\end{align}
The perturbation set 
$B_{\rho,\tau}(G)$ means that there are at most $\rho$ injected nodes, and at most $\rho\cdot \tau$ perturbed edges. This perturbation set belongs to a $l_0$-norm ball. 

Furthermore, the perturbation can occur before or after the model training, which is defined as \textit{evasion} attack and \textit{poisoning} attack.
We show that our proposed scheme is applicable for both the evasion and poisoning attacks.

\subsection{Goal of Provable Defense}
Our goal is to build up a smoothed classifier that can provide certified robustness against GIA. We emphasize that our method is model-agonistic in that it does not require to know model details. 
Specifically, we denote the perturbed graph with malicious nodes injected as $G'$, and its corresponding adjacency matrix as $A'$. For any graph classifier $f(\cdot)$, we create its smoothed classifier $g(\cdot)$. Our goal is to verify whether the classification result for a given node $v$ remains unchanged: $g_v(G)\overset{\text{?}}{=}g_v(G')$, for all adversarial example $G'\in B_{\rho,\tau}(G)$ in a predefined perturbation set.

\section{Certified Robustness against GIA}
\label{Sec:Method}
In this section, we first introduce sparsity-aware smoothing \cite{bojchevski2020efficient} that serves as the basis of our scheme. Then, we propose our node-aware bi-smoothing scheme, and provide the general theoretical condition for provable robustness.

\subsection{Preliminary: Sparsity-aware Smoothing}
Our node-aware bi-smoothing scheme is a nontrivial modification from the sparsity-aware smoothing \cite{bojchevski2020efficient}.
As mentioned above, sparsity-aware randomized smoothing is capable of providing $l_0$-ball guarantee for graph modification attack (GMA), in which the perturbation set denoted as $B_{r_a,r_d}$ is at most $r_a$ edges can be added and $r_d$ edges be deleted among existing nodes.
Specifically, it first specifies a randomization scheme $\phi$ that randomly adds or deletes edges:
$\mathbb{P}(\phi(A)_{ij}\neq A_{ij})= p_-^{A_{ij}}p_+^{(1-A_{ij})}$, where $p_-, p_+ \in[0,1]$ are the probability of adding edges or deleting edges (set $p_+=0$ in our adaptation). Based on the randomization, it constructs a smoothed classifier $g$ defined in Eq.~\eqref{eqn:smooth_g0}.
The model verifies whether $g(G)=g(G')$ for any adversarial example $G'\in B_{r_a,r_d}(G)$ in the given graph modification perturbation set. With this existing certifying framework proposed for $l_0$-ball graph modification attack (GMA), next, we illustrate how to generalize it to graph injection attack (GIA).

\subsubsection{A Direct Adaptation as Baseline}
We can use this model to certify the node injection perturbation set $B_{\rho,\tau}(G)$ defined in \eqref{eqn:pertb_ball}, by pre-injecting $\rho$ isolated nodes in the clean graph, and then applying this model to certify if the model can tolerate adding arbitrary $\rho\cdot \tau$ edges (i.e., $r_a=\rho\tau$). This is based on the assumption that the isolated/singleton nodes will not impact the classification results of other nodes. Note that this assumption holds for almost all graph models, such as all message-passing GNNs~\cite{scholten2022randomized} and common recommendation models. For the poisoning setting, we can merge the training phrase and testing phrase of the base classifier as a whole classifier $F(G)$ that takes $G$ as input for both training and then making predictions. To avoid the effect from isolated nodes to other nodes, we can let $F(G)$ bypass all isolated nodes in the training phase. (For the detailed adaptation with theoretical proof, please refer to Section~\ref{Sec:Method}, Theorem~\ref{theorm:certify_condition} with $p_n=0, p_-=p_e>0$.)

Despite the applicable adaptation, we point out that, given a perturbed graph $G'\in B_{\rho,\tau}(G)$, its likelihood ratio (sample overlap) $\frac{\mathbb{P}(\phi(G')=Z)}{\mathbb{P}(\phi(G)=Z)}>0$ only if all inserted edges are deleted, and the probability is $(p_-)^{\rho\tau}$, which diminishes significantly as the number of injected nodes and allowable edges grow. To enlarge the probability within the positive likelihood region, we subsequently introduce our node-aware bi-smoothing scheme.

\subsection{Node-aware Bi-Smoothing}


The main idea of our certificates against node injection is 
to design suitable smoothing distributions: (1) deleting edges $\phi_e(G)$, and (2) deleting nodes $\phi_n(G)$. Specifically, $\phi_e(G)$ randomly deletes edges in $G$ with probability $p_e$, and $\phi_n(G)$ randomly deletes nodes (all its incident edges) with probability $p_n$. We combine edge-level and node-level smoothing distributions to form $\phi(G)=(\phi_e(G),\phi_n(G))$, which we termed \textbf{node-aware bi-smoothing}, to generate the randomized smoothing samples and then classify all of the graphs to obtain the ``majority vote". 
We formally represent our smoothed classifier $g$ as follows:
\begin{align}
\label{eqn:smooth_g}
    &g_v(G):=\argmax_{y\in \{1,\cdots,C\}}p_{v,y}(G), \\
    &p_{v,y}(G):=\mathbb{P}(f_v(\phi(G))=y),\nonumber
\end{align}
where $p_{v,y}(G)$ represents the probability that the base graph classifier $f$ returned the class $y$ for node $v$ given a randomized graph $\phi(G)$, and smoothed classifier $g(\cdot)$ returns the ``majority votes'' of $f(\cdot)$.

We assume that for any graph model, the classification result of a query node $v\in G'$ is the same as $v\in G$ if the injected nodes are isolated from all existing nodes in the graph $G$. 
Next, we briefly illustrate how node-aware bi-smoothing can enlarge the probability of isolating all injected nodes. All the perturbed edges from an injected node are removed if the injected node is deleted in node deletion.
Moreover, since the attacker is restricted to injecting only a few edges per node, both node deletion and edge deletion contribute to the probability of individually deleting perturbed edges. A perturbed edge can be deleted either through edge deletion or by deleting the other node that the edge connects to
(See Figure~\ref{fig:nodeaware_scheme} for examples, and proof of Theorem~\ref{theorm:certify_condition} for more details). Since there are $\rho$ injected nodes, and at $\tau$ injected edges per injected nodes, the probability of deleting all the perturbed edges is $\tilde{p}:=(p_n+(1-p_n)(p_e+p_n-p_ep_n)^{\tau})^{\rho}$ under our proposed node-aware bi-smoothing. However, when adapting sparsity-aware smoothing~\cite{bojchevski2020efficient} ($p_n=0$), the probability is $(p_e)^{\rho\tau}$, which is much smaller. Next, we formulate the certified robustness verification problem and show that such probability is the key to yielding the robustness guarantee (Theorem~\ref{theorm:certify_condition}).

\begin{figure}[t]
    \centering
\includegraphics[width=0.40\textwidth,height=5.2cm]{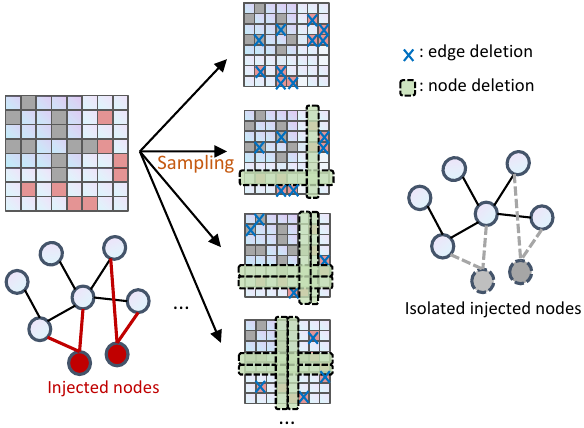}
    \caption{Illustration of \textbf{node-aware} bi-smoothing.}
    \label{fig:nodeaware_scheme}
\end{figure}

\subsection{Condition for Certified Robustness under Node-aware Bi-Smoothing}
We first formulate the certifying problem as a linear program following the idea of ~\cite{bojchevski2020efficient,lee2019tight} and then derive the condition for certified robustness.

\subsubsection{Problem Formulation}
For $\forall v\in \mathcal{V}$ and a given graph $G$, we assume that the top-class of $g_v(G)$ is $y_A=\argmax_{y\in\{1,\cdots,C\}}p_{v,y}(G)$ and the running-up class is $y_B=\argmax_{y\neq \hat{y}_A(v)} p_{v,y}(G)$. Let $p_A:=p_{v,y_A}(G)$, and $p_B:=p_{v,y_B}(G)$, if the node $v$ is correctly classified by $g$ under clean graph with certificate, we must have $p_A\geq \underline{p_A}>\overline{p_B}\geq p_B$, where the $\underline{p_A}$ is the lower bound of $p_A$, and $\overline{p_B}$ is the upper bound of $p_B$. The prediction can be further certified under perturbed graph if $p'_A>p'_B, \forall G'\in B_{r_a,r_d}(G)$, where $p'_A:=p_{v,y_A}(G')$ and $p'_B:=p_{v,y_B}(G')$ are the classification probabilities under perturbed graph.


The $p'_A$ and $p'_B$ can be obtained based on the fact that the randomized sample $\phi(G)$ and $\phi(G')$ have a probability of being overlapped, and the likelihood ratio is the same within some regions. We can divide sample space into disjoint regions $\mathbb{G}=\bigcup_i^I \mathcal{R}_i$, where $\mathcal{R}_i$ denote the consent likelihood region that $\frac{\mathbb{P}(\phi(G)=Z)}{\mathbb{P}(\phi(G')=Z)}=c_i$ for some constant $c_i$. 
Let $r_i=\mathbb{P}(\phi(G)\in\mathcal{R}_i)$, $r'_i=\mathbb{P}(\phi(G')\in\mathcal{R}_i)$ denote the probability that the random sample fall in the partitioned region $\mathcal{R}_i$. By the law of total probability, we have
$p_{v,y}(G)=\sum_i \mathbb{P}(f(Z)=y | Z\in \mathcal{R}_i)\mathbb{P}(\phi(G)=Z\in\mathcal{R}_i).$
Let $h_i:=\mathbb{P}(f(Z)=y_A|Z\in \mathcal{R}_i)$ and $t_i:=\mathbb{P}(f(Z)=y_B|Z\in \mathcal{R}_i)$, we have, $p'_A=p_{v,y_A}(G')=h^Tr'$, and $p'_B=p_{v,y_B}(G')=t^Tr'$.
Then, the verification problem can be defined as a Linear Programming (LP) problem:
\begin{align}
\label{opt:randomsmooth}
\min_{h,t}\quad &\mu=p'_A-p'_B=h^Tr'-t^Tr', \\
\text{s.t.} \quad &h^Tr=\underline{p_A},\, t^Tr=\overline{p_B},\nonumber\\
& 0\leq h \leq1,\, 0\leq t \leq 1,\nonumber
\end{align}
where the $\underline{p_A}$ is the lower bound of $p_A$, and $\overline{p_B}$ is the upper bound of $p_B$; the vectors $h\in[0,1]^{I}$ and $t\in[0,1]^{I}$ determine the worse-case classifier that assigns class $y_A$ and class $y_B$ among the regions such that the classification margin $\mu$ under perturbed graph is minimized. Hence, the optimal $\mu^*>0$ indicates that the prediction is certified to be consistent for $\forall G'\in B_{r_a,r_d}(G)$. 

\subsubsection{Solution \& Condition} The LP problem \eqref{opt:randomsmooth} can be solved directly according to sorted constant likelihood ratio regions (\cite{bojchevski2020efficient}, Appendix B). The worst-case classifier
$h$ will assign class $y_A$ in decreasing order of regions by their constant likelihood ratios 
($c_1\geq c_2 \geq \cdots \geq c_I$) until $\mathbb{P}(f_v(\phi(A))=y_A)=h^Tr=\underline{p_A}$, and $t$ will assign class $y_B$ in increasing order of the constant likelihood regions ($c_I\leq \cdots \leq c_1$) until $\mathbb{P}(f_v(\phi(A))=y_B)=t^Tr=\overline{p_B}$. Subsequently, leveraging this solution, we establish the theoretical condition of the certificate under our node-aware bi-smoothing scheme:


\begin{thm}
\label{theorm:certify_condition}
Let $f:\mathbb{G}\longrightarrow \{1,\cdots, C\}^{n}$ be any graph classifier, $g$ be its smoothed classifier defined in \eqref{eqn:smooth_g} with $\phi(G)=(\phi_e(G),\phi_n(G))$, $v\in G$ be any query node, $B_{\rho,\tau}(G)$ be the node injection perturbation set defined in \eqref{eqn:pertb_ball}. Suppose $y_A, y_B\in \{1,\cdots, C\}$ and $\underline{p_A}, \overline{p_B}\in[0,1]$. Then we have
$g_v(G')=g_v(G)$, $\forall G' \in B_{\rho,\tau}(G)$, if:
\begin{equation}
\label{eqn:certify_condition1}
    \mu_{\rho,\tau}:=\tilde{p}\,(\underline{p_A}-\overline{p_B}+1)-1 >0, 
\end{equation}
where $\tilde{p}:=(p_n+(1-p_n)(p_e+p_n-p_ep_n)^{\tau})^{\rho}$.
\end{thm}
\begin{proof}
    (Sketch) There are two constant likelihood ratios: $c_1=1/\tilde{p}$ when all inserted edges are removed and $c_2=0$ when they are not. The worst-case classifier with condition $\mathbb{P}(f_v(\phi(A))=y_A)=\underline{p_A}$ and $\mathbb{P}(f_v(\phi(A))=y_B)=\overline{p_B}$ will assign class $y_A$ in the low likelihood ratio region in priority in order to make the $p'_A$ (classification probability under perturb graph) as small as possible. On the other hand, it will assign class $y_B$ in the high likelihood ratio region in priority in order to make the $p'_B$ as large as possible. The $\mu_{\rho,\tau}$ is calculated from  $p'_A-p'_B$ under such a worst-case classifier. 
    Please refer to Appendix.~\ref{Sec:Appendix_A} for detailed proof.
\end{proof}

Based on the Theorem.~\ref{theorm:certify_condition}, we have the following corollary that further highlights the important role of the probability $\tilde{p}$ in certifiable condition:
\begin{corollary}
A node is only certifiable with the necessary conditions:
\begin{equation}
    \tilde{p}>\frac{1}{2}.
\end{equation}
\end{corollary}
\begin{proof}
    With the definition of $\underline{p_A}$ and $\overline{p_B}$, we have $\underline{p_A}-\overline{p_B}\leq 1$. Then, 
    \begin{align}
    \mu_{\rho,\tau}&=\tilde{p}(\underline{p_A}-\overline{p_B}+1)-1\nonumber\\
    &\leq \tilde{p}(1+1)-1\nonumber\\
    &\leq 2\cdot \tilde{p}-1.\nonumber
    \end{align}
    According to Theorem~\ref{theorm:certify_condition}, we can certify a node if $\mu_{\rho,\tau}>0$. We have $\mu_{\rho,\tau}>0$ only if $2\cdot \tilde{p}-1>0$, which means $\tilde{p}>\frac{1}{2}$.
\end{proof}

With the Theorem.~\ref{theorm:certify_condition}, we can now give black-box certified robustness for graph models against graph injection evasion attacks. Next, we show that this is also applicable to poisoning attacks with small changes.

\subsection{Improving Certificate against Poisoning Attack}

In this subsection, we show that our certifying scheme can also work for the poisoning attack threat model with minor adaptation. First, the certifying condition defined in Theorem~\ref{theorm:certify_condition} does not rely on the structure of the target model. That is, the Theorem~\ref{theorm:certify_condition} is suitable for any graph model as long as the isolated nodes do not impact the model predictions on other nodes. The graph model training process, however, can be viewed as an end-to-end function that takes in a graph for training and outputs a model parameter. With a fixed model parameter, it takes the same graph as input and outputs the predictions for each node. We can combine these two processes as a complex function $F:\mathbb{G}\longrightarrow \{1,\cdots, C\}^{n}$, so that the previous certifying scheme is applicable to the poisoning threat model. The only difference is that the classifier itself depends on the data. 

To avoid the impact of isolated nodes on the model parameter, we propose two different strategies termed \textbf{node-aware-include} and \textbf{node-aware-exclude}. By excluding isolating nodes from training while including them in the testing phase, node-aware-include strategy has the same certifying scheme as an evasion attack, because the data sampling is totally the same as in Theorem~\ref{theorm:certify_condition}. Nevertheless, the graph models trained without isolated nodes might have poor generalization on isolated nodes in the testing. Furthermore, some graph models, such as graph-based recommender system models, cannot make predictions for the nodes (i.e., users and items) that are not involved in the training phase. To deal with these problems, we propose a variant termed node-aware-exclude that excludes isolated nodes totally from training and testing. However, the sample space is slightly different from that in the Theorem~\ref{theorm:certify_condition} because the base model does not vote for the isolated nodes. Next, we formally define the smoothed classier with node-aware-exclude and provide the corresponding certifying condition.

We let the smoothed classifier $g$ under node-aware-exclude abstain from voting for all isolated nodes:
\begin{align}
\label{eqn:smooth_g_exclude}
    &g_v(G):=\argmax_{y\in \{1,\cdots,C\}}p_{v,y}(G), \\
    &p_{v,y}(G):=\mathbb{P}(F_v(\phi(G))=y)),\nonumber\\
    &F_v(\phi(G))=\text{ABSTAIN, if $v\in \phi(G)$ is isolated},\nonumber 
\end{align}
where $p_{v,y}(G)$ represents the probability that the base GNN classifier $F$ returned the class $y$ for node $v$ under the smoothing distribution $\phi(G)$. Note that the base classifier $F(\phi(G))$ here is first trained on $\phi(G)$ and then makes predictions. To derive the certified condition under such a smoothed classifier, we add a common assumption on the attacker that the attack edges added to a node $v$ should not exceed its original degree $d(v)$ (which is widely adopted in almost all attackers to ensure their stealthiness). The certified condition is given as the following theorem:

\begin{thm}
\label{theorm:certify_condition_exclude}
Let $f:\mathbb{G}\longrightarrow \{1,\cdots, C\}^{n}$ be any graph classifier, $g$ be its smoothed classifier defined in \eqref{eqn:smooth_g_exclude} with $\phi(G)=(\phi_e(G),\phi_n(G))$, $v\in G$ be any query node, $B_{\rho,\tau}(G)$ be the node injection perturbation set defined in \eqref{eqn:pertb_ball}, and the attack edges added to a node $v$ should not exceed its original degree $d(v)$. Suppose $y_A, y_B\in \{1,\cdots, C\}$ and $\underline{p_A}, \overline{p_B}\in[0,1]$. Then we have 
$g_v(G')=g_v(G)$, $\forall G' \in B_{\rho,\tau}(G)$, if:
\begin{align}
\label{eqn:certify_condition2}
\mu_{\rho,\tau}&:=\tilde{p}(\underline{p_A}-\frac{(1-\underline{p'_0})\overline{p_B}}{(1-p_0)}+1-\underline{p'_0})-(1-\underline{p'_0})>0, 
\end{align}
where $\tilde{p}:=(p_n+(1-p_n)(p_e+p_n-p_ep_n)^{\tau})^{\rho}$, $d(v)$ denotes the degree of node $v$, and $p_0:=p_n+(1-p_n)(p_e+p_n-p_ep_n)^{d(v)}$ is the probability that the node $v$ is deleted by the smoothing $\phi(G)$, $\underline{p'_0}:=p_n+(1-p_n)(p_e+p_n-p_ep_n)^{2d(v)}$.  
\end{thm}

\begin{proof}
    (Sketch) Given a node $v$, the classifier does not vote for it if the node $v$ is isolated in the smoothing. We need to calculate the likelihood ratio of regions that intersect with the region where $v$ is not excluded. For the $\phi(G)$, it has a probability of $1-p_0$ that the node $v$ is included in the voting, while for $\phi(G')$, the probability is upper bound by $1-\underline{p'_0}$ and lower bound by $1-p_0$ because the attacker inserts new edges to node $v$, and the number of new edges on a single node should not exceed the original degree by assumption. The likelihood ratio has two possible values: $c_1=\frac{1-p_0}{\tilde{p}(1-p'_0)}$ when all inserted edges are removed and $c_2=0$ when they are not. 
    See Appendix.~\ref{Sec:Appendix_A} for the complete proof.
\end{proof}

With Theorem.~\ref{theorm:certify_condition_exclude}, we next illustrate that it can be further extended to provide provable robust recommendations. 

\subsection{Certificate for Recommender System}
In particular, a recommender system can be regarded as a $K$-label classifier that predicts $K$ items for each user (node). To generalize our certifying scheme to the recommender system, we adopt the framework of PORE~\cite{jia2023pore} proposed by Jia et al., which defined the certified robustness problem as how many malicious users (nodes) can be injected while the recommendation for a user $u$ is maintained to a certain extent, i.e., \textit{at least $r$ recommended items are overlapped with ground truth items $I_u$}. 


A graph-based recommender system is trained on the user-item interaction graph $G$. Unlike PORE, which randomly selects $s$ users from the graph for aggregation (similar to \cite{jia2021intrinsic}), our approach involves applying node-aware bi-smoothing to generate random graphs. 
Specifically, in our method, $\phi_n(G)$ removes all ratings of a user with probability $p_n$, while $\phi_e(G)$ removes items with probability $p_e$. Assuming that a base recommender system trained on $\phi(G)$ predicts $K'$ items for a user $u$, we denote these predictions as $F_u(\phi(G))$. Our smoothed recommender system predicts the top-$K$ items based on the item probabilities obtained from the base recommender system. The probabilities, denoted as $p_{u,i}=\mathbb{P}(i\in F_u(\phi(G)))$, represent the probability of item $i$ being included in $F_u(\phi(G))$. We define the smoothed recommender system as $g_u(G)$:
\begin{align}
\label{eqn:smooth_g_RS}
    &g_u(G):=\{i|i \in \text{top-}K(p_{u,:})\}, \\
    &p_{u,i}:=\mathbb{P}(i\in F_u(\phi(G))),\nonumber\\
    &F_u(\phi(G))=\text{ABSTAIN, if $u\in \phi(G)$ is isolated},\nonumber 
\end{align}
where $\text{top-}K(p_{u,:})$ gives the top-$K$ items with the largest recommendation probability $p_{u,i}$.


According to~\cite{jia2023pore}, we can get at least $r$ recommended items that match with ground truth items $I_u$ if the $r$th highest item probability among items $I_u$ is higher than the $(K-r+1)$th highest item probability among $I\setminus I_u$ under the poisoned graph, where $I$ is all the items in the training set. We next provide the condition for $|g_u(G') \cap I_u|\geq r, \forall G'\in B_{\rho,\tau}(G)$ in the following theorem:
 
\begin{thm}
\label{theorm:certify_condition_RS}
    Let $F_u(G)$ be any base recommender system trained on $G$ and recommend $K'$ items to the user $u$, $g_u(G)$ be its smoothed recommender defined in \eqref{eqn:smooth_g_RS}, $u\in G$ be any query user, $B_{\rho,\tau}(G)$ be the node injection perturbation set defined in \eqref{eqn:pertb_ball}. Then, we have at least $r$ recommended items after poisoning are overlapped with ground truth items $I_u$: $|g_u(G') \cap I_u|\geq r, \forall G'\in B_{\rho,\tau}(G)$ if:
    \begin{equation}
    \label{eqn:certify_condition_RS}
        \hat{p}\,\underline{p_r}-\min_{H_c}(\overline{p}_{H_c}+K'(1-\hat{p})(1-p_0))/c>0,
    \end{equation}
    where $\hat{p}:=(p_n+(1-p_n)p_e^{\tau})^{\rho}$, $p_0:=p_n+(1-p_n)(p_e)^{d(u)}$ is the probability that the user $u$ is deleted by the smoothing $\phi(G)$, $d(u)$ is the number of user ratings in training set, $\underline{p_r}$ is the lower bound of the $r$th largest item probability among $\{p_{u,i}|i\in I_u\}$, $H_c$ denote any subset of the top-$(K-r+1)$ largest items among $I\setminus I_u$ with size $c$, $\overline{p}_{H_c}:=\sum_{j\in H_c} \overline{p_{u,j}}$ is the sum of probability upper bounds for $c$ items in $H_c$.
\end{thm}
\begin{proof}
    (sketch) All malicious users are removed in the randomization with probability $\hat{p}:=(p_n+(1-p_n)p_e^{\tau})^{\rho}$. See detailed proof in Appendix.~\ref{Sec:Appendix_A}.
\end{proof}

\section{Implementation in Practice}
\label{Sec:Implem}
With the certifying conditions from Theorem~\ref{theorm:certify_condition}, Theorem~\ref{theorm:certify_condition_exclude}, and Theorem~\ref{theorm:certify_condition_RS}, we aim to demonstrate how to instantiate them to train a defense model and obtain its certified robustness in both \textit{evasion} and \textit{poisoning} settings. 

\subsection{Certified Robustness Against Evasion Attack} 
Following~\cite{cohen2019certified,bojchevski2020efficient}, we train a graph model with noise augmentation to enhance the model's generalization on smoothed samples. In each epoch of training, we apply $\phi(G)$ to add noise to the graph (Algorithm~\ref{alg:training_noise}). After a graph model is trained, we sample $N$ random graphs $G_1,G_2,\cdots,G_N$ from the smoothing distribution $\phi(G)$ to process Monte Carlo: $p_{v,y}(G)\approx \sum_{i=1}^{N} \mathbb{I}(f_v(G_i)=y))/N$. 
Based on this frequency, we can obtain the top two predictions $y_A$ and $y_B$. Nevertheless, the prediction might not always be consistent due to the randomness. There are two levels of randomness we need to deal with: the prediction of $y_A$ and its probability $p_A$. To guarantee that the model predicts $y_A$ with probability at least $1-\alpha$, following \cite{cohen2019certified}, we employ a two-sided hypothesis test on the count of $y_A$ prediction $n_A \sim Binomial(n_A+n_B,\frac{1}{2})$, where $n_A:=\sum_{i=1}^{N} \mathbb{I}(f_v(G_i)=y_A)$ and $n_B:=\sum_{i=1}^{N} \mathbb{I}(f_v(G_i)=y_B)$. The model returns ABSTAIN during certifying if the p-value is greater than $\alpha$. To bound the probability $p_A$ and $p_B$, similar to~\cite{cohen2019certified,bojchevski2020efficient}, we compute a lower bound of $\underline{p_A}$ and upper bound of $\overline{p_B}$ based on the Clopper-Pearson Bernoulli confidence interval with confidence $\alpha/C$, where $C$ is the class number of the classifier. These lead to a lower bound of classification margin $\mu_{\rho,\tau}$, and it entails a valid certificate simultaneously with confidence level probability $\alpha$. The detailed practical certifying process is further outlined in Algorithm~\ref{alg:certify_evasion}. 
\begin{algorithm} 
\caption{Training with noise (evasion)}  
\label{alg:training_noise}  
\begin{algorithmic}[1]   
\REQUIRE Clean graph $G$, smoothing distribution $\phi(G)$ with smoothing parameters $p_e$ and $p_n$, training epoch $E$.
\FOR{$e=1,\cdots,E$}
\STATE{Draw a random graph $G_e\sim \phi(G)$. }
\STATE{$f=train\_model(f(G_e))$ on training nodes.}
\ENDFOR
\RETURN A base classifier $f(\cdot)$
\end{algorithmic}  
\end{algorithm}
\begin{algorithm}  
\caption{Certified robustness with Monte Carlo sampling (evasion)}  
\label{alg:certify_evasion}
\begin{algorithmic}[1]  
\REQUIRE Clean graph $G$, smoothing distribution $\phi(G)$ with smoothing parameters $p_e$ and $p_n$, trained base classifier $f(\cdot)$, sample number $N$, confidence level $\alpha$, perturbation budget $\rho$ and $\tau$.
\STATE{Draw $N$ random graphs $\{G_i|\sim G_i \sim \phi(G)\}_{i=1}^N$.}
\STATE{$counts=|\{i: f(G_i)=y\}|$, for $y=1, \cdots, C$.}
\STATE{$y_A,y_B=$ top two indices in $counts$.}
\STATE{$n_A,n_B=counts[y_A],counts[y_B]$.}
\STATE{$\underline{p_A},\overline{p_B}=\text{CP\_Bernolli}(n_A,n_B,N,\alpha)$.}
\IF{Binomial($n_A+n_B,\frac{1}{2})>\alpha$}
    \RETURN ABSTAIN
\IF{$\mu_{\rho,\tau}>0$}
\RETURN Certified prediction $y_A$.
\ENDIF
\ENDIF
\RETURN ABSTAIN
\end{algorithmic}  
\end{algorithm}

\subsection{Certified Robustness Against Poisoning Attack}
In poisoning attacks, because the perturbation is before the training phase, we take the training and inference of a model as an end-to-end function $F(\cdot)$ to substitute the base classifier $f(\cdot)$ in the evasion attack. That is, for each randomized graph $G_i \sim \phi(G)$, we first train a model based on $G_i$ and then make predictions. We further summarize the certifying process in Algorithm~\ref{alg:certify_poisoning}. Note that the node-aware-include and node-aware-exclude primarily differ in the calculation of $\mu_{\rho,\tau}$. The former utilizes Eq.~\eqref{eqn:certify_condition1}, while the latter employs Eq.~\eqref{eqn:certify_condition2}. Regarding the implementation of the recommender system, it shares a similar training process with node-aware-exclude. The key distinction is that the prediction is obtained by the classifier defined in \eqref{eqn:smooth_g_RS}, and $\mu_{\rho,\tau}$ should be computed with \eqref{eqn:certify_condition_RS}.

\begin{algorithm} 
\caption{Certified robustness with Monte Carlo sampling (poisoning)}  
\label{alg:certify_poisoning}
\begin{algorithmic}[1]  
\REQUIRE Clean graph $G$, smoothing distribution $\phi(G)$ with smoothing parameters $p_e$ and $p_n$, sample number $N$, confidence level $\alpha$, perturbation budget $\rho$ and $\tau$.
\STATE{Draw $N$ random graphs $\{G_i|\sim G_i \sim \phi(G)\}_{i=1}^N$.}
\STATE{$F(G_i):$ model trained on $G_i$ and makes predictions.}
\STATE{$counts=|\{i: F(G_i)=y\}|$, for $y=1, \cdots, C$.}
\STATE{$y_A,y_B=$ top two indices in $counts$.}
\STATE{$n_A,n_B=counts[y_A],counts[y_B]$.}
\STATE{$\underline{p_A},\overline{p_B}=CP\_Bernolli(n_A,n_B,N,\alpha)$.}
\IF{Binomial($n_A+n_B,\frac{1}{2})>\alpha$}
    \RETURN ABSTAIN
\IF{$\mu_{\rho,\tau}>0$}
\RETURN Certified prediction $y_A$.
\ENDIF
\ENDIF
\RETURN ABSTAIN
\end{algorithmic}  
\end{algorithm}

\subsection{Empirical Robustness}
Our proposed node-aware bi-smoothing can not only provide certified robustness but also serve as a general defense strategy that alleviates the threat of graph injection attack (GIA). We can follow exactly the same process used to train a smoothed model to achieve empirical robustness in the presence of a perturbed graph. Importantly, our approach is compatible with a wide range of base classifiers. This property further offers the practical values of our model.

\section{Evaluation}
\label{Sec:Results}
We conduct extensive experiments on three datasets to evaluate our proposed certifiably robust framework for \textit{node classification} and  \textit{recommender system}. We assess the certifiable robustness of the smoothed classifiers against \textit{evasion} attack and \textit{poisoning} attack. In summary, our experiments show the following findings:

\begin{itemize}
    \item Our proposed node-aware bi-smoothing scheme significantly enhances the certified accuracy and average certifiable radius under various realistic graph injection attack (GIA) scenarios. 
    \item The variant node-aware-exclude method we propose for poisoning attacks further improves the certification performance in both the node classification task and recommendation task.
    \item Our node-aware bi-smoothing scheme has shown competitive empirical defense performance when compared to existing baselines. 
    \item Ablation studies demonstrate the crucial role of node-aware bi-smoothing and node-aware-exclude in achieving successful certification against GIA. 
\end{itemize}

\subsection{Experiment Setting}
In general, we follow the settings in~\cite{cohen2019certified,bojchevski2020efficient}. Next, we will explain the detailed settings, including datasets and models, certificate parameters, baselines, and evaluation metrics.
\subsubsection{Datasets and Models}
We take Graph Convolution Neural Network (GCN)~\cite{kipf2016semi}, one of the most representative GNNs, as the base classifier in node classification on Cora-ML and Citeseer datasets. For evaluating the recommender system, we take an item-based recommender system named SAR~\cite{argyriou2020microsoft} on MovieLens-100k dataset~\cite{harper2015movielens} following~\cite{jia2023pore}. Specifically, the Cora-ML dataset contains $2,995$ nodes, $8,416$ edges, and $7$ classes, with an average node degree of $5.68$. The Citeseer dataset contains $3,327$ nodes, $4,732$ edges, and $6$ classes, with an average node degree of $3.48$. For each class, we sample 50 nodes as the training set, 50 for the validation set, and the remaining for the testing set. 
The MovieLens-100k dataset contains about $100,000$ rating records involving $943$ users and $1,682$ items. For each user, we take $85\%$ of its history ratings as training data and the remaining for testing data. 

\subsubsection{Certificate Parameters}
By default, we set the number of smoothing samples as $N=100,000$, for certifying GCN node classification against evasion attack, and $N=1,000$ for poisoning attack. For certifying the recommender system, we set the $N=100,000$. For all the experiments, we see the confidence level as $\alpha=0.01$. For the node injection perturbation set, we evaluate a range of node injection numbers $\rho$ and various edge budgets $\tau=5,10$ ($\tau=5$ if not mentioned). 

\subsubsection{Adapted Baselines}
Given the absence of previous work on certifying general node classification tasks against GIA, we adapt existing certificates designed for other tasks. There are three certifying schemes adaptable to our task:
\begin{itemize}
    \item Bagging-cert~\cite{jia2021intrinsic}: As mentioned in Section~\ref{Sec:Intro}, the bagging-cert was originally designed for certifying inserting or deleting training samples in image classification tasks. To extend it for node infection perturbation, we can view each node with its incident edges as an independent training sample, such that, the problem is the same as image classification. Note that it can only
    certify against poisoning attacks.
    \item Sparsity-aware~\cite{bojchevski2020efficient}: Another way is to use sparsity-aware by adding $\rho$ singleton nodes to the clean graph, and then certifying how many edge insertions it can withstand. When we set $p_n=0$ in our node-aware bi-smoothing, it becomes sparsity-aware smoothing (our $p_e$ is analogous to $p_-$ in \cite{bojchevski2020efficient}).
    \item PORE~\cite{jia2023pore}: It extends the bagging-cert~\cite{jia2021intrinsic} to provide a provable recommender system scheme under node injection attack, and we employ it as our baseline when certifying the recommender system.
\end{itemize}

\subsubsection{Evaluation Metrics}

A common metric to measure the robustness of a model with guarantee is \textbf{\textit{certified accuracy}}: it is the ratio of samples that the prediction is both \textit{correct} and \textit{certified} to be consistent under the defined perturbation set. We formally define the certified accuracy with the given attack budget $\rho$ and $\tau$ as follows:
$\xi(\rho,\tau)=\frac{1}{n}\sum_{i=1}^n \mathbb{I}(g_{v_i}(G)=g_{v_i}(G')=y_i), \forall G'\in B_{\rho,\tau}(G),$ where $y_i$ is the ground truth of node $v_i$. In this paper, we evaluate the certified accuracy of the set of testing nodes. 
However, evaluating the certificate strength is insufficient due to the inherent trade-off between prediction quality (i.e., clean accuracy) and certified accuracy. In general, higher 
smoothing variance improves the certified accuracy, but it reduces the prediction confidence. For this reason, following~\cite{schuchardt2020collective,schuchardt2023localized}, we also quantify \textbf{\textit{average certifiable radius}} (ACR) with given degree budget $\tau_0$: $ACR=\sum_{\rho=1}^{+\infty}\rho \cdot(\xi(\rho,\tau_0)-\xi(\rho+1,\tau_0))$. Intuitively, it is the discrete integral (area) under the certified accuracy curve (Figure~\ref{fig:ACR_explain}). 
\begin{figure}[h]
    \centering
    \includegraphics[width=0.20\textwidth,height=3.0cm]{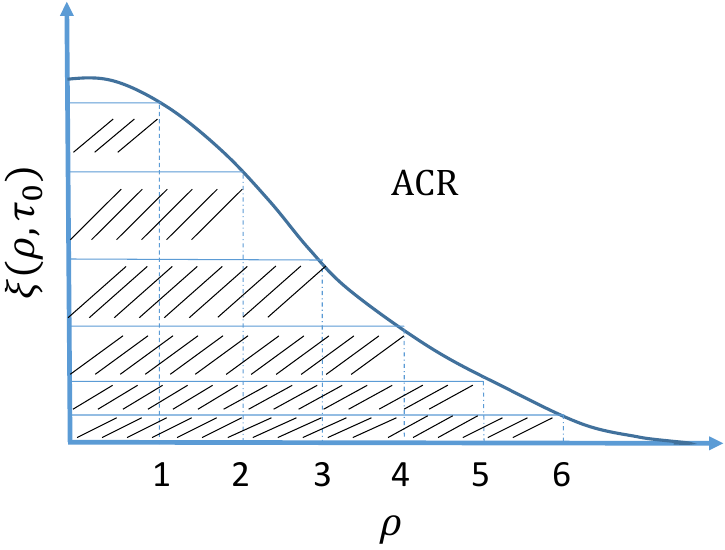}
    \caption{Illustration of ACR, where $\xi(\rho,\tau_0)$ denotes the certified accuracy under a fixed degree budget $\tau_0$.}
    \label{fig:ACR_explain}
\end{figure}

When evaluating the robustness of the recommender system, we adopt \textit{\textbf{certified precision}} and \textit{\textbf{certified recall}} following~\cite {jia2023pore}. These metrics measure the certified overlapping of recommended items and ground truth items:

\begin{align}
    &\text{certified precision} = \min_{G'\in B_{\rho,\tau}(G)} \frac{|I_u\cap g_u(G')|}{K},\\
    &\text{certified recall} = \min_{G'\in B_{\rho,\tau}(G)} \frac{|I_u\cap g_u(G')|}{I_u},
\end{align}
where $K$ is the number of recommended items, and $I_u$ is the ground truth recommendations. Similarly, \textit{we can calculate \textbf{ACR} by substituting certified accuracy with certified precision or recall}. 

\subsection{Certified  Robustness for Node Classification}

\subsubsection{Against Graph Injection Evasion Attack}

\begin{table}[hbt!]
\centering
\caption{Certified accuracy comparison under \textit{evasion} perturbation. For each method, we report the best results under different smoothing parameters, while the $\alpha$ and $N$ are the same.}
\setlength{\tabcolsep}{3pt}
\begin{tabular}{lclrrrrr}
\hline
\multicolumn{3}{c}{certified   accuracy} & \multicolumn{4}{c}{$\rho$} & \multirow{2}{*}{ACR} \\ \cline{1-7}
dataset & \multicolumn{1}{l}{$\tau$} & methods & \multicolumn{1}{l}{0} & \multicolumn{1}{l}{3} & \multicolumn{1}{l}{5} & \multicolumn{1}{l}{10} &  \\ \hline
\multirow{4}{*}{Cora-ML} & \multirow{2}{*}{5} & sparse-aware\cite{bojchevski2020efficient} & \textbf{0.809} & 0.000 & 0.000 & 0.000 & 0.691 \\
 &  & \textbf{node-aware} & 0.735 & \textbf{0.730} & \textbf{0.730} & \textbf{0.729} & \textbf{100.648} \\ \cline{2-8}
 & \multirow{2}{*}{10} & sparse-aware\cite{bojchevski2020efficient} & \textbf{0.809} & 0.000 & 0.000 & 0.000 & 0.000 \\
 &  & \textbf{node-aware} & 0.735 & \textbf{0.730} & \textbf{0.729} & \textbf{0.721} & \textbf{51.390} \\ \hline
\multirow{4}{*}{Citeseer} & \multirow{2}{*}{5} & sparse-aware\cite{bojchevski2020efficient} & \textbf{0.705} & 0.000 & 0.000 & 0.000 & 0.644 \\
 &  & \textbf{node-aware} & 0.674 & \textbf{0.666} & \textbf{0.666} & \textbf{0.666} & \textbf{31.558} \\ \cline{2-8}
 & \multirow{2}{*}{10} & sparse-aware\cite{bojchevski2020efficient} & \textbf{0.705} & 0.000 & 0.000 & 0.000 & 0.000 \\
 &  & \textbf{node-aware} & 0.674 & \textbf{0.666} & \textbf{0.666} & \textbf{0.666} & \textbf{16.979} \\ \hline
\end{tabular}
\label{tab:evasion_CA}
\end{table}

\begin{figure}[hbt!]
\centering
    \subfigure[Sparsity-aware \cite{bojchevski2020efficient}]{
    \includegraphics[width=0.18\textwidth,height=2.8cm]{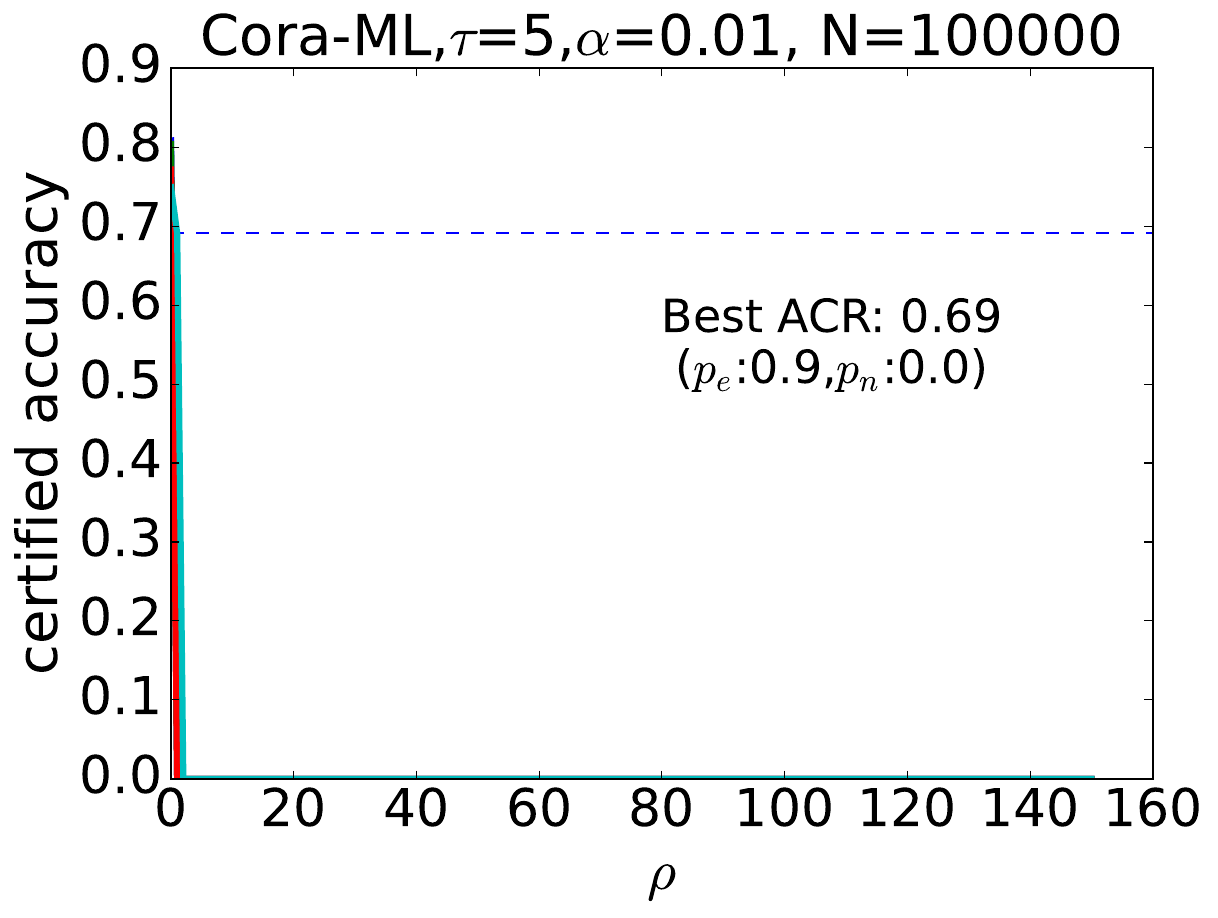}
    }
    \subfigure[Sparsity-aware \cite{bojchevski2020efficient}]{
    \includegraphics[width=0.255\textwidth,height=2.8cm]{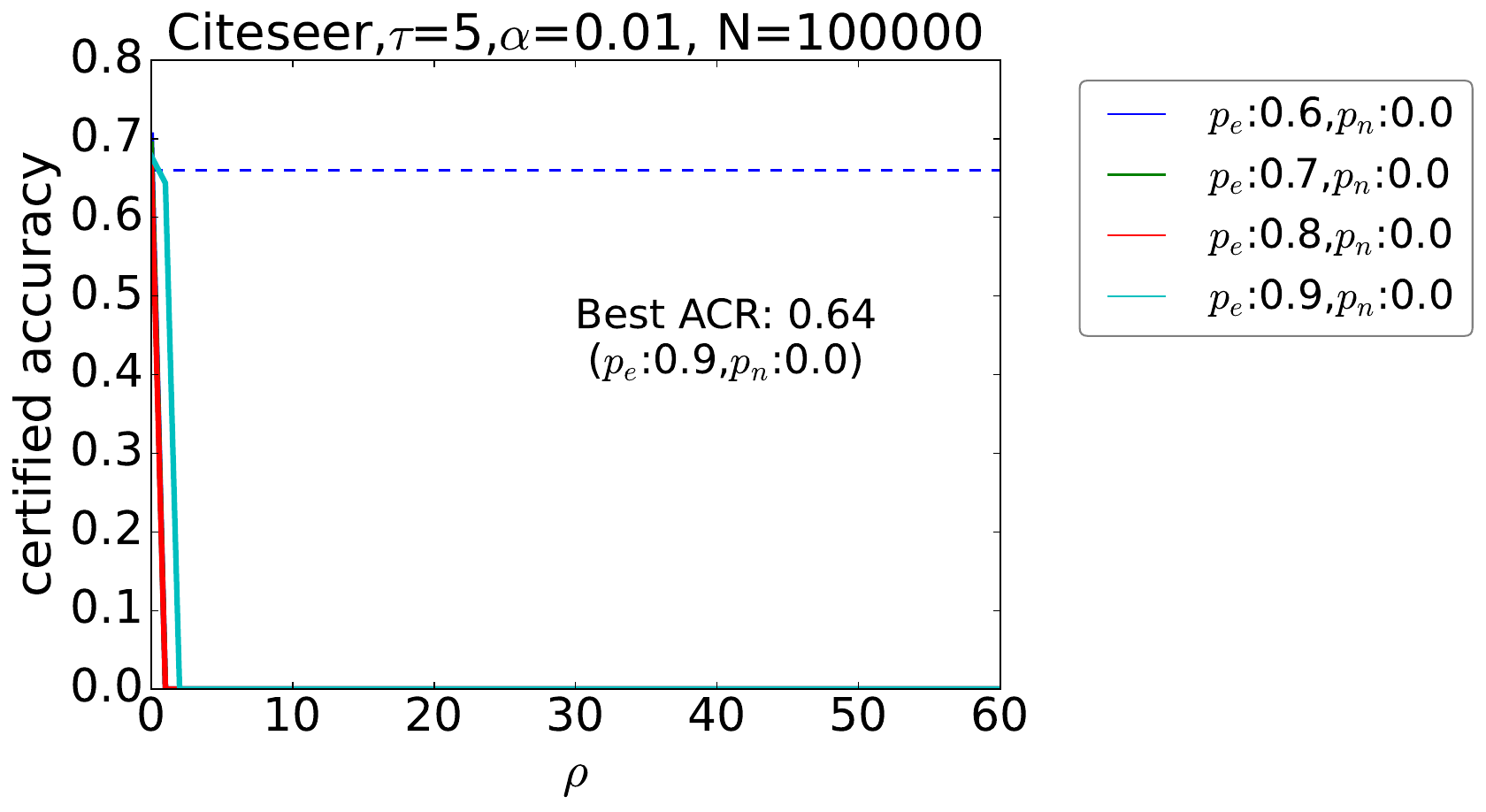}
    }
    \subfigure[\textbf{Node-aware} ($p_e=0$)]{
    \includegraphics[width=0.18\textwidth,height=2.8cm]{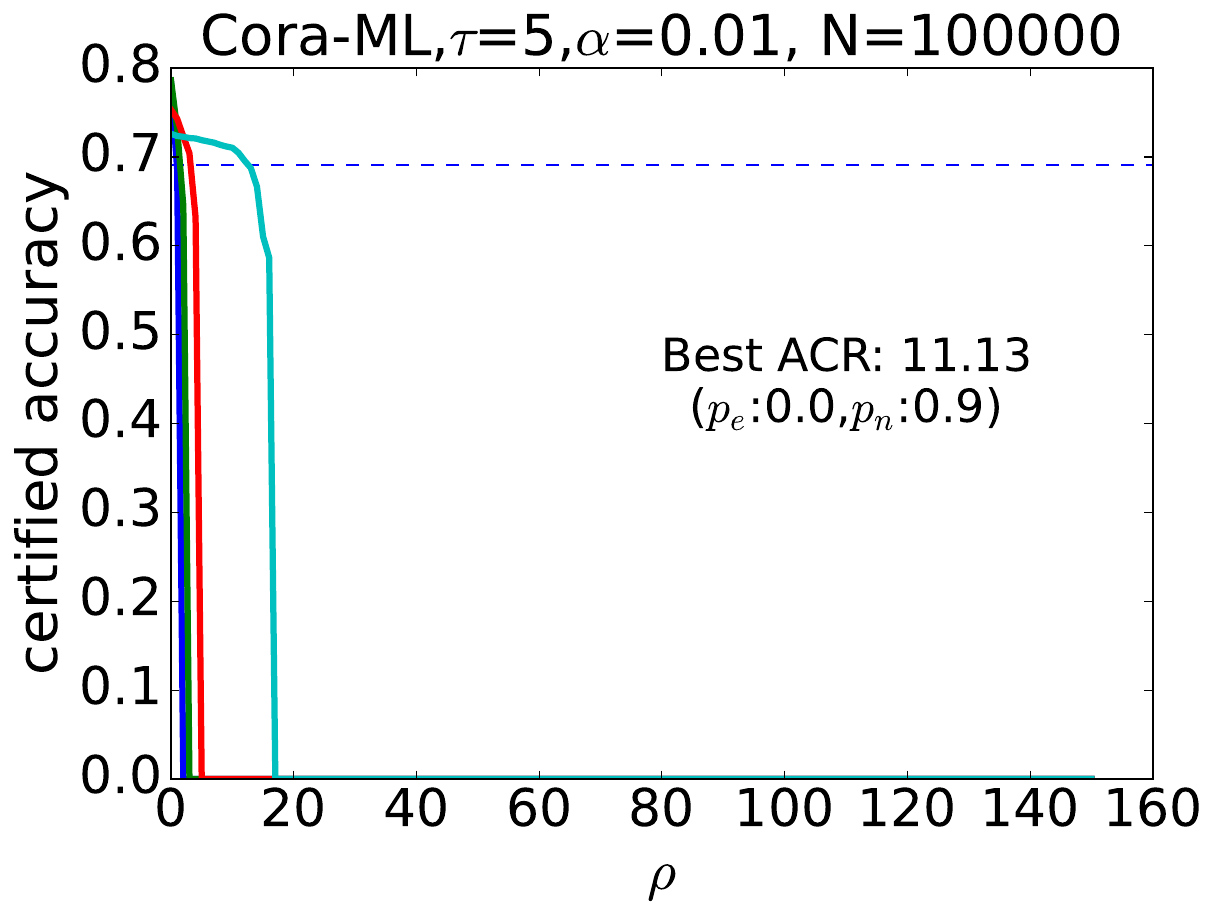}
    }
    \subfigure[\textbf{Node-aware} ($p_e=0$)]{
    \includegraphics[width=0.255\textwidth,height=2.8cm]{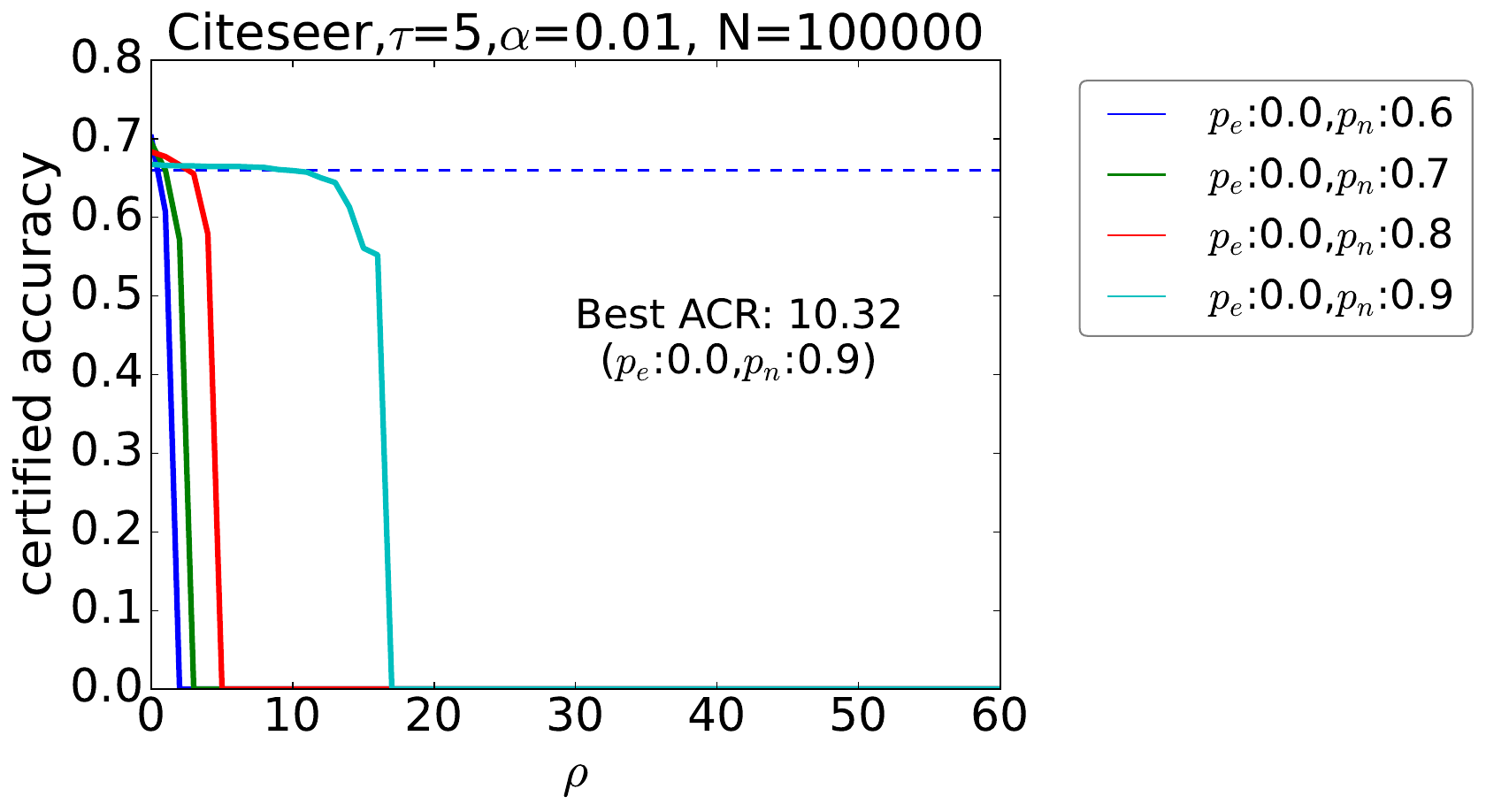}
    }
    \subfigure[\textbf{Node-aware} ($p_e,p_n>0$)]{
    \includegraphics[width=0.18\textwidth,height=2.8cm]{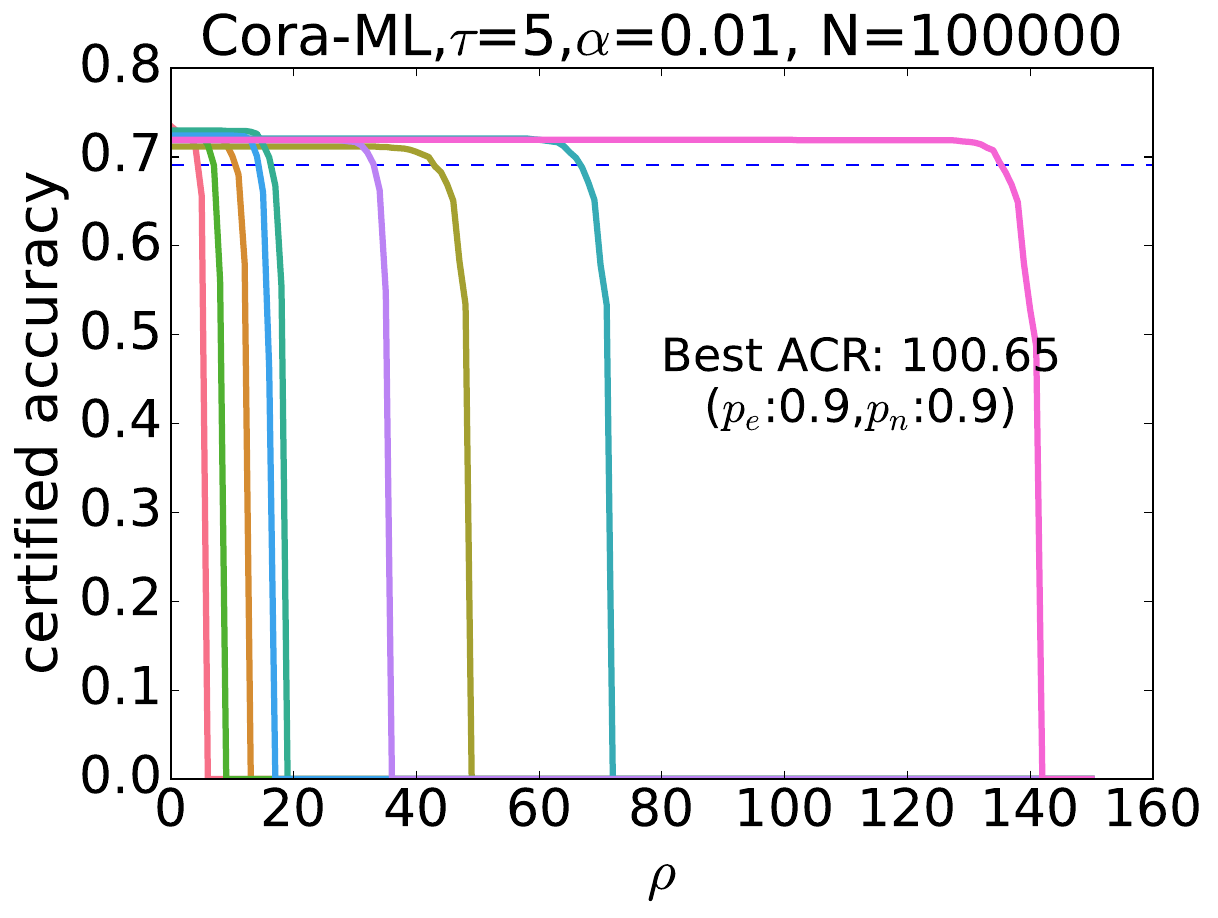}
    }
    \subfigure[\textbf{Node-aware} ($p_e,p_n>0$)]{
    \includegraphics[width=0.255\textwidth,height=2.8cm]{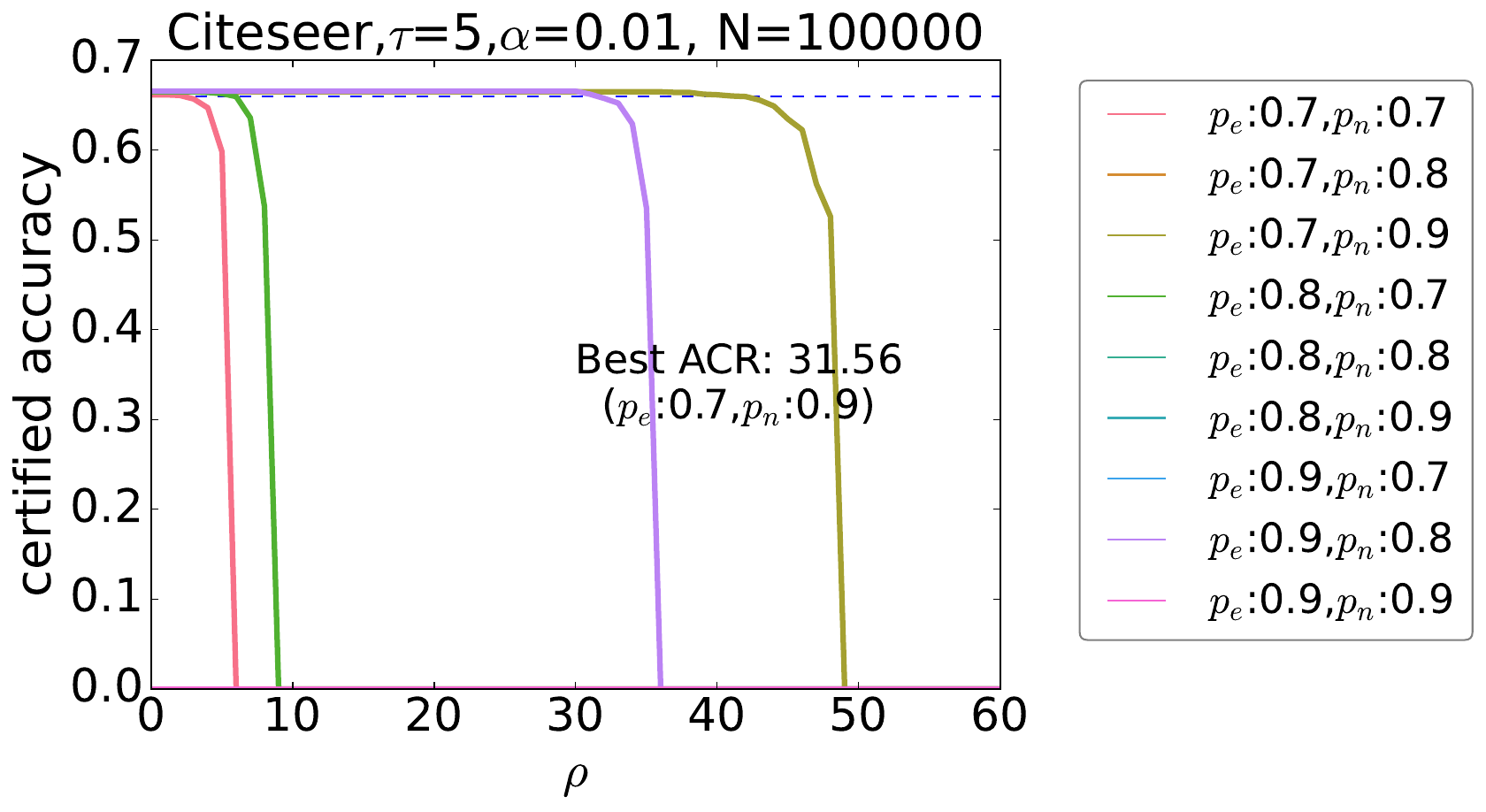}
    }
\caption{Certified accuracy under \textit{evasion} perturbation with $\tau=5$. $\rho$: the number of injected nodes, $\tau$: the edge budget per injected node. The blue dotted line represents the accuracy of Multilayer Perceptron (MLP).}
\label{fig:cer_evasion}
\end{figure}
\vspace{-10pt}
For evasion attacks, we adapt sparsity-aware smoothing~\cite{bojchevski2020efficient} as a baseline for comparison (since bagging-cert~\cite{jia2021intrinsic} is designed for poisoning attacks only).
In Table.~\ref{tab:evasion_CA}, we report the certified accuracy under the attack budgets $\rho=3,5,10$, and also the clean accuracy ($\rho=0$) of the smoothed classifier. The average certifiable radius (ACR) comparison is also shown in Table.~\ref{tab:evasion_CA}. For all the metrics, we report the best result for each method under various smoothing hyper-parameters. Compared to the baseline, we observe that \textit{our proposed node-aware bi-smoothing leads to overwhelmingly higher certified accuracy} under the same budgets. For example, on the Cora-ML dataset, when there are at most $10$ injected nodes with a budget degree of $5$, we achieve a certified accuracy of $72.9\%$ while the baseline is $0.0\%$.

In Figure.~\ref{fig:cer_evasion}, we report the certified accuracy curve under a range of attack budget $\rho$ using various smoothing parameters ($p_e$ and $p_n$). We find that sparse-smoothing is not able to certify attack budget $\rho>2$ with $\tau=5$ (Figure.~\ref{fig:cer_evasion}a,\ref{fig:cer_evasion}b), this highlights that \textit{the direct adaptation from graph modification attack (GMA) to graph injection attack (GIA) is not effective although it is possible}. In contrast,\textit{ our node-aware bi-smoothing has non-trivial ACR under various parameters ($p_e>0.7$ and $p_n>0.7$)}. 

Due to the trade-off between certified accuracy and clean accuracy, we also report the clean accuracy under various smoothing parameters in Figure.~\ref{fig:clean_acc}a,\ref{fig:clean_acc}b. 
Since the MLP does not rely on graph structure, node injection attacks under an evasion scenario are ineffective. As a result, a robust graph classifier is meaningful only if its clean accuracy is higher than the MLP. We only report the effective results (the parameters with clean accuracy higher than the MLP).

\subsubsection{Against Graph Injection Poisoning Attack}
Similar to evasion attacks, we also observed significantly higher certifiable performance under our node-aware bi-smoothing compared to baselines sparsity-aware~\cite{bojchevski2020efficient} and bagging-cert~\cite{jia2021intrinsic} (Table.~\ref{tab:poison_CA} and Figure.~\ref{fig:cer_poison}). 
Notably, \textit{our ensemble smoothed classifier has the potential to increase the clean accuracy}, as shown in Figure.~\ref{fig:clean_acc}c,\ref{fig:clean_acc}d,\ref{fig:clean_acc}e,\ref{fig:clean_acc}f. Without sacrificing clean accuracy, our proposed node-aware-exclude has over $55\%$ certifiable accuracy on Cora-ML, and $43\%$ certifiable accuracy on Citeseer with $10$ allowable malicious node with arbitrary features, while the two baseline methods are not able to certify any one of the nodes under the same condition as shown in Table.~\ref{tab:poison_CA} and Figure.~\ref{fig:cer_poison}. Moreover, our node-aware-exclude has an average certifiable radius (ACR) of $16.66$ on Cora-ML and $12.05$ on Citeseer, which improved baselines by $760\%$ and $530\%$, respectively. \textit{These experimental results reveal the effectiveness of our proposed method in poisoning attack scenarios.}

\begin{table}[hbt!]
\centering
\caption{Certified accuracy comparison under \textit{poisoning} perturbation ($\tau=5$).}
\setlength{\tabcolsep}{3pt}
\begin{tabular}{clrrrrr}
\hline
\multicolumn{2}{c}{certified   accuracy} & \multicolumn{4}{c}{$\rho$} & \multirow{2}{*}{ACR} \\ \cline{1-6}
\multicolumn{1}{l}{dataset} & methods & \multicolumn{1}{l}{0} & \multicolumn{1}{l}{3} & \multicolumn{1}{l}{5} & \multicolumn{1}{l}{10} &  \\ \hline
\multirow{4}{*}{Cora-ML} & sparse-aware\cite{bojchevski2020efficient} & \textbf{0.832} & 0.000 & 0.000 & 0.000 & 0.229 \\
 & bagging\cite{jia2021intrinsic} & 0.744 & 0.354 & 0.163 & 0.000 & 1.931 \\
 & \textbf{node-aware-include} & 0.770 & {\ul 0.480} & {\ul 0.428} & {\ul 0.351} & {\ul 8.297} \\
 & \textbf{node-aware-exclude} & {\ul 0.819} & \textbf{0.666} & \textbf{0.587} & \textbf{0.554} & \textbf{16.658} \\ \hline
\multirow{4}{*}{Citeseer} & sparse-aware\cite{bojchevski2020efficient} & \textbf{0.740} & 0.000 & 0.000 & 0.000 & 0.130 \\
 & bagging\cite{jia2021intrinsic} & 0.681 & 0.362 & 0.146 & 0.001 & 1.905 \\
 & \textbf{node-aware-include} & {\ul 0.734} & {\ul 0.493} & {\ul 0.443} & {\ul 0.321} & {\ul 8.835} \\
 & \textbf{node-aware-exclude} & 0.717 & \textbf{0.530} & \textbf{0.462} & \textbf{0.428} & \textbf{12.048} \\ \hline
\end{tabular}
\label{tab:poison_CA}
\end{table}

\begin{figure}[hbt!]
\centering
    \subfigure[Sparsity-aware \cite{bojchevski2020efficient}]{
    \includegraphics[width=0.18\textwidth,height=2.8cm]{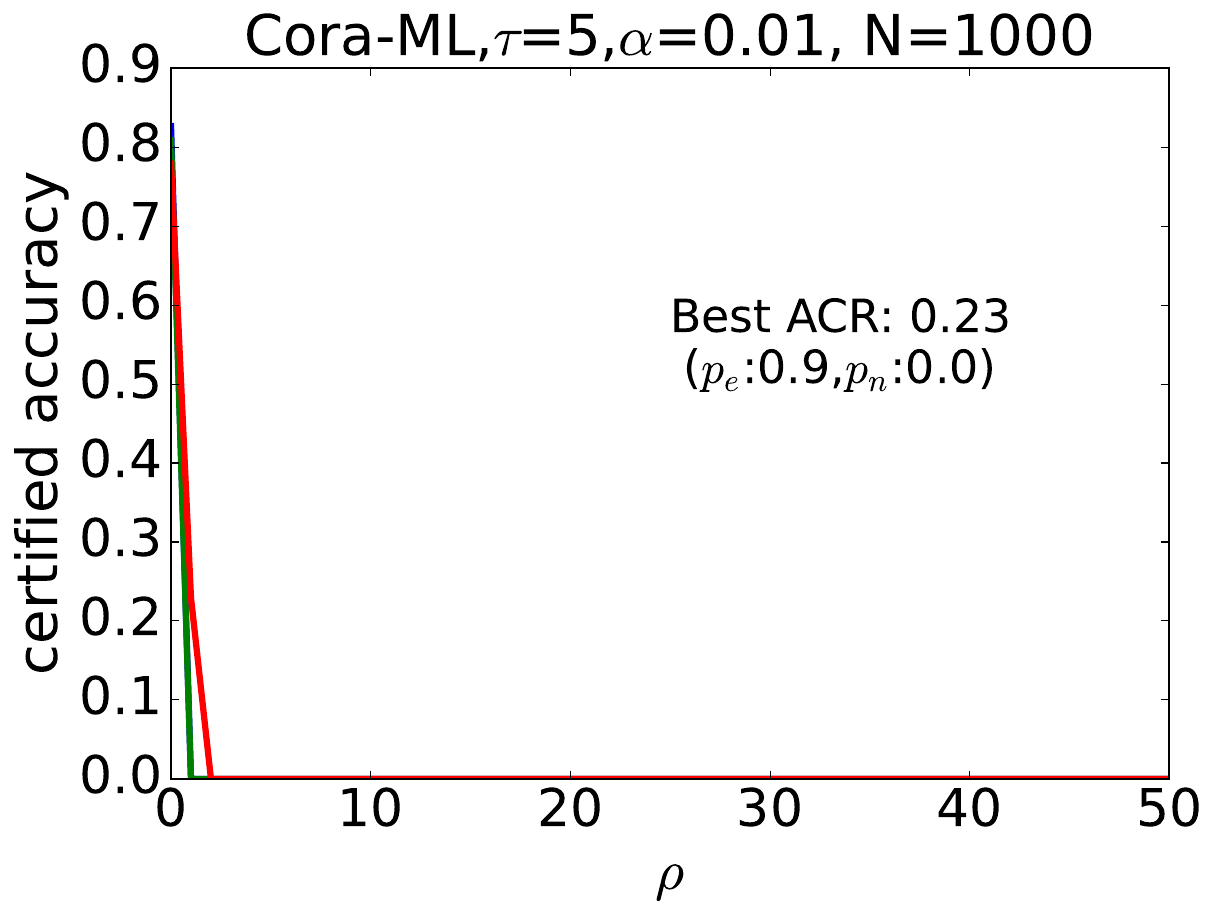}
    }
    \subfigure[Sparsity-aware \cite{bojchevski2020efficient}]{
    \includegraphics[width=0.255\textwidth,height=2.8cm]{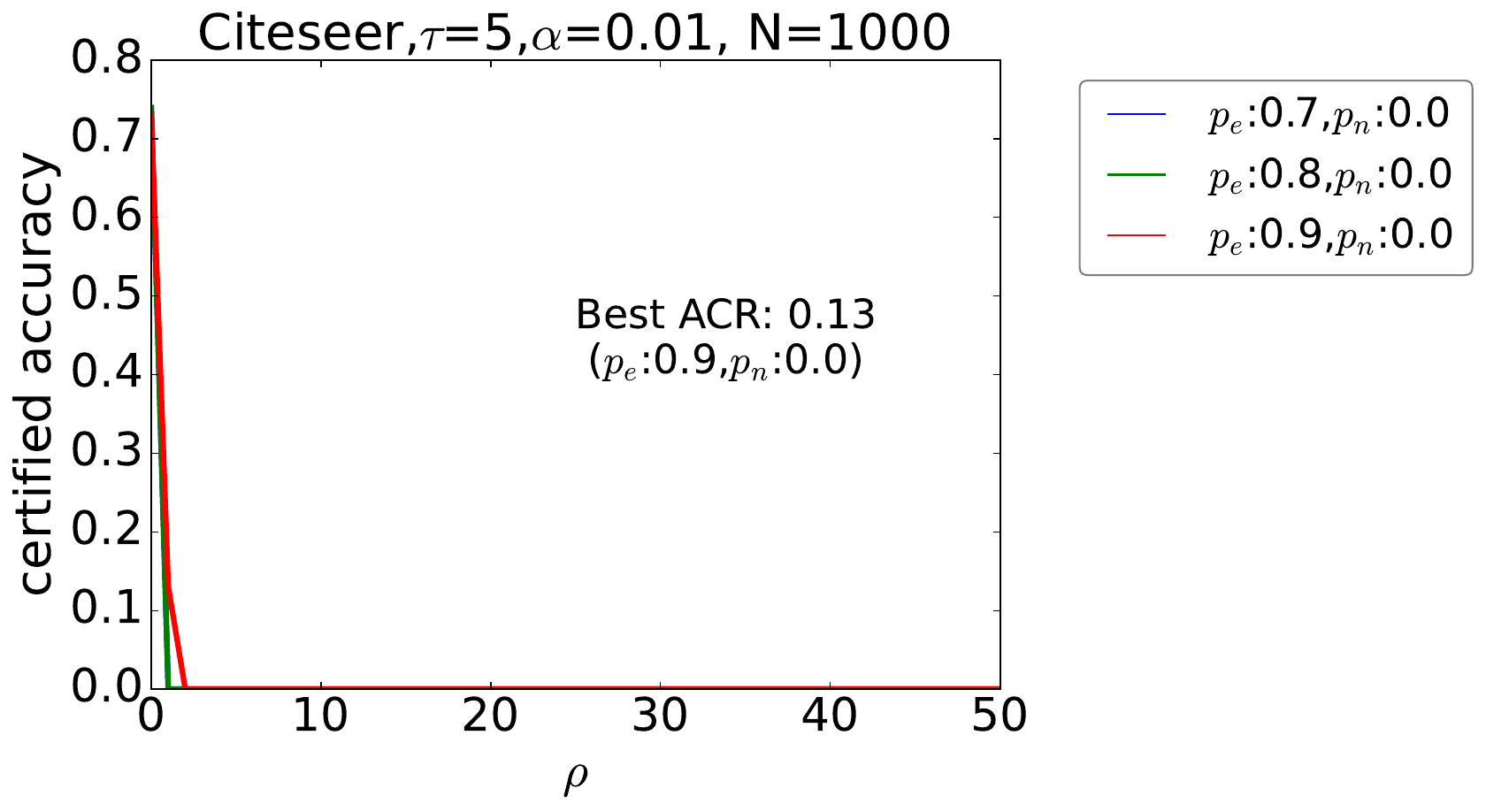}
    }
    \subfigure[Bagging-cert\cite{jia2021intrinsic}]{
    \includegraphics[width=0.18\textwidth,height=2.8cm]{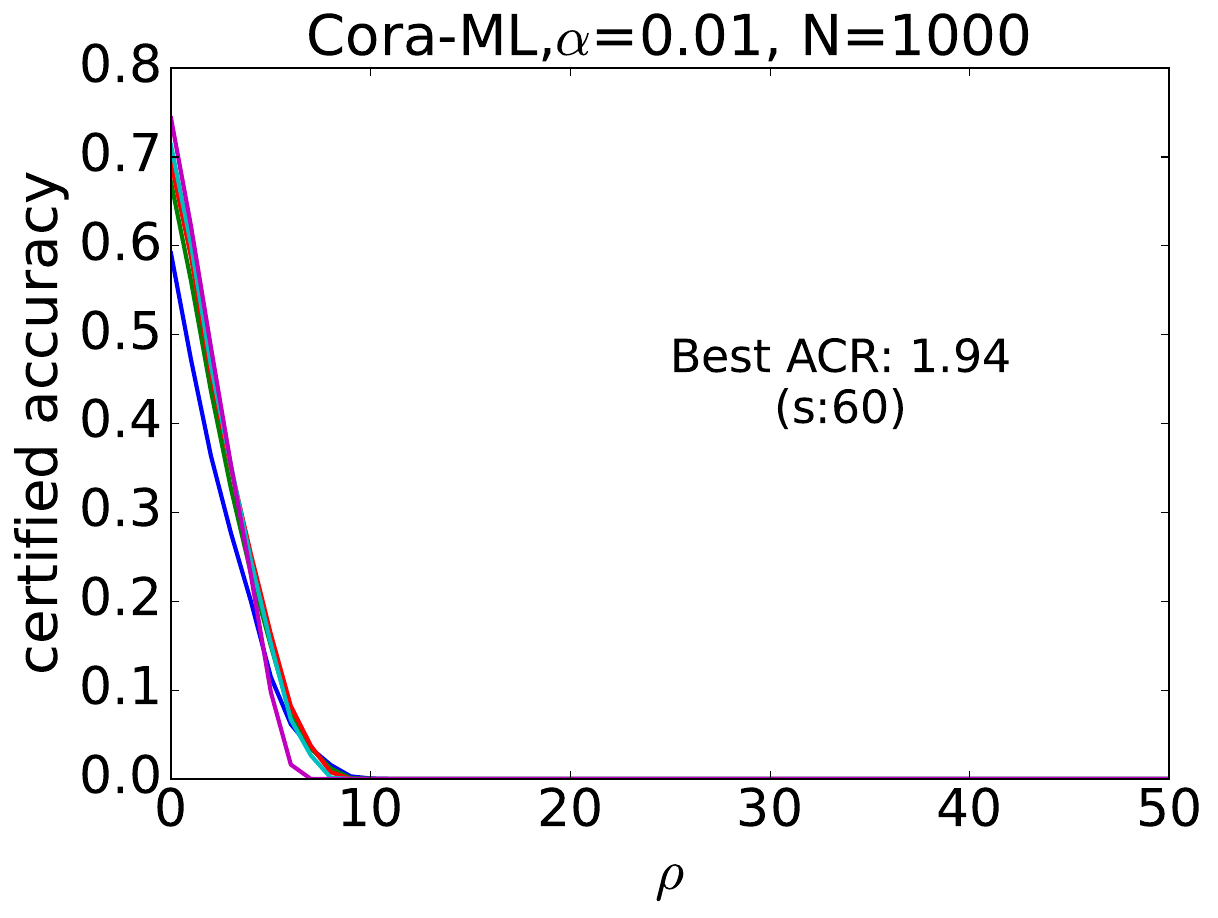}
    }
    \subfigure[Bagging-cert\cite{jia2021intrinsic}]{
    \includegraphics[width=0.255\textwidth,height=2.8cm]{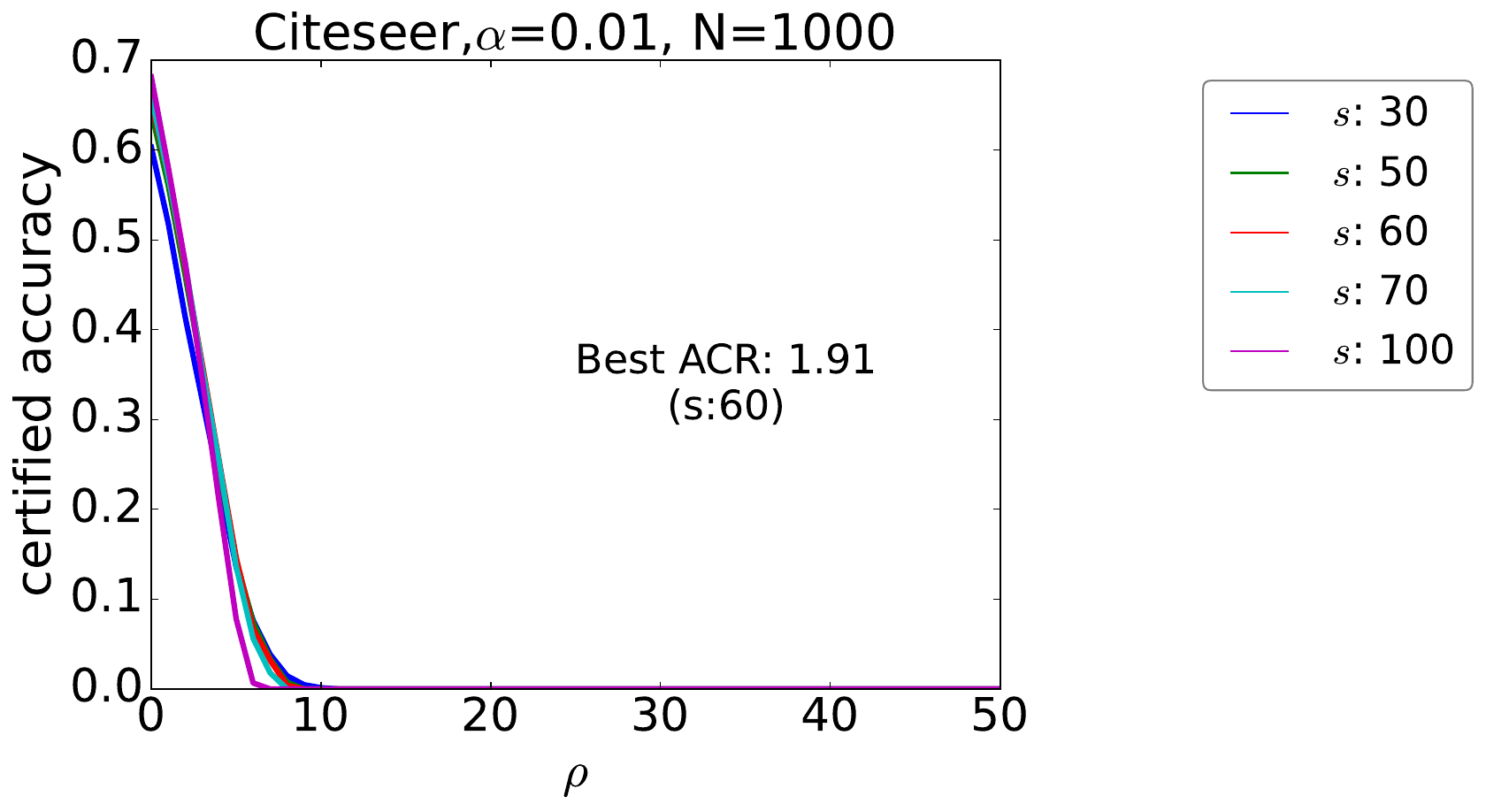}
    }
    \subfigure[\textbf{Node-aware}-include]{
    \includegraphics[width=0.18\textwidth,height=2.8cm]{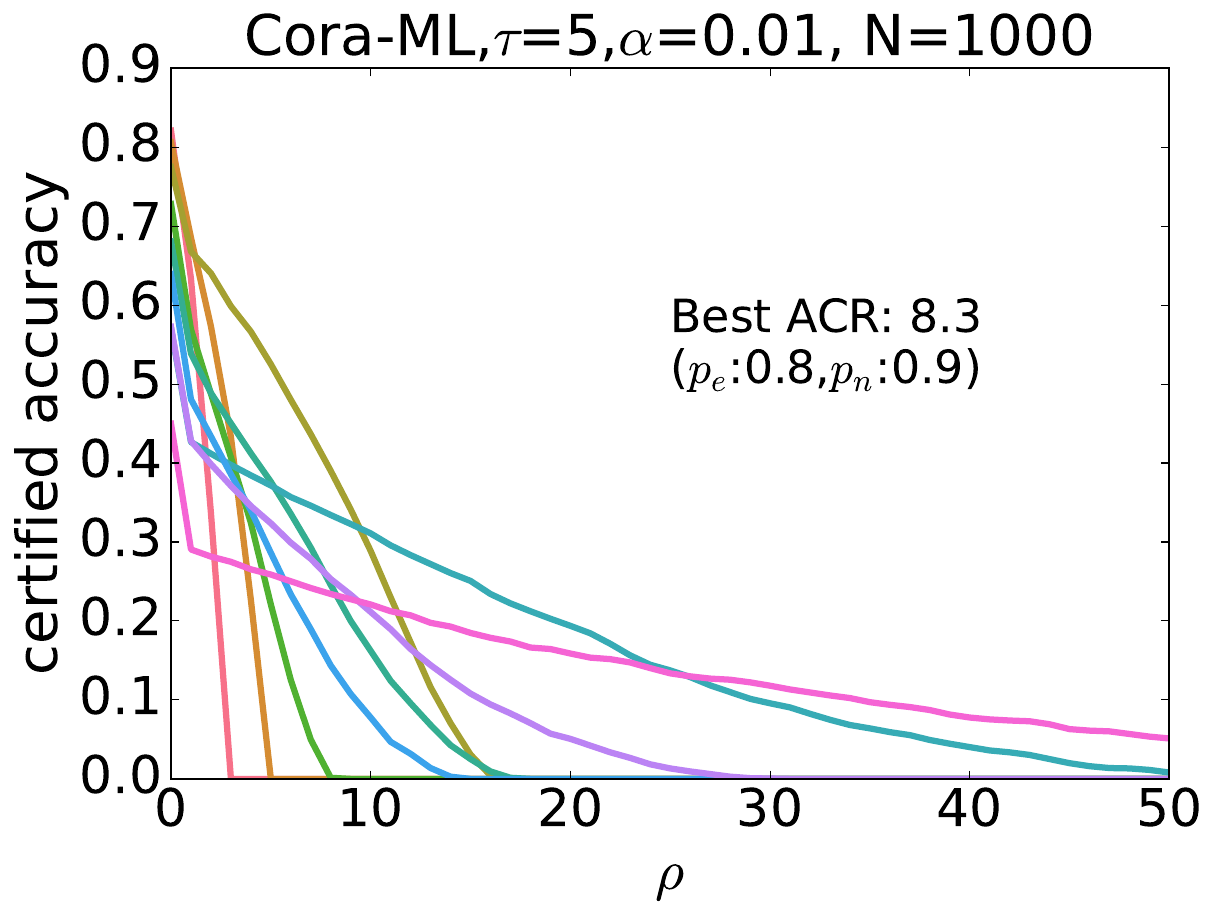}
    }
    \subfigure[\textbf{Node-aware}-include]{
    \includegraphics[width=0.255\textwidth,height=2.8cm]{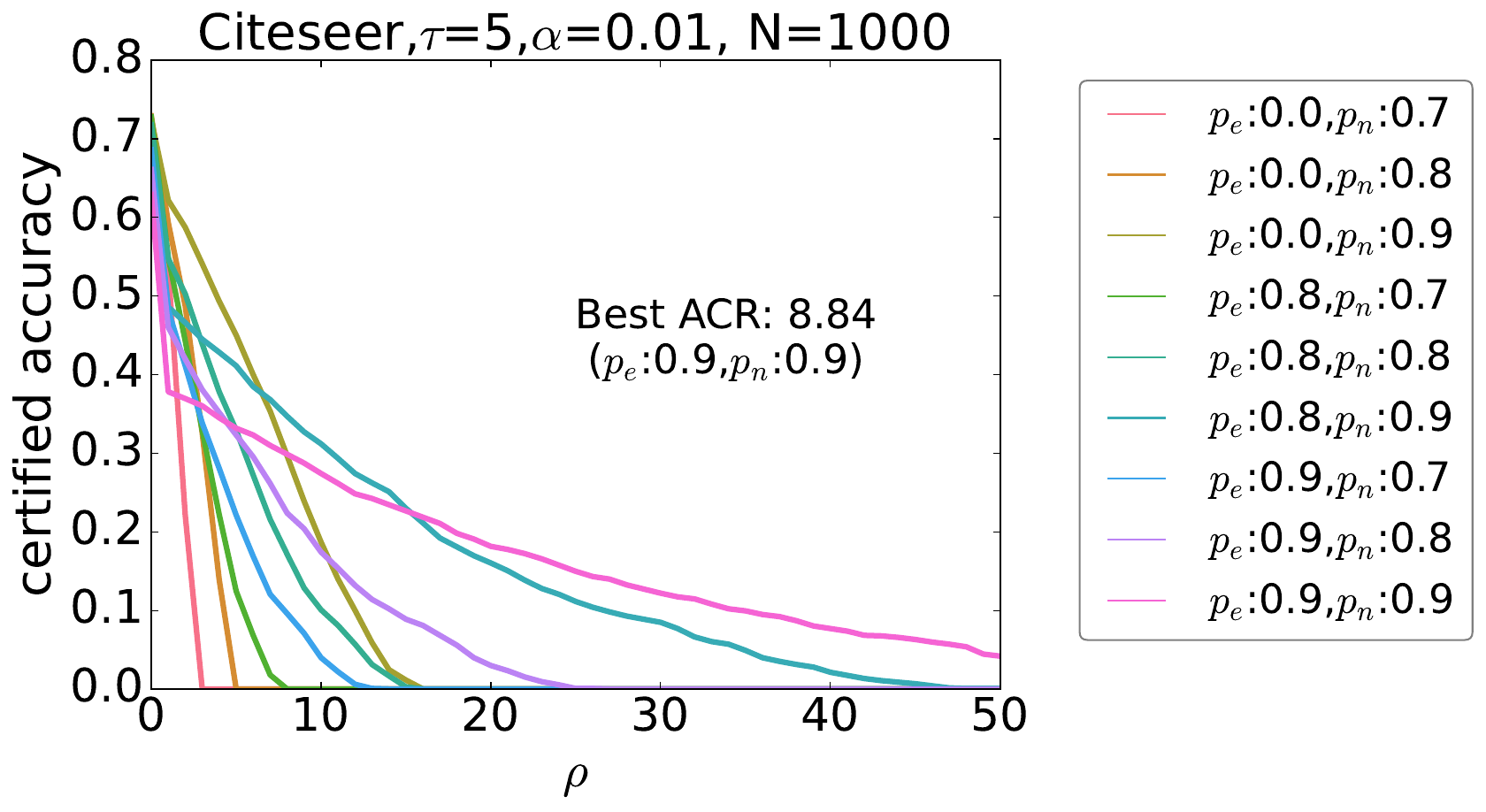}
    }
    \subfigure[\textbf{Node-aware}-exclude]{
    \includegraphics[width=0.18\textwidth,height=2.8cm]{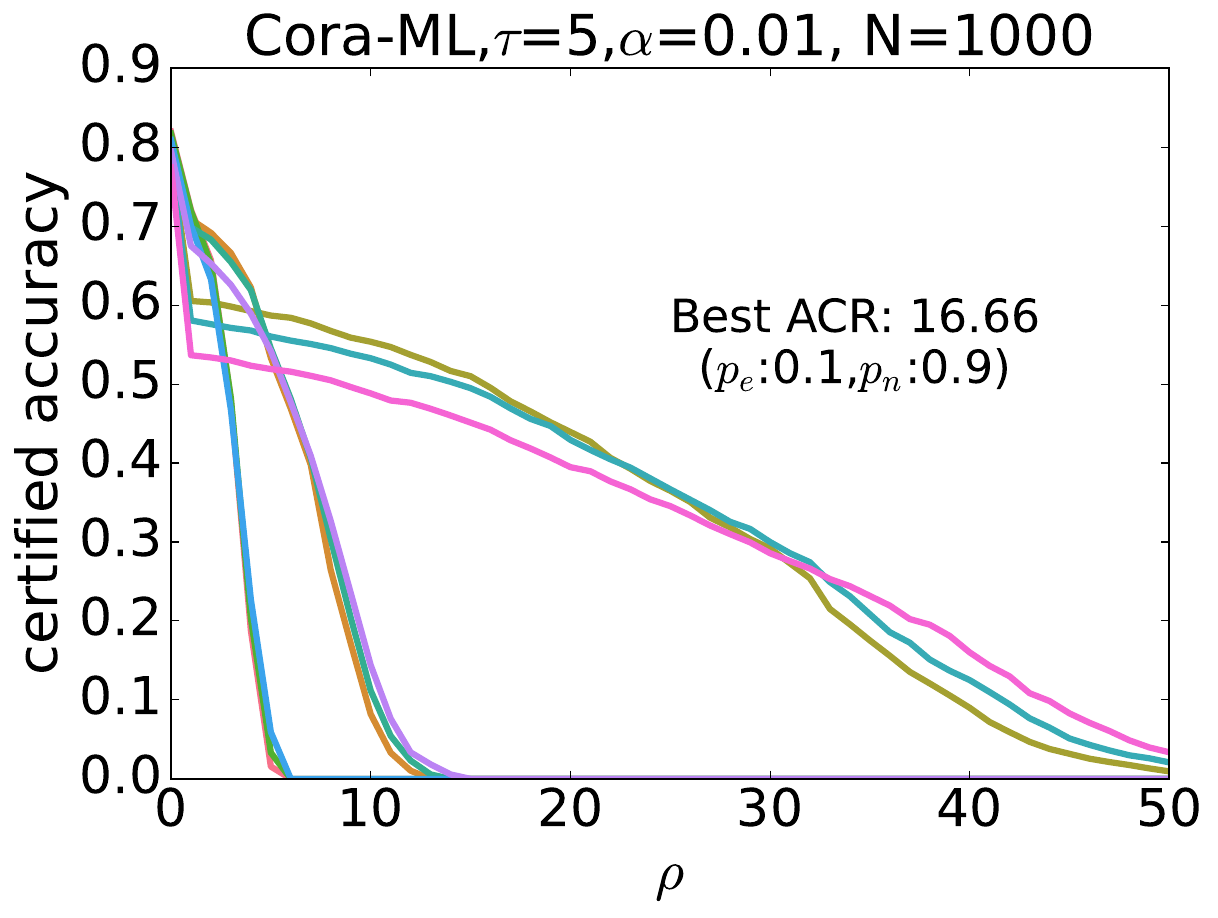}
    }
    \subfigure[\textbf{Node-aware}-exclude]{
    \includegraphics[width=0.255\textwidth,height=2.8cm]{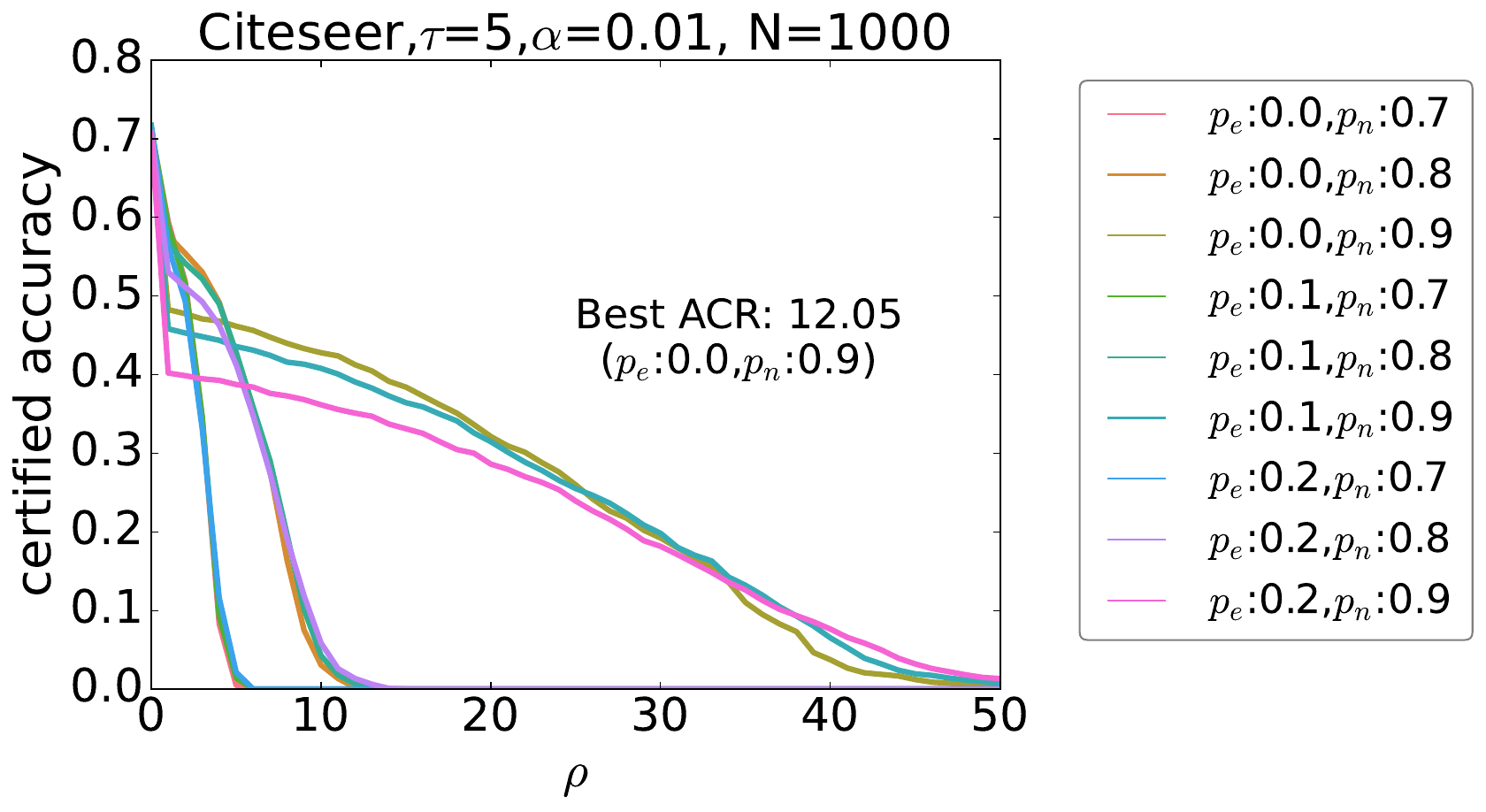}
    }
\caption{Certified accuracy under \textit{poisoning} perturbation with $\tau=5$. $s$ is the bagging size of the model~\cite{jia2021intrinsic}.
The sharp decrease in certified accuracy at the beginning is due to the ABSTAIN for less confident $y_A$.}
\label{fig:cer_poison}
\end{figure}

\begin{table}[hbt!]
\centering
\setlength{\tabcolsep}{3pt}
\caption{Certified precision and recall comparison on recommender system ($\tau=10$).}
\begin{tabular}{clrrrrr}
\hline
\multicolumn{2}{c}{MovieLens-100k} & \multicolumn{4}{c}{$\rho$} & \multirow{2}{*}{ACR} \\ \cline{1-6}
\multicolumn{1}{l}{metrics} & methods & 0 & 3 & 5 & 10 &  \\ \hline
\multirow{2}{*}{\begin{tabular}[c]{@{}c@{}}certified \\ precisioin\end{tabular}} & PORE\cite{jia2023pore} & 0.203 & 0.056 & 0.038 & 0.015 & 0.427 \\
 & \textbf{node-aware-exclude} & \textbf{0.209} & \textbf{0.081} & \textbf{0.053} & \textbf{0.018} & \textbf{0.617} \\ \hline
\multirow{2}{*}{\begin{tabular}[c]{@{}c@{}}certified \\ recall\end{tabular}} & PORE\cite{jia2023pore} & 0.170 & 0.050 & 0.037 & 0.018 & 0.424 \\
 & \textbf{node-aware-exclude} & \textbf{0.174} & \textbf{0.074} & \textbf{0.051} & \textbf{0.021} & \textbf{0.601} \\ \hline
\end{tabular}
\label{tab:RecS_CA}
\end{table}

\begin{figure}[h]
    \centering
    \subfigure[PORE~\cite{jia2023pore}]{
    \includegraphics[width=0.18\textwidth,height=2.8cm]{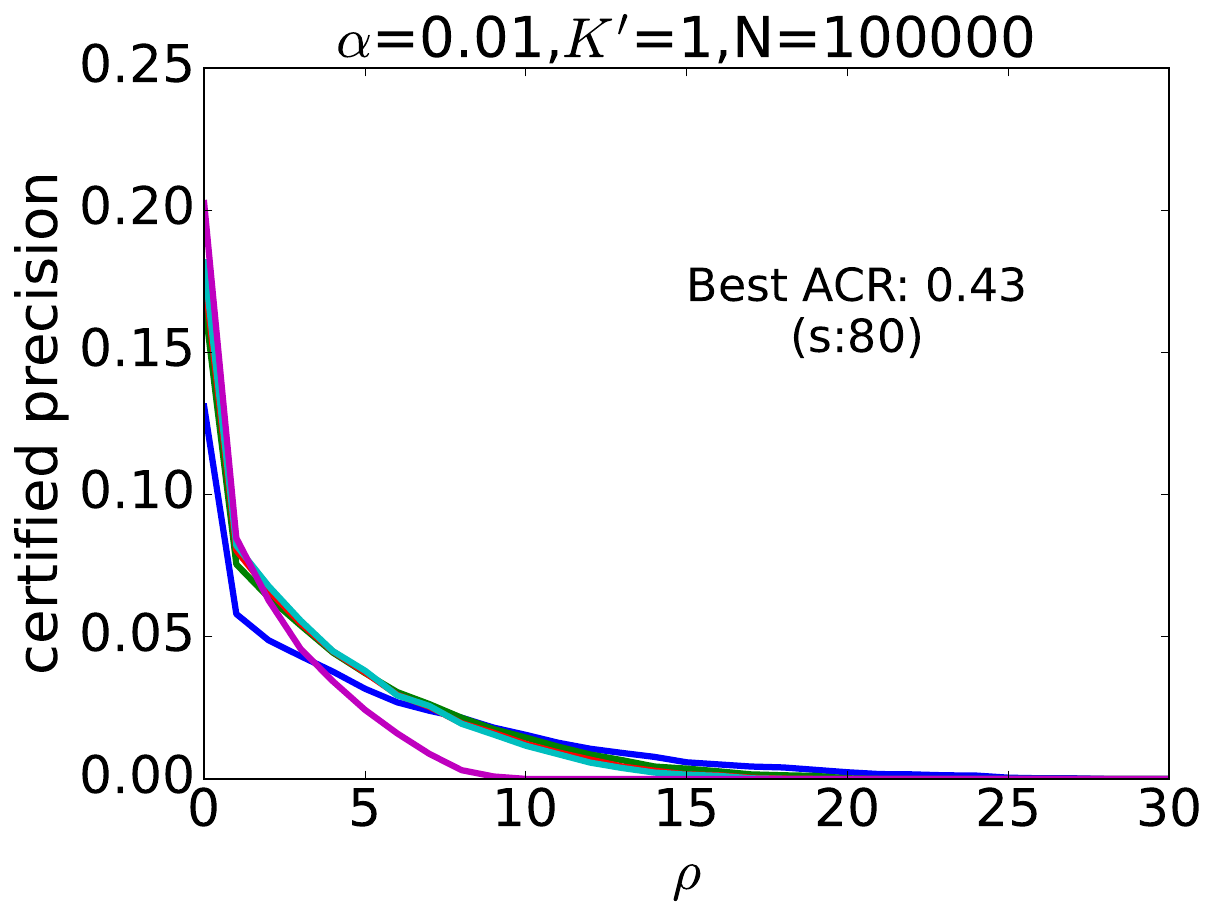}
    }
    \subfigure[PORE~\cite{jia2023pore}]{
    \includegraphics[width=0.255\textwidth,height=2.8cm]{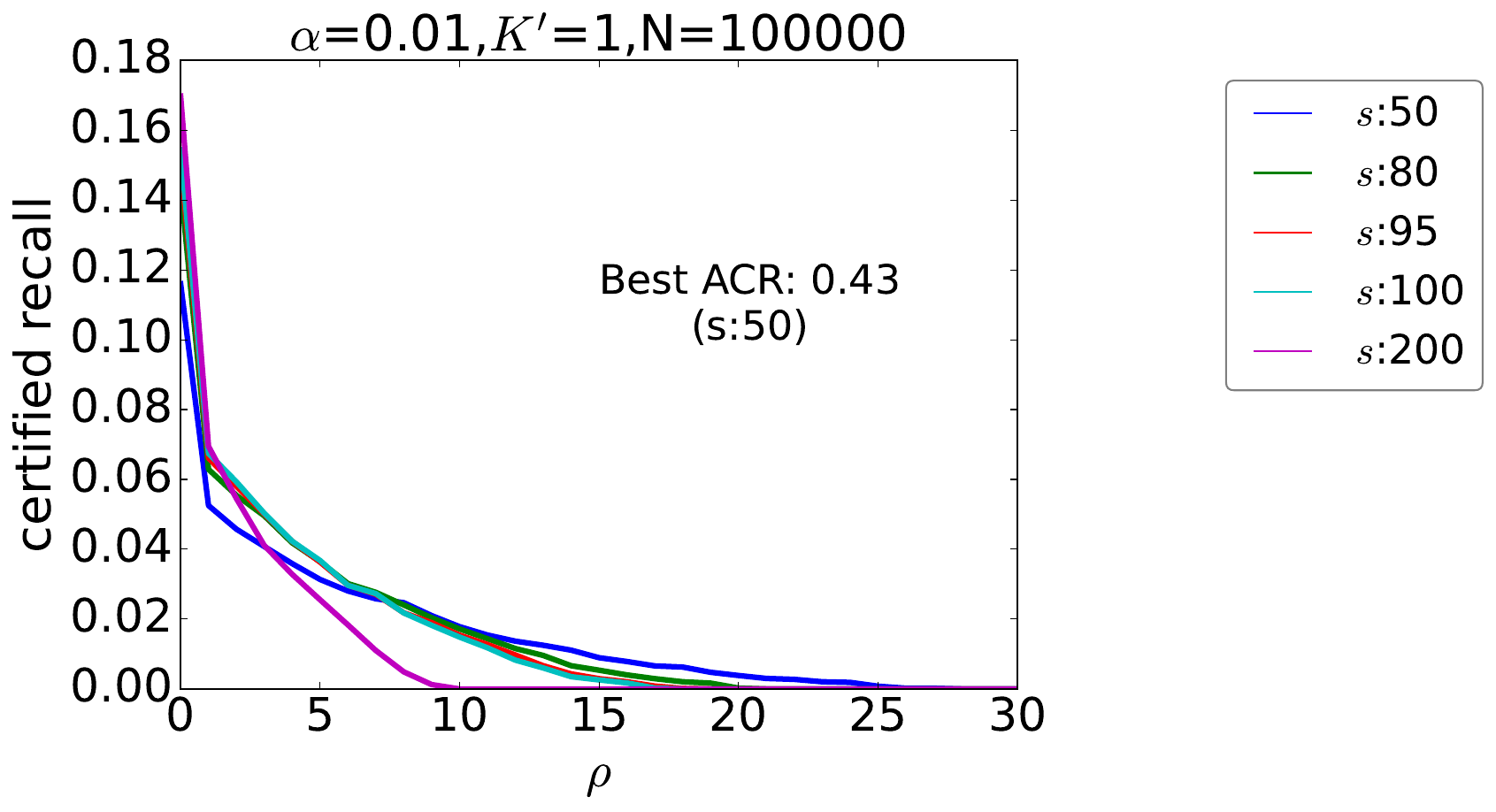}
    }
    \subfigure[\textbf{Node-aware}-exclude]{
    \includegraphics[width=0.18\textwidth,height=2.8cm]{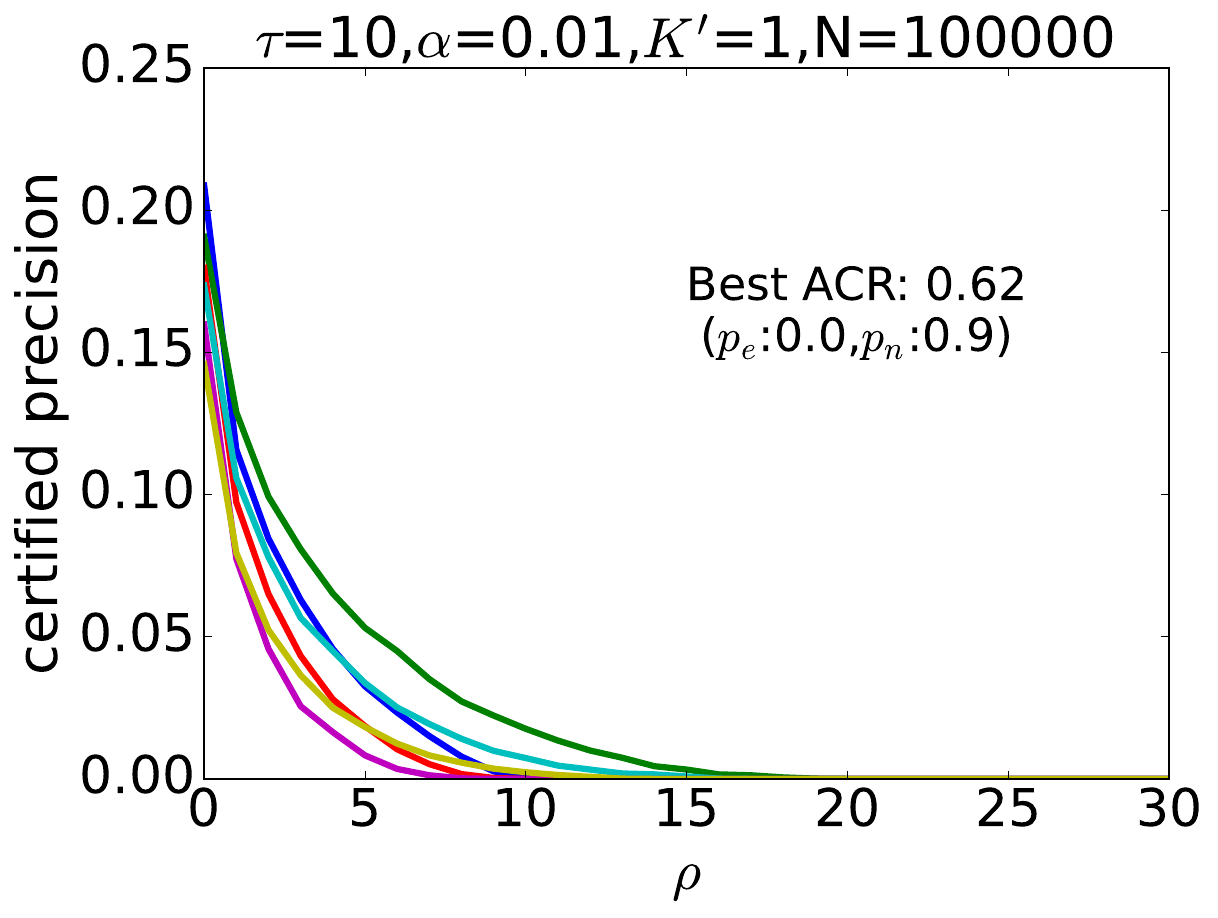}
    }
    \subfigure[\textbf{Node-aware}-exclude]{
    \includegraphics[width=0.255\textwidth,height=2.8cm]{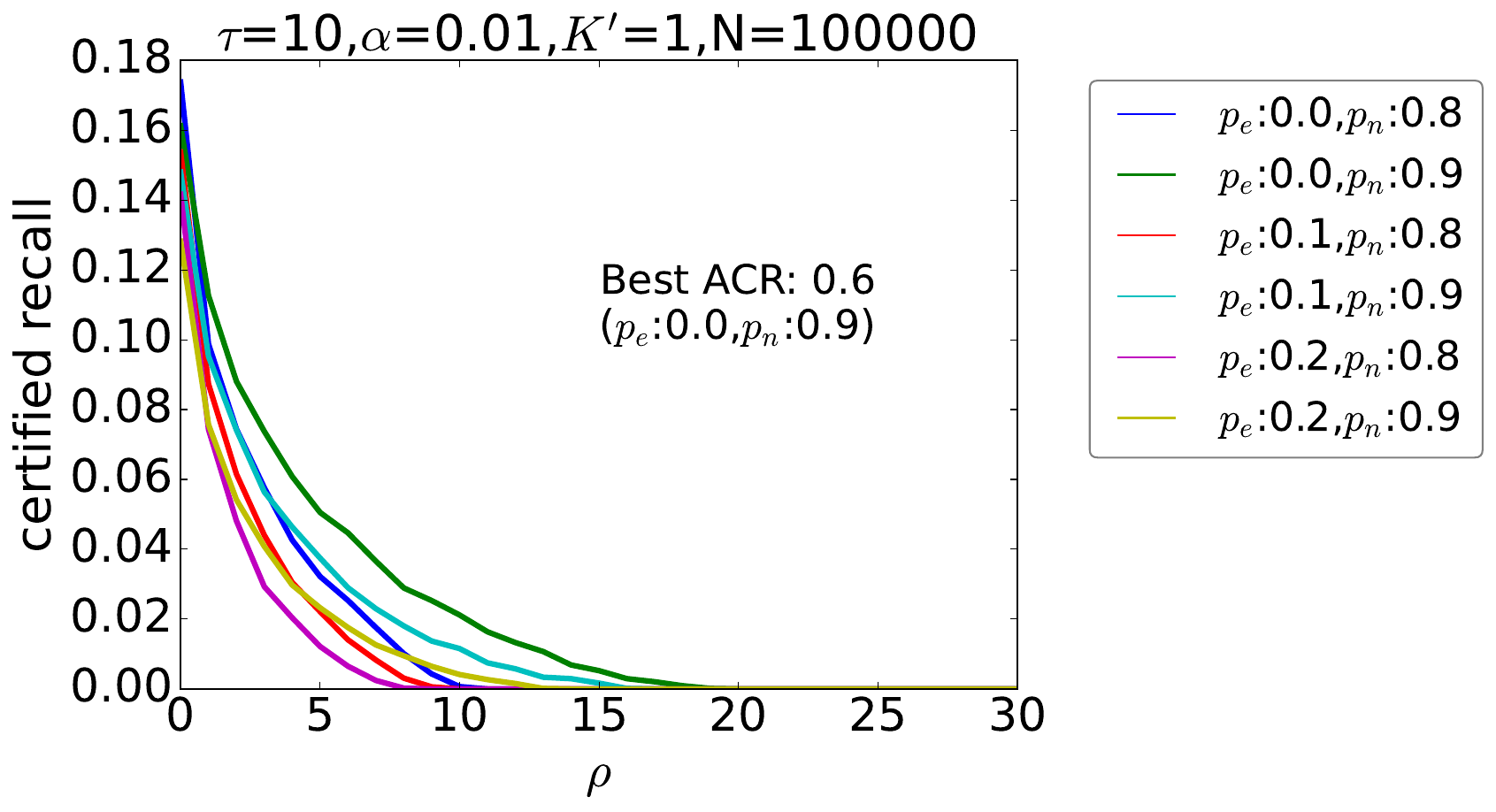}
    }
    \caption{Certified precision and recall on SAR Recommender system (MovieLens-100K dataset, $85\%$ training) under poisoning perturbation, where $s$ is the bagging size of the model PORE~\cite{jia2023pore}, $K'$ is the number of items recommended by the base recommender, and we set $K=10$ as the number of items recommended by the smoothed recommender.}
    \label{fig:cer_RS}
\end{figure}


\subsection{Certifiable Robustness for Recommender System}
Our proposed node-aware-exclude is also applicable to provide provable recommendations. We compare our model to the provable recommender system, PORE~\cite{jia2023pore}. Similar to node classification, our model is capable of considering the restricted attacker with node degree constraint, while PORE can only certify malicious nodes with unlimited degree budgets. Table.~\ref{tab:RecS_CA} and Figure.~\ref{fig:cer_RS} show the certified precision and recall under PORE and our proposed smoothing with degree budget $\tau=10$. Notably, 
our smoothed classifier has observed better certification performance.

\subsection{Empirical Robustness for Node Classification}
In this section, we take a state-of-the-art GIA attacker HAOGIA~\cite{chen2022understanding} as an example to study the empirical robustness of our model. For comparison, we take \textit{four} widely-used GNN defense models, GCN~\cite{kipf2016semi}, ProGNN~\cite{jin2020graph}, RobustGCN\cite{zhu2019robust}, and GNNGuard~\cite{zhang2020gnnguard} as baselines, and evaluate their accuracy under attacks with budgets $\rho=\{10,20,30,40,50\}$ and $\tau=5$. The results are presented in Table.~\ref{tab:empirical}. 

Although our model is primarily designed for certified robustness, it achieves a competitive empirical accuracy.
Notably, the node-aware-include variant \textit{maintains nearly unchanged empirical accuracy even as the attack budget increases}.
In both datasets, we achieve the best empirical defense when the number of injected nodes $\rho\geq 40$. \textit{These results highlight the effectiveness of our model as an empirical defense framework.} Although the node-aware-exclude variant experiences a slight decrease in empirical accuracy compared to \textit{include}, it achieves the best certified accuracy. 
Notably, while our model can provide both empirical and certified robustness, other common defense baselines can only offer empirical robustness without any guarantee.

\begin{table*}
\centering
\caption{Empirical robust accuracy of different defense models under HAOGIA~\cite{chen2022understanding} attack. The baseline models can only provide empirical robustness, while our method offers both empirical and certified robustness. We show the parameters achieve better certified accuracy (the last $3^{th}$ and $4^{th}$ columns) and better empirical accuracy (the last two columns).}
\setlength{\tabcolsep}{2.5pt}
\begin{tabular}{clllllllll}
\hline
\multicolumn{2}{c}{defense   models} & GCN & ProGNN & RobustGCN & GNNGuard & \textbf{node-aware-include} & \textbf{node-aware-exclude} & \textbf{node-aware-include} & \textbf{node-aware-exclude} \\ \hline
\multicolumn{1}{l}{} & attack & \multicolumn{4}{c}{\multirow{2}{*}{empirical robust accuracy}} & ($p_e:0.8, p_n:0.9$) & ($p_e:0.1, p_n:0.9$) & ($p_e:0.1, p_n:0.7$) & ($p_e:0.1, p_n:0.7$) \\ \cline{7-10} 
\multicolumn{1}{l}{dataset} & $\rho$ & \multicolumn{4}{c}{} & empirical (certified) & empirical (certified) & empirical (certified) & empirical (certified) \\ \hline
\multirow{6}{*}{Cora-ML} & clean & {\ul 0.816} & \textbf{0.832} & 0.800 & 0.792 & 0.571 & 0.784 & 0.814 & 0.807 \\
 & 10 & {\ul 0.815} & \textbf{0.831} & 0.800 & 0.788 & 0.560 (0.311) & 0.778 (\textbf{0.533}) & 0.814 (0.000) & 0.811 (0.000) \\
 & 20 & {\ul 0.815} & \textbf{0.830} & 0.802 & 0.790 & 0.551 (0.194) & 0.776 (\textbf{0.429}) & 0.814 (0.000) & 0.809 (0.000) \\
 & 30 & {\ul 0.815} & \textbf{0.816} & 0.791 & 0.782 & 0.550 (0.096) & 0.775 (\textbf{0.300}) & 0.813 (0.000) & 0.785 (0.000) \\
 & 40 & 0.775 & 0.791 & 0.775 & 0.756 & 0.546 (0.040) & 0.770 (\textbf{0.125}) & \textbf{0.813} (0.000) & {\ul0.801} (0.000) \\
 & 50 & 0.764 & 0.771 & 0.763 & 0.745 & 0.543 (0.008) & 0.762 (\textbf{0.021}) & \textbf{0.808} (0.000) & {\ul0.793} (0.000) \\ \hline
\multicolumn{1}{l}{} & $\rho$ & \multicolumn{4}{c}{empirical robust accuracy} & ($p_e:0.8, p_n:0.9$) & ($p_e:0.1, p_n:0.9$) & ($p_e:0.6, p_n:0.7$) & ($p_e:0.7, p_n:0.7$) \\ \hline
\multirow{6}{*}{Citeseer} & clean & 0.700 & 0.719 & 0.702 & 0.668 & 0.675 & 0.714 & \textbf{0.736} & {\ul 0.732} \\
 & 10 & 0.695 & 0.707 & 0.688 & 0.657 & 0.614 (0.312) & 0.700 (\textbf{0.408}) & \textbf{0.728} (0.000) & {\ul 0.724} (0.002) \\
 & 20 & 0.685 & 0.681 & 0.683 & 0.649 & 0.611 (0.160) & 0.685 (\textbf{0.315}) & \textbf{0.732} (0.000) & {\ul 0.714} (0.000) \\
 & 30 & 0.647 & 0.673 & 0.654 & 0.623 & 0.607 (0.085) & 0.693 (\textbf{0.198}) & \textbf{0.730} (0.000) & {\ul 0.711} (0.000) \\
 & 40 & 0.638 & 0.648 & 0.638 & 0.610 & 0.603 (0.021) & 0.681 (\textbf{0.065}) & \textbf{0.730} (0.000) & {\ul 0.709} (0.000) \\
 & 50 & 0.629 & 0.611 & 0.618 & 0.615 & 0.600 (0.001) & 0.677 (\textbf{0.007}) & \textbf{0.729} (0.000) & {\ul 0.706} (0.000) \\ \hline
\end{tabular}
\label{tab:empirical}
\end{table*}

\subsection{Ablation Study and Hyper-parameters}

\begin{figure}[hbt!]
\centering
    \subfigure[$\underline{p_A}$]{
    \includegraphics[width=0.220\textwidth,height=2.8cm]{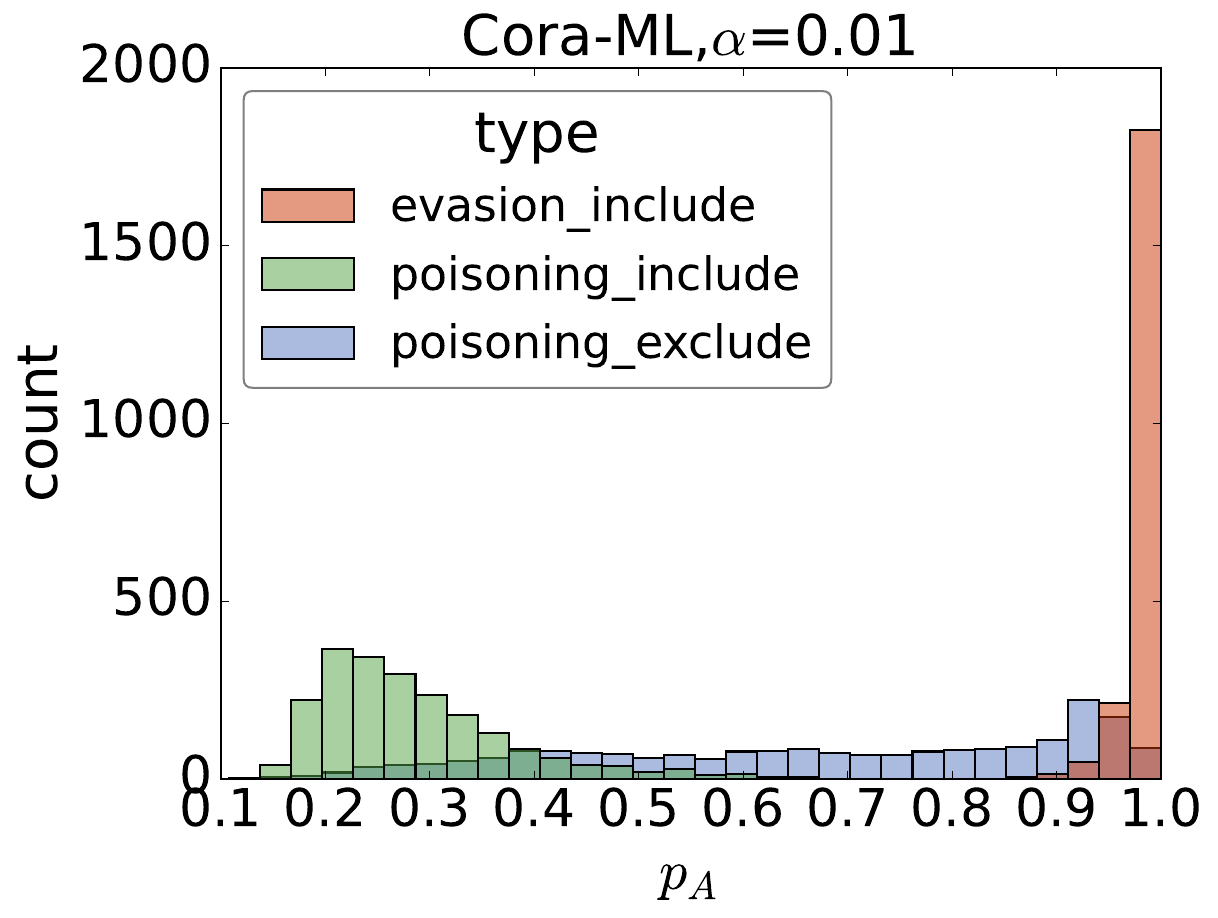}
    }
    \subfigure[$\underline{p_A}$]{
    \includegraphics[width=0.220\textwidth,height=2.8cm]{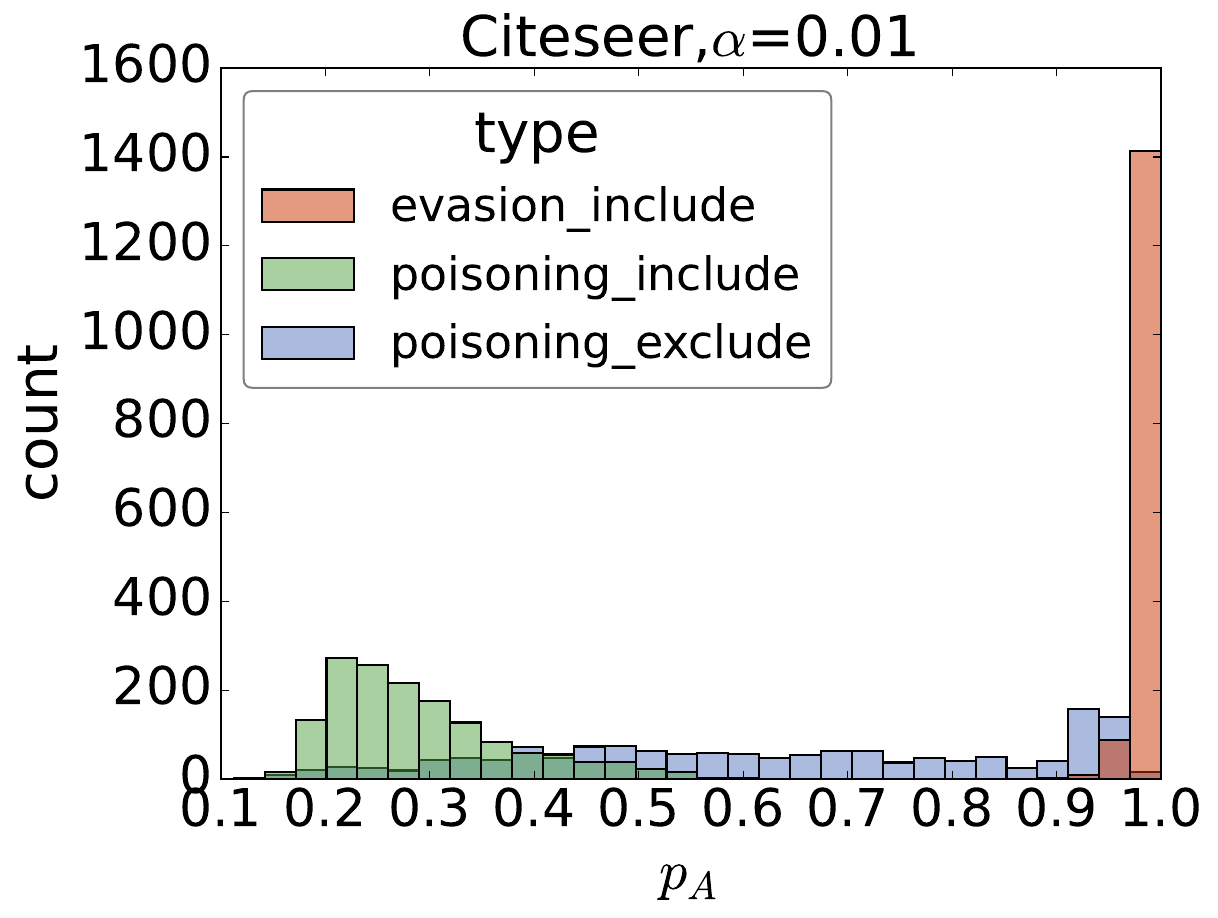}
    }
\caption{Histograms of $\underline{p_A}$ under different types of model.}
\label{fig:pA_pB}
\end{figure}

To further analyze the important factor of the effectiveness of our proposed certifying scheme, we study our node-aware bi-smoothing with a single smoothing distribution and compare the node-aware-include with node-aware-exclude. 

\subsubsection{Edge Deletion and Node Deletion smoothing}
The ablation studied with the smoothing node-deletion only ($p_e=0$) and edge-deletion only ($p_n=0$) are shown in Figure.~\ref{fig:cer_evasion} and \ref{fig:cer_poison}. Note that, the baseline~\cite{bojchevski2020efficient} corresponds to the edge-deletion only ($p_n=0$).

In the case of edge-deletion only ($p_n=0$), we observe that the certifying performance improves as $p_e$ increases, particularly in the node-aware-include strategy where the base model also votes for isolated nodes.
In node-aware-\textit{exclude} (Figure.~\ref{fig:cer_poison}g,\ref{fig:cer_poison}h), 
we achieve the highest Average Certifiable Radius (ACR) with small values of $p_e$ ($0.0$ and $0.1$). 
This phenomenon can be attributed to a high $p_e$ resulting in a large number of isolated nodes that the model does not provide votes for. 
Consequently, the model has a less confident $\underline{p_A}$ and a higher ratio of ABSTAIN in our statistical testing for $y_A$ due to the limited sample size. Nevertheless, this issue can be mitigated as $N$ increases (Figure.\ref{fig:cer_poison_N}), where $p_e>0$ consistently achieves better performance. In the case of node-deletion only ($p_e=0$), we observe an increasing performance as $p_n$ increases (Figure.\ref{fig:cer_evasion}c,\ref{fig:cer_evasion}d and Figure.\ref{fig:cer_poison}e,\ref{fig:cer_poison}f). 

These ablation studies clearly demonstrate the importance of both edge-deletion and node-deletion smoothing techniques, and the latter has a more significant impact, which is further supported by the observations in Figure.~\ref{fig:hyper_heat}.


\subsubsection{Comparing node-aware-include and node-aware-exclude}
When comparing the \textit{include} strategy with the \textit{exclude} strategy in the poisoning attack scenario (Table.\ref{tab:poison_CA} and Figure.\ref{fig:cer_poison}), we observe that the \textit{exclude} strategy enhances the performance in terms of the ACR and certified accuracy. This improvement is primarily attributed to the higher confidence in $p_A$ achieved by the \textit{exclude} strategy (Figure.~\ref{fig:pA_pB}).

\subsubsection{Hyper-parameters Analysis}

\begin{figure}[hbt!]
    \centering
    \includegraphics[width=0.25\textwidth,height=3.6cm]{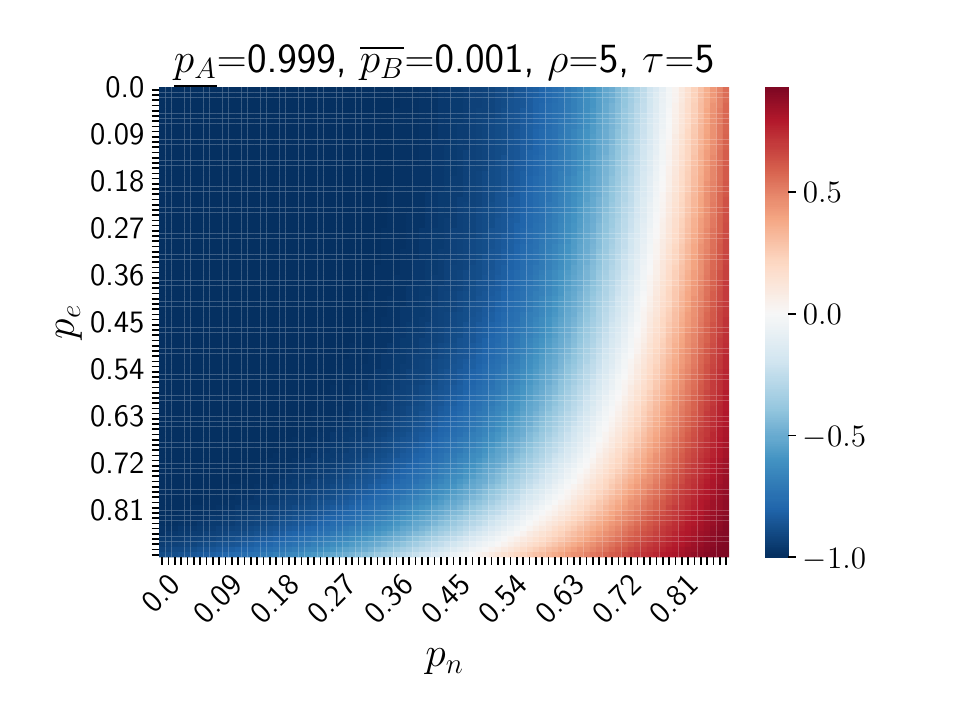}
    \caption{The impact of smoothing-parameters ($p_e$, $p_n$) on $\mu_{\rho,\tau}$ (\textit{include}) under sufficiently large $\underline{p_A}$. This figure shows that both node deletion and edge deletion smoothing play an important role in the certifying condition $\mu_{\rho,\tau}>0$ (red). }
    \label{fig:hyper_heat}
\end{figure}

\begin{figure}[hbt!]
\centering
    \subfigure[\textbf{Node-aware}-include]{
    \includegraphics[width=0.18\textwidth,height=2.8cm]{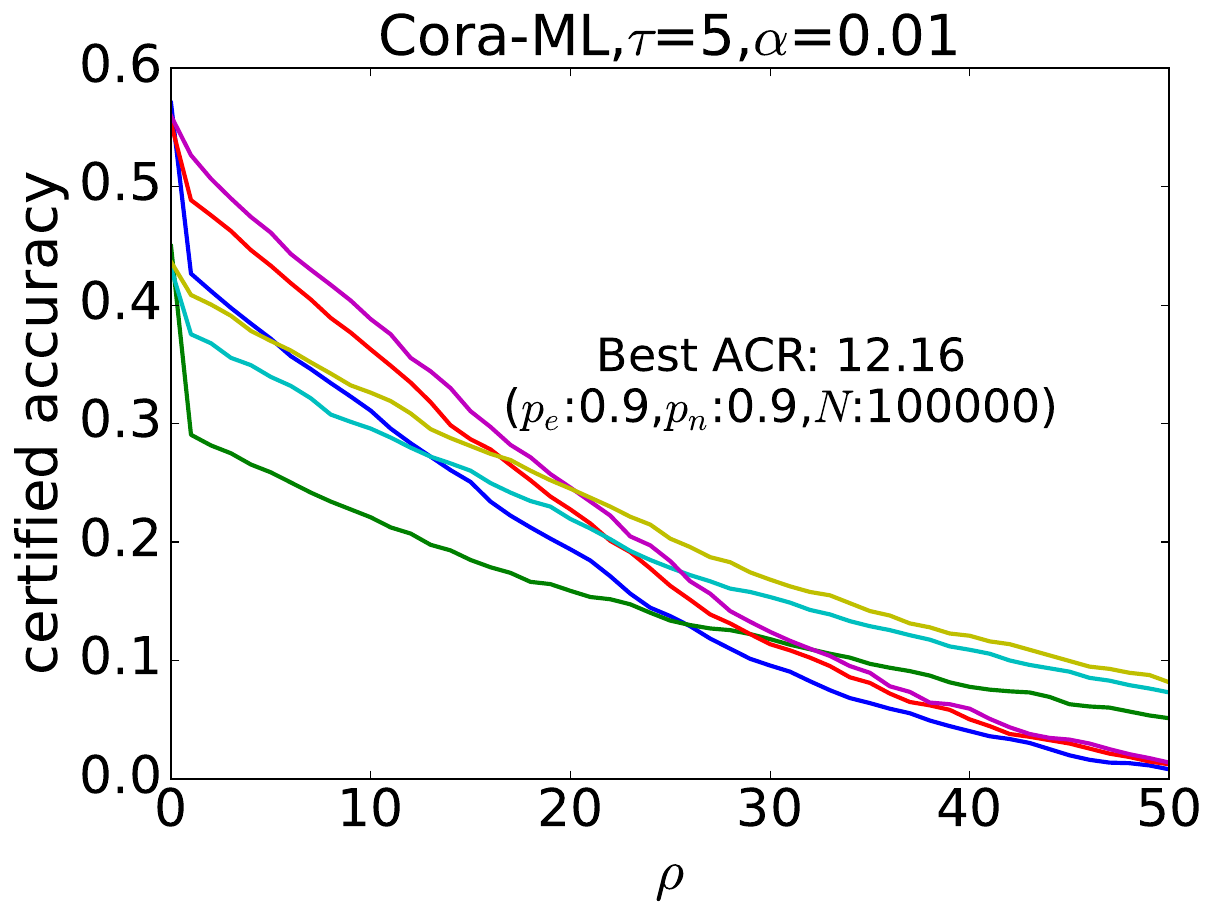}
    }
    \subfigure[\textbf{Node-aware}-include]{
    \includegraphics[width=0.255\textwidth,height=2.8cm]{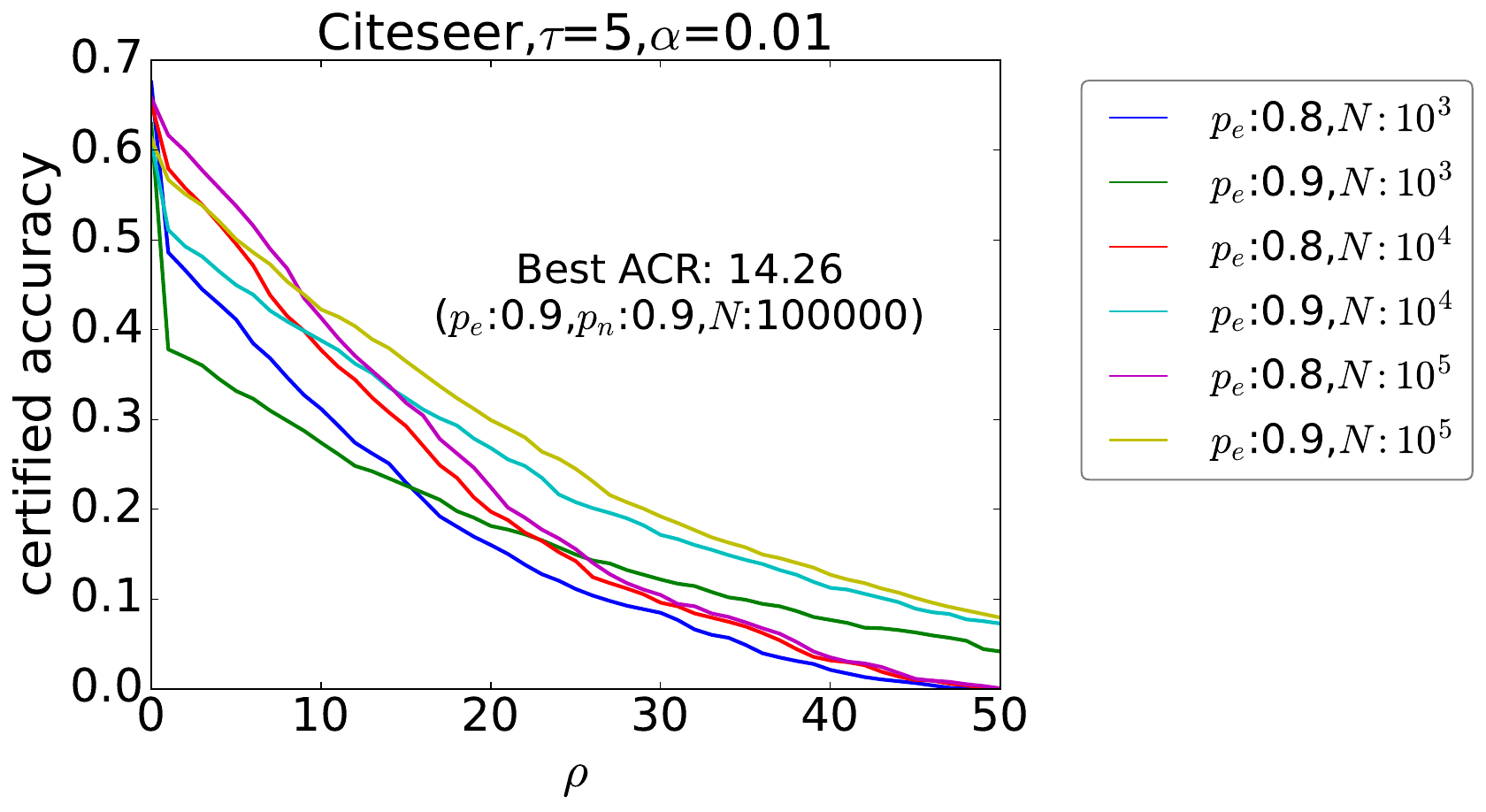}
    }
    \subfigure[\textbf{Node-aware}-exclude]{
    \includegraphics[width=0.18\textwidth,height=2.8cm]{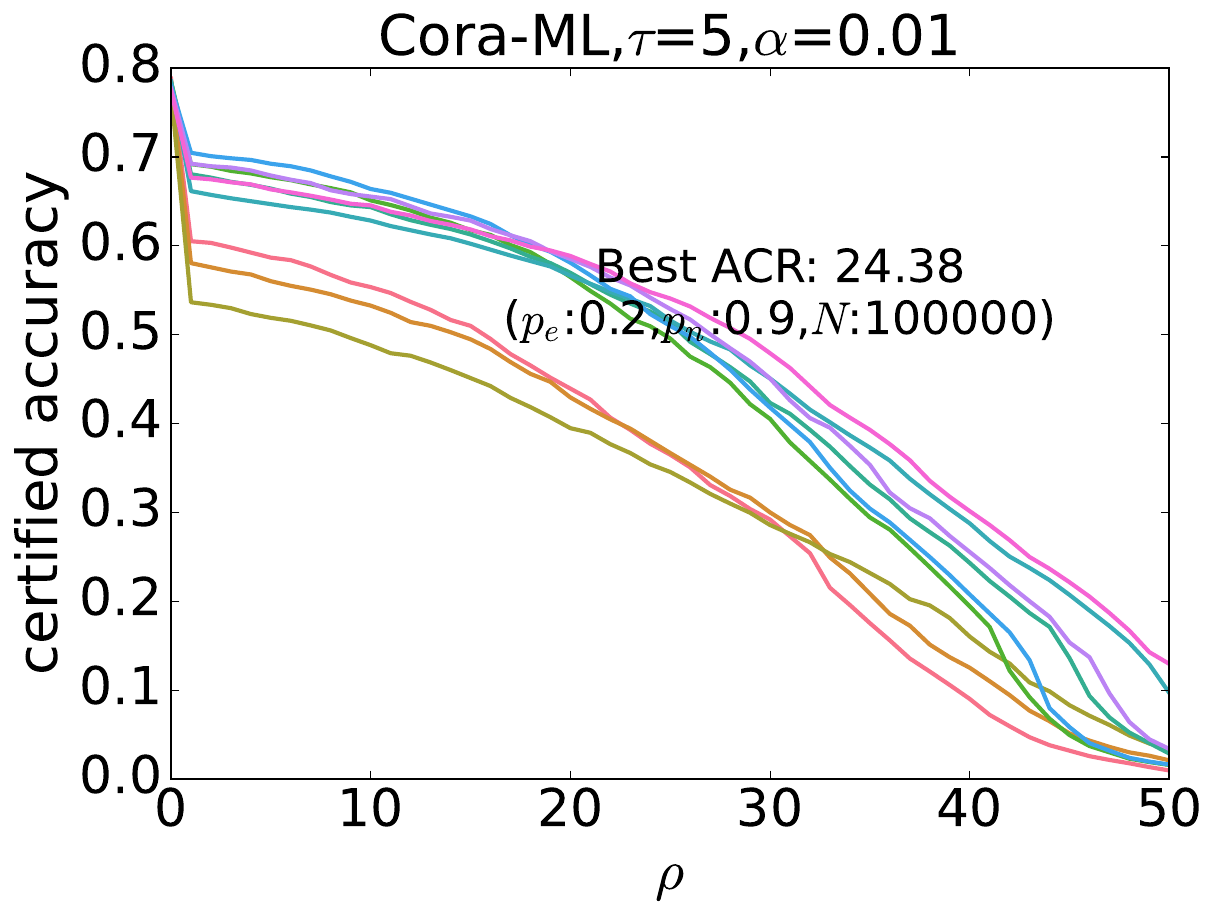}
    }
    \subfigure[\textbf{Node-aware}-exclude]{
    \includegraphics[width=0.255\textwidth,height=2.8cm]{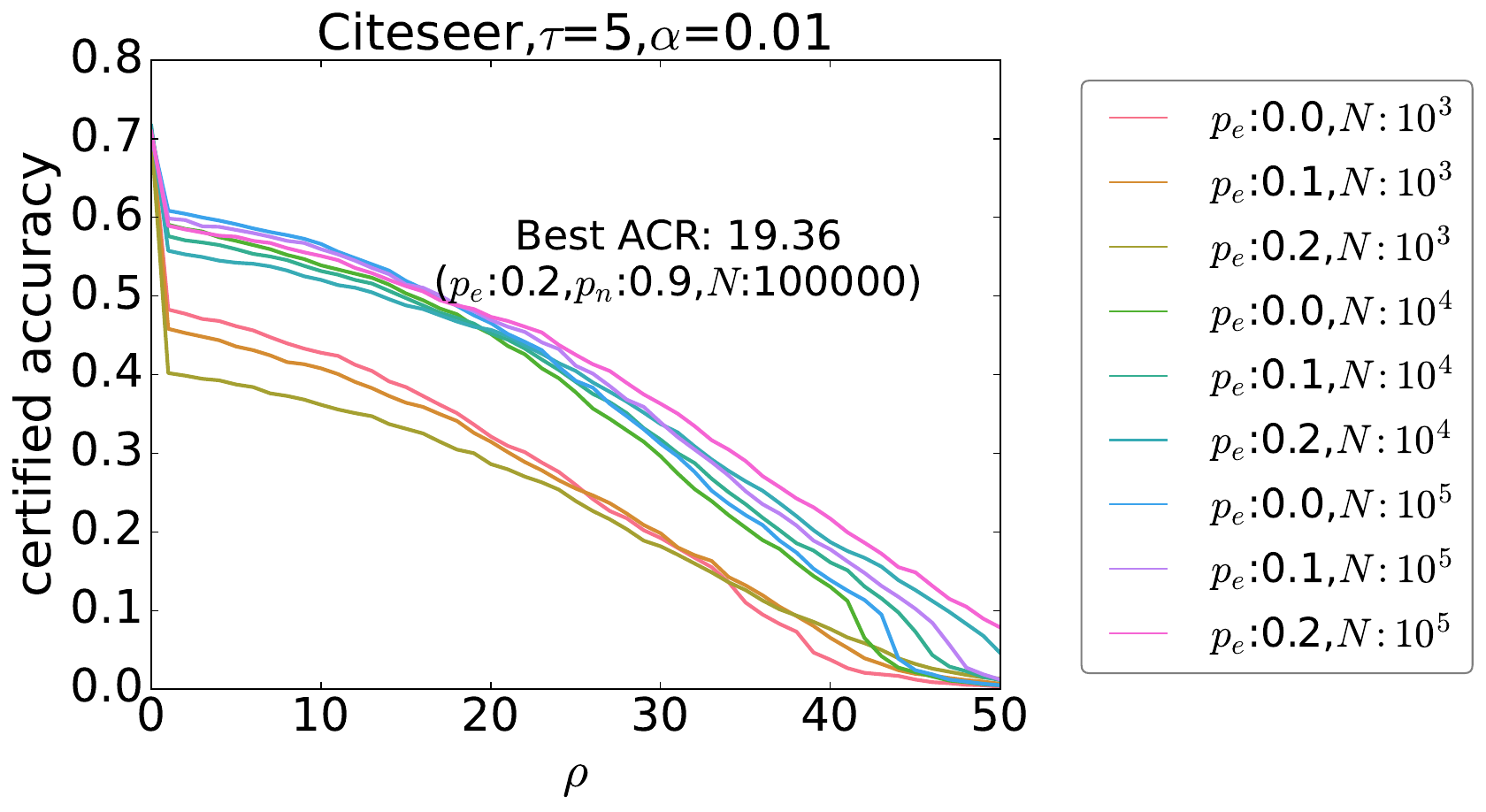}
    }
\caption{Impact of $N$ on certified accuracy under \textit{poisoning} perturbation with $p_n=0.9$, $\tau=5$.}
\label{fig:cer_poison_N}
\end{figure}


Figure.~\ref{fig:hyper_heat} provides a visualization of the impact of $p_e$ and $p_n$ under the same $\underline{p_A}$ and $\overline{p_B}$. It is evident that larger values of $p_e$ and $p_n$ correspond to higher certifying margins $\mu_{\rho,\tau}$, highlighting the crucial role played by both edge-deletion and node-deletion smoothing techniques.

Additionally, we evaluate the effect of varying numbers of Monte-Carlo samples $N$ in Figure.~\ref{fig:cer_poison_N}. Notably, as the value of $N$ increases, the abstain rate decreases significantly, leading to improved certified accuracy and ACR.

\section{Related Work}
\label{Sec:RelatedW}
This section discusses the current certifiable defense framework under \textit{evasion} and \textit{poisoning} attack scenarios. 

Certified robustness is a research topic that aims to provide guarantees for the stability and reliability of machine learning models, especially against adversarial attacks. It can be divided into two main categories: exact verification and randomized smoothing. The exact verification is usually designed for a specific model, while the randomized smoothing is more scalable and flexible for any graph model. Our research falls into the latter one. 

Randomized smoothing methods try to estimate the probability of the model's output under a given perturbation set, using randomization such as Gaussian noise~\cite{cohen2019certified,lecuyer2019certified}.
However, most of the existing works focus on computer vision, especially image classification~\cite{li2022double,fischer2021scalable,jia2022almost,levine2021deep,levine2020robustness,scholten2023hierarchical}, or tabular data~\cite{horvath2022randomized}, and few of them consider the challenges and opportunities of certified robustness for graph data, such as graph classification, node classification, link prediction, etc. 

Some recent works on certified robustness for graph data mainly focus on defending graph modification attacks (GMAs) in \textit{evasion} settings~\cite{bojchevski2020efficient,wang2021certified,jia2020certified,schuchardt2023localized}. Wang et al. ~\cite{wang2021certified} is the first to extend randomized smoothing to graph data and establish the certified robustness guarantee of any GNN for both node and graph classifications. In particular, it employs Bernoulli distribution to randomize the graph data, that every edge has a fixed probability of being deleted or kept. Jia et al.~\cite{jia2020certified} employs a similar scheme to investigate the provable robustness of graph community detection. Bojchevski et al.~\cite{bojchevski2020efficient} improves the calculation efficiency of \cite{wang2021certified,jia2020certified} by reducing the number of constant likelihood ratio region partition, and it can handle perturbations on both
the graph structure and the node attributes. Schuchardt et al.~\cite{schuchardt2020collective,schuchardt2023localized} further improve the sparsity-aware smoothing~\cite{schuchardt2023localized} by designing collective robustness by considering a more realistic and restricted attacker that can only use one perturbed input at once to disrupt as much as node predictions in the graph. 

Despite the recent research progress of certified robustness on graph data, the provable robustness against poisoning attacks on graphs remains mostly unexplored. 
There are mainly two views providing the robustness
certificates for poisoning attacks on image data: 1) Randomized smoothing framework based on ~\cite{cohen2019certified} can also handle poisoning attacks by
regarding the training-prediction procedure as a whole end-to-end function. 2) Aggregation-based technique that partitions the training data into subsets.  
Rosenfeld et al.~\cite{rosenfeld2020certified} is the first to extend randomized smoothing on label-flipping poisoning attacks. Wang et al.~\cite{wang2020certifying} and Weber et al.~\cite{weber2023rab} extend randomized smoothing to defend against backdoor (poisoning) attacks. However, \cite{wang2020certifying,rosenfeld2020certified,weber2023rab} is limited to data
poisoning attacks that only modify existing training examples. Levine et al.~\cite{levine2021deep} proposed an aggregation-based certifying scheme against poisoning attacks, including label-flipping, sample insertion, and deletion. The key insight is that removing a training sample or
adding a new sample will only change the contents of one partition. Similarly, Jia et al.~\cite{jia2021intrinsic} provided certifiable robustness of aggregation ensembles model via the Neyman-Pearson lemma. The only existing certifiable model against poisoning attack on graph data is PORE~\cite{jia2023pore} proposed by Jia et al., which generalized \cite{jia2021intrinsic} to establish the first certifying framework for the recommender system. 

These works mentioned above are all orthogonal to ours. 
Although \cite{bojchevski2020efficient,wang2021certified,jia2020certified,schuchardt2023localized} studied the certifiable robust model for graph data against evasion attack, they are designed for graph modification attack (GMA). Moreover, \cite{rosenfeld2020certified,wang2020certifying,rosenfeld2020certified,weber2023rab} are designed for poisoning attacks on image data only, while we study poisoning attacks of graph injection attack. \cite{jia2023pore} is the most related to our work, but it only focuses on the recommender system. To the best of our knowledge, we are the first to study certified robustness for black-box graph models of general node classification (also suitable for recommender systems) against graph injection attack (GIA) in both \textit{evasion} and \textit{poisoning} scenarios.

\section{Limitations and Future Work}
\label{Sec:Limit}
%
This paper focuses on a provable robust framework against graph injection attacks based on randomized smoothing. Nevertheless, the drawback of randomized smoothing is the computation overload. Future work might consider extending de-randomized smoothing~\cite{levine2020randomized,levine2021improved,horvath2022randomized} to our framework to tackle the challenge of high running time. To further improve certifiable performance of randomized smoothing, there are two common strategies: improving the training process~\cite{zhai2020macer,liu2022robust,jeong2021smoothmix,gosch2023adversarial} and applying collective certification~\cite{chen2022collective,schuchardt2020collective}. The former aims to increase the intrinsic robustness of the model, while the latter further constrains the attacker to be more realistic in that it can only forge one attack sample to achieve its overall goal.  

\section{Conclusion}
\label{Sec:Conc}

This paper investigates the task of certifying graph-based classifiers against graph injection attacks (GIA). We propose a novel \textbf{node-aware bi-smoothing} scheme that provides certificates specifically designed to defend against GIAs under both evasion and poisoning threat models. Additionally, we propose a variant called \textbf{node-aware-exclude} to further enhance the certified performance against poisoning attacks. We evaluate the certified robustness of our model against GIAs on the GCN node classifier and SAR recommender system. While there is no previous work specifically addressing certifying general node classification against GIA, we generalize two certified robust models originally designed for other tasks and compare our model with them. Through extensive experiments on three datasets, we provide comprehensive benchmarks of our certified models against GIA. 
Furthermore, we evaluate the effectiveness of our model as an empirical defense method against a real GIA and compare it with four common defense models. Through extensive experiments, we demonstrate that our proposed framework not only provides significant certified robustness but also achieves competitive empirical robustness. These results demonstrate the effectiveness of our proposed model in defending against GIAs and highlight its importance in ensuring the security and robustness of graph node classification tasks.

\bibliographystyle{IEEEtran}
\bibliography{references}

\appendices

\section{Theoretical Proofs}
\label{Sec:Appendix_A}
\setcounter{thm}{0}
\begin{thm}
(Restate) Let $f:\mathbb{G}\longrightarrow \{1,\cdots, C\}^{n}$ be any graph classifier, $g$ be its smoothed classifier defined in \eqref{eqn:smooth_g} with $\phi(G)=(\phi_e(G),\phi_n(G))$, $v\in G$ be any query node, $B_{\rho,\tau}(G)$ be the node injection perturbation set defined in \eqref{eqn:pertb_ball}. Suppose $y_A, y_B\in \{1,\cdots, C\}$ and $\underline{p_A}, \overline{p_B}\in[0,1]$. Then we have 
$g_v(G')=g_v(G)$, $\forall G' \in B_{\rho,\tau}(G)$, if:
\begin{equation}
    \mu_{\rho,\tau}:=\tilde{p}(\underline{p_A}-\overline{p_B}+1)-1 >0,\nonumber 
\end{equation}
where $\tilde{p}:=(p_n+(1-p_n)(p_e+p_n-p_ep_n)^{\tau})^{\rho}$.
\end{thm}

\begin{proof}
To solve the certifying problem defined in \eqref{opt:randomsmooth}, we need to calculate the likelihood ratio of $\phi(A)$ and $\phi(A')$. Let $\Lambda(Z)=\frac{\mathbb{P}(\phi(A)=Z)}{\mathbb{P}(\phi(A')=Z)}$ be the likelihood ratio, where $Z\in \mathbb{G}$ is any possible graph produced by $\phi(A)$ or $\phi(A')$. However, the difficulty lies in that the $\phi(A)$ and $\phi(A')$ are of different dimensions, which makes the probability hard to obtain. To tackle the challenge, we propose a straightforward strategy by pre-injecting $\rho$ isolated nodes in the clean graph $A$, such that the adjacency matrix $A$ has the same dimension as $A'$. Furthermore, in order to maintain the dimension of $\phi(A)$, we construct an equivalent setting for node deletion smoothing. If a node $v$ is deleted in the smoothing $\phi(\cdot)$, we delete all the edges incident to that node and keep the isolated node $v$, which is equivalent to setting the $v^{th}$ raw $A_{v:}$ to zeros. Note that the prediction for a graph will not be affected by the isolated nodes since it does not provide any information to the existing nodes. Thus, in this setting, all the graphs involved in the computation are in the same dimension.

According to \cite{bojchevski2020efficient}, the likelihood ratio $\Lambda$ is only depends on the bits $C:=\{(i,j)| A_{ij}\neq A'_{ij}\}$ which is the set of index that $A_{ij}\neq A'_{ij}$. Under node injection perturbation, the $A'$ and $ A$ are only different among the submatrix $A_{n:(n+\rho),1:(n+\rho)}$ (the raws of injected nodes), and they have exactly $\rho \cdot \tau$ different bits. That is $|C|=\rho 
\cdot\tau$. Since the smoothing randomization $\phi$ will not add any edge, bits in $\phi(A)_{n:(n+\rho),1:(n+\rho)}$ are always zeros. So that $\Lambda(Z)=\frac{\mathbb{P}(\phi(A)=Z)}{\mathbb{P}(\phi(A')=Z)}>0$ if and only if all injected edges in $A'_{n:(n+\rho),1:(n+\rho)}$ are set to zeros by $\phi(A')$, and we define such a region as $\mathcal{R}_1=\{Z|\,Z_{n:(n+\rho),1:(n+\rho)}=0\}$, while the other region as $\mathcal{R}_2=\{Z|\,Z_{n:(n+\rho),1:(n+\rho)}\neq0\}$. We have:
\begin{align}
    \Lambda(Z)=\frac{\mathbb{P}(\phi(A)=Z)}{\mathbb{P}(\phi(A')=Z)}=\left\{
\begin{array}{cl}
1/\tilde{p}, & \text{if } Z\in\mathcal{R}_1,\\
0, & \text{if } Z\in\mathcal{R}_2. \\
\end{array}
\right. 
\end{align}

The meaning of the region $\mathcal{R}_1=\{Z|\,Z_{n:(n+\rho),1:(n+\rho)}=0\}$ is to set all the injected nodes isolated from others. If $\phi(A')$ set an injected node $\tilde{v}$ isolated, the $\phi(A')$ deleted all the incident edges to node $\tilde{v}$, or $\phi(A')$ delete the node $\tilde{v}$. For an injected node, The probability of deleting the node itself in node deletion smoothing is $p_n$. If it is not deleted, each of the edges has a probability of $p_n+(1-p_n)p_e=p_n+p_e-p_np_e$ being deleted. Because the edge connects to other existing nodes, deleting other nodes also deletes the edge connects to the node (Figure~\ref{fig:nodeaware_scheme}, bottom). If an edge is not deleted by this, it has a probability of $p_e$ being deleted in edge deletion smoothing. Since there are $\rho$ injected nodes, and $\tau$ injected edges for each injected node, the probability of $\phi(A')\in\mathcal{R}_1$ is $\tilde{p}=(p_n+(1-p_n)(p_e+p_n-p_ep_n)^{\tau})^{\rho}$. Specifically, the corresponding probabilities for the two constant likelihood ratio regions are:

$$\left\{  
\begin{array}{l}
\mathbb{P}(\phi(A)\in \mathcal{R}_1)=1,\\
\mathbb{P}(\phi(A)\in \mathcal{R}_2)=0,\\
\end{array}
\right.  
$$
$$\left\{  
\begin{array}{l}
\mathbb{P}(\phi(A')\in \mathcal{R}_1)=\tilde{p},\\
\mathbb{P}(\phi(A')\in \mathcal{R}_2)=1-\tilde{p}.
\end{array}
\right.  
$$

The worst-case classifier defined in problem \eqref{opt:randomsmooth} will assign class $y_A$ in decreasing order ($\Lambda(Z\in\mathcal{R}_1)>\Lambda(Z\in\mathcal{R}_2)$) of the constant likelihood regions until $\mathbb{P}(f_v(\phi(A))=y_A)=\underline{p_A}$, and assign class $y_B$ in increasing order ($\Lambda(Z\in\mathcal{R}_2)<\Lambda(Z\in\mathcal{R}_1)$) of the constant likelihood regions until $\mathbb{P}(f_v(\phi(A))=y_B)=\overline{p_B}$. Therefore, the worst-case classifier is: 

$$\mathbb{P}(f_v(Z)=y_A)=
\left\{  
\begin{array}{ll}
\underline{p_A}, &Z\in \mathcal{R}_1,\\
0, &Z\in \mathcal{R}_2,\\
\end{array}
\right.  
$$
$$\mathbb{P}(f_v(Z)=y_B)=
\left\{  
\begin{array}{ll}
\overline{p_B}, &Z\in \mathcal{R}_1,\\
1, &Z\in \mathcal{R}_2.\\
\end{array}
\right.  
$$

Under this classifier, we can verify that: 
\begin{align}
&\quad\,\, \mathbb{P}(f_v(\phi(A))=y_A)\nonumber\\
&=\mathbb{P}(\phi(A)=Z\in \mathcal{R}_1)\mathbb{P}(f_v(Z)=y_A|Z\in\mathcal{R}_1)
\nonumber\\
&\quad+\mathbb{P}(\phi(A)=Z\in \mathcal{R}_2)\mathbb{P}(f_v(Z)=y_A|Z\in\mathcal{R}_2)\nonumber\\
&=1\cdot \underline{p_A}+0\cdot0=\underline{p_A}.\nonumber
\end{align}
\begin{align}
&\quad\,\, \mathbb{P}(f_v(\phi(A))=y_B)\nonumber\\
&=\mathbb{P}(\phi(A)=Z\in \mathcal{R}_1)\mathbb{P}(f_v(Z)=y_B|Z\in\mathcal{R}_1)
\nonumber\\
&\quad+\mathbb{P}(\phi(A)=Z\in \mathcal{R}_2)\mathbb{P}(f_v(Z)=y_B|Z\in\mathcal{R}_2)\nonumber\\
&=1\cdot \underline{p_B}+0\cdot1=\overline{p_B}.\nonumber
\end{align}

With this worst-case classifier, we can obtain the worst-case classification margin under $\phi(A')$, which we denoted as $\mu_{\rho,\tau}$:
\begin{align}
&\quad\,\, \mathbb{P}(f_v(\phi(A'))=y_A)\nonumber\\
&=\mathbb{P}(\phi(A')=Z\in \mathcal{R}_1)\mathbb{P}(f_v(Z)=y_A|Z\in\mathcal{R}_1)
\nonumber\\
&\quad+\mathbb{P}(\phi(A')=Z\in \mathcal{R}_2)\mathbb{P}(f_v(Z)=y_A|Z\in\mathcal{R}_2)\nonumber\\
&=\tilde{p}\cdot \underline{p_A}+(1-\tilde{p})\cdot0\nonumber\\
&=\tilde{p}\:\underline{p_A}.\nonumber
\end{align}

\begin{align}
&\quad\,\, \mathbb{P}(f_v(\phi(A'))=y_B)\nonumber\\
&=\mathbb{P}(\phi(A')=Z\in \mathcal{R}_1)\mathbb{P}(f_v(Z)=y_B|Z\in\mathcal{R}_1)
\nonumber\\
&\quad+\mathbb{P}(\phi(A')=Z\in \mathcal{R}_2)\mathbb{P}(f_v(Z)=y_B|Z\in\mathcal{R}_2)\nonumber\\
&=\tilde{p}\:\overline{p_B}+1-\tilde{p}.\nonumber
\end{align}
\begin{align}
    \mu_{\rho,\tau}&=\mathbb{P}(f_v(\phi(A'))=y_A)-\mathbb{P}(f_v(\phi(A'))=y_B)\nonumber\\
    &=\tilde{p}\:\underline{p_A}-(\tilde{p}\:\overline{p_B}+1-\tilde{p})\nonumber\\
    &=\tilde{p}(\underline{p_A}-\overline{p_B}+1)-1.\nonumber
\end{align}
If $\mu_{\rho,\tau}>0$, we can certify the prediction of $g_v(A')=y_A$. Otherwise, we cannot certify the prediction of $g_v(A')$. 
\end{proof}

\begin{thm}
(Restate) Let $f:\mathbb{G}\longrightarrow \{1,\cdots, C\}^{n}$ be any graph classifier, $g$ be its smoothed classifier defined in \eqref{eqn:smooth_g_exclude} with $\phi(G)=(\phi_e(G),\phi_n(G))$, $v\in G$ be any query node, $B_{\rho,\tau}(G)$ be the node injection perturbation set defined in \eqref{eqn:pertb_ball}, and the attack edges added to a node $v$ should not exceed its original degree $d(v)$. Suppose $y_A, y_B\in \{1,\cdots, C\}$ and $\underline{p_A}, \overline{p_B}\in[0,1]$. Then we have 
$g_v(G')=g_v(G)$, $\forall G' \in B_{\rho,\tau}(G)$, if:
\begin{align}
\mu_{\rho,\tau}&:=\tilde{p}(\underline{p_A}-\frac{(1-\underline{p'_0})\overline{p_B}}{(1-p_0)}+1-\underline{p'_0})-(1-\underline{p'_0})>0, \nonumber
\end{align}
where $\tilde{p}:=(p_n+(1-p_n)(p_e+p_n-p_ep_n)^{\tau})^{\rho}$, $d(v)$ denotes the degree of node $v$, and $p_0:=p_n+(1-p_n)(p_e+p_n-p_ep_n)^{d(v)}$ is the probability that the node $v$ is deleted by the smoothing $\phi(G)$, $\underline{p'_0}:=p_n+(1-p_n)(p_e+p_n-p_ep_n)^{2d(v)}$.  
\end{thm}

\begin{proof}
Let $Z$ be any possible graph from $\phi(G)$ or $\phi(G')$,  $v\vdash Z$ denote  $v\in Z$  is not isolated, we now need to compute the likelihood ratio with $v\vdash \phi(G)$:
\begin{align}
\label{eqn:likelihood_ratio_exclude}
    \Lambda(Z)&=\frac{\mathbb{P}(\phi(A)=Z, v\vdash Z)}{\mathbb{P}(\phi(A')=Z, v\vdash Z)}\\
&=\left\{
\begin{array}{cl}
\frac{1-p_0}{\tilde{p}(1-p'_0)}, & \text{if } Z\in\mathcal{R}_1 \,(Z_{n:(n+\rho),1:(n+\rho)}=0),\\
0, & \text{if } Z\in\mathcal{R}_2 \,(Z_{n:(n+\rho),1:(n+\rho)}\neq0). \nonumber\\
\end{array}
\right. 
\end{align}    
Specifically, the corresponding probabilities for the two constant likelihood ratio regions are:

$$\left\{  
\begin{array}{l}
\mathbb{P}(\phi(A)\in \mathcal{R}_1,v\vdash \phi(A))=1-p_0,\\
\mathbb{P}(\phi(A)\in \mathcal{R}_2,v\vdash \phi(A))=0,\\
\end{array}
\right.  
$$
$$\left\{  
\begin{array}{l}
\mathbb{P}(\phi(A')\in \mathcal{R}_1,v\vdash \phi(A'))=\tilde{p}(1-p'_0),\\
\mathbb{P}(\phi(A')\in \mathcal{R}_2,v\vdash \phi(A'))=(1-\tilde{p})(1-p'_0),
\end{array}
\right.  
$$
where $p_0:=p_n+(1-p_n)(p_e+p_n-p_ep_n)^{d(v)}$, $p'_0=p_n+(1-p_n)(p_e+p_n-p_ep_n)^{d(v')}$ denotes the probability that the node $v$ is deleted by the smoothing $\phi(G)$ and $\phi(G')$ respectively; $d(v)$, $d(v')$ denotes the degree of node $v$ in $G$ and $G'$, respectively. Similarly, the worst-case classifier is: 
$$\mathbb{P}(f_v(Z)=y_A)=
\left\{  
\begin{array}{ll}
\frac{\underline{p_A}}{(1-p_0)}, &Z\in \mathcal{R}_1, v\vdash Z,\\
0, &Z\in \mathcal{R}_2, v\vdash Z,\\
\end{array}
\right.  
$$
$$\mathbb{P}(f_v(Z)=y_B)=
\left\{  
\begin{array}{ll}
\frac{\overline{p_B}}{(1-p_0)}, &Z\in \mathcal{R}_1, v\vdash Z,\\
1, &Z\in \mathcal{R}_2, v\vdash Z.\\
\end{array}
\right.  
$$


With this worst-case classifier, we can obtain the worst-case classification margin under $\phi(A')$, which we denoted as $\mu_{\rho,\tau}$:
\begin{align}
&\quad\,\, \mathbb{P}(f_v(\phi(A'))=y_A)\nonumber\\
&=\mathbb{P}(f_v(\phi(A'))=y_A,v\vdash Z)\nonumber\\
&=\mathbb{P}(\phi(A')=Z \in \mathcal{R}_1,v\vdash Z)\nonumber\\
&\quad\times \mathbb{P}(f_v(Z)=y_A|Z\in\mathcal{R}_1,v\vdash Z)\nonumber\\
&\quad+\mathbb{P}(\phi(A')=Z\in \mathcal{R}_2,v\vdash Z)\nonumber\\
&\quad\times\mathbb{P}(f_v(Z)=y_A|Z\in\mathcal{R}_2,v\vdash Z)\nonumber\\
&=\tilde{p}(1-p'_0)\cdot \frac{\underline{p_A}}{(1-p_0)}\nonumber\\
&\geq \tilde{p}\cdot\underline{p_A},\nonumber
\end{align}
where the inequality is due to $\frac{(1-p'_0)}{(1-p_0)}\geq 1$, because the node degree of $v$ in the perturbed graph must be larger than in the clean graph: $d(v')\geq d(v)$. With the assumption that $d(v')\leq 2d(v)$, we have $p'_0>\underline{p'_0}:=p_n+(1-p_n)(p_e+p_n-p_ep_n)^{2d(v)}$, and $(1-p'_0)\leq (1-\underline{p'_0})$:
\begin{align}
&\quad\,\, \mathbb{P}(f_v(\phi(A'))=y_B)\nonumber\\
&=\mathbb{P}(\phi(A')=Z\in \mathcal{R}_1,v\vdash Z)\nonumber\\
&\quad\times \mathbb{P}(f_v(Z)=y_B|Z\in\mathcal{R}_1,v\vdash Z)\nonumber\\
&\quad+\mathbb{P}(\phi(A')=Z\in \mathcal{R}_2,v\vdash Z)\nonumber\\
&\quad\times \mathbb{P}(f_v(Z)=y_B|Z\in\mathcal{R}_2,v\vdash Z)\nonumber\\
&=\tilde{p}(1-p'_0)\cdot\frac{\overline{p_B}}{(1-p_0)}+(1-\tilde{p})(1-p'_0)\nonumber\\
&\leq \tilde{p}\cdot \overline{p_B}\cdot\frac{(1-\underline{p'_0})}{(1-p_0)}+(1-\tilde{p})(1-\underline{p'_0}),\nonumber
\end{align}
where the inequality is due to $(1-p'_0)\leq (1-\underline{p'_0})$. Then, we can obtain a lower bound of the worst-case classification margin under $\phi(A')$:
\begin{align}
    &\quad \mathbb{P}(f_v(\phi(A'))=y_A)-\mathbb{P}(f_v(\phi(A'))=y_B)\nonumber\\
    &\geq \tilde{p}\cdot\underline{p_A}-[\tilde{p}\cdot \overline{p_B}\cdot\frac{(1-p_n)}{(1-p_0)}+(1-\tilde{p})(1-\underline{p'_0})]\nonumber\\
    &= \tilde{p}(\underline{p_A}-\overline{p_B}\cdot\frac{(1-\underline{p'_0})}{(1-p_0)})-1+\underline{p'_0}+\tilde{p}-\tilde{p}\underline{p'_0}\nonumber\\
    &= \tilde{p}(\underline{p_A}-\frac{(1-\underline{p'_0})\overline{p_B}}{(1-p_0)}+1-\underline{p'_0})-(1-\underline{p'_0}):=\mu_{\rho,\tau}.\nonumber
\end{align}
If $\mu_{\rho,\tau}>0$, we can certify the prediction of $g_v(A')=y_A$. Otherwise, we cannot certify the prediction of $g_v(A')$. 
\end{proof}

\begin{thm}
(Restate) Let $F_u(G)$ be any base recommender system trained on $G$ and recommend $K'$ items to the user $u$, $g_u(G)$ be its smoothed recommender defined in \eqref{eqn:smooth_g_RS}, $u\in G$ be any query user, $B_{\rho,\tau}(G)$ be the node injection perturbation set defined in \eqref{eqn:pertb_ball}, and the attack edges added to a node $v$ should not exceed its original degree $d(v)$. Then, we have at least $r$ recommended items after poisoning are overlapped with ground truth items $I_u$: $|g_u(G') \cap I_u|\geq r, \forall G'\in B_{\rho,\tau}(G)$ if:
    \begin{equation}
        \hat{p}\,\underline{p_r}-\min_{H_c}(\overline{p}_{H_c}+K'(1-\hat{p})(1-p_0))/c>0,\nonumber
    \end{equation}
    where $\hat{p}:=(p_n+(1-p_n)p_e^{\tau})^{\rho}$, $\underline{p_r}$ is the lower bound of the $r$th largest item probability among $\{p_{u,i}|i\in I_u\}$, $H_c$ denote any subset of the top-$(K-r+1)$ largest items among $I\setminus I_u$ with size $c$, $\overline{p}_{H_c}:=\sum_{j\in H_c} \overline{p_{u,j}}$ is the sum of probability upper bounds for $c$ items in $H_c$, $p_0:=p_n+(1-p_n)(p_e)^{d(u)}$ is the probability that the user $u$ is deleted by the smoothing $\phi(G)$, $d(u)$ is the number of user ratings in training set.
\end{thm}

\begin{proof}
For a user-item interaction graph $G$, we represent it as a user-item interaction matrix denoted by $A$, where each row represents a user and each column represents an item, and the elements $A_{ui}=1$ if the interaction exists between user $u$ and item $i$, otherwise $A_{ui}=0$. If there are $n$ users and $m$ items in the training set, we have the shape of $A$ of $n\times m$. Let $I$ denote all the items in the training set, we have $|I|=m$. Let $Z$ be any possible matrix from $\phi(A)$ or $\phi(A')$,  $u\vdash Z$ denote  $u\in Z$ has at least one rating, we have the likelihood ratio with $u\vdash \phi(G)$:
\begin{align}
    \Lambda(Z)&=\frac{\mathbb{P}(\phi(A)=Z, u\vdash Z)}{\mathbb{P}(\phi(A')=Z, u\vdash Z)}\\
&=\left\{
\begin{array}{cl}
1/\hat{p}, & \text{if } Z\in\mathcal{R}_1 \,(Z_{n:(n+\rho),1:m}=0),\\
0, & \text{if } Z\in\mathcal{R}_2 \,(Z_{n:(n+\rho)),1:m}\neq0). \nonumber\\
\end{array}
\right. 
\end{align} 

Specifically, the corresponding probabilities for the two constant likelihood ratio regions are:

$$\left\{  
\begin{array}{l}
\mathbb{P}(\phi(A)\in \mathcal{R}_1,u\vdash \phi(A))=1-p_0,\\
\mathbb{P}(\phi(A)\in \mathcal{R}_2,u\vdash \phi(A))=0,\\
\end{array}
\right.  
$$
$$\left\{  
\begin{array}{l}
\mathbb{P}(\phi(A')\in \mathcal{R}_1,u\vdash \phi(A'))=\hat{p}(1-p_0),\\
\mathbb{P}(\phi(A')\in \mathcal{R}_2,u\vdash \phi(A'))=(1-\hat{p})(1-p_0),
\end{array}
\right.  
$$
where $p_0:=p_n+(1-p_n)(p_e)^{d(u)}$ denotes the probability that the node $u$ is deleted by the smoothing $\phi(G)$, and $d(u)$ denotes the degree (number of rating) of user $u$ in $G$. Note that the newly injected user will not increase the degree of the existing user, the $\phi(A')$, so that $\mathbb{P}(u\vdash \phi(A'))$ is also $(1-p_0)$. 
Similarly, if we know the lower bound of $p_{u,i}:=\mathbb{P}(i\in F_u(\phi(G)))$, then the worst-case classifier returns the smallest $p'_{u,i}:=\mathbb{P}(i\in F_u(\phi(G')))$ is: 
$$\mathbb{P}(i\in F_u(Z))=
\left\{  
\begin{array}{ll}
\frac{\underline{p_{u,i}}}{(1-p_0)}, &Z\in \mathcal{R}_1, u\vdash Z,\\
0, &Z\in \mathcal{R}_2, u\vdash Z,\\
\end{array}
\right.  
$$
On the contrary, if we know the upper bound of $p_{u,j}:=\mathbb{P}(j\in F_u(\phi(G)))$, then the worst-case classifier returns the largest $p'_{u,j}:=\mathbb{P}(j\in F_u(\phi(G')))$ is:
$$\mathbb{P}(j\in F_u(Z))=
\left\{  
\begin{array}{ll}
\frac{\overline{p_{u,j}}}{(1-p_0)}, &Z\in \mathcal{R}_1, u\vdash Z,\\
1, &Z\in \mathcal{R}_2, u\vdash Z.\\
\end{array}
\right.  
$$

For the $r$th items under the clean graph, we denote its lower bound of probability as $\underline{p_r}$, and then we have the bound for its probability under the poisoned graph: $p'_r\geq \hat{p}\,\underline{p_r}$. We have at least $r$ recommended items overlapped with ground truth items $I_u$ if the $r$th largest item probability among items $I_u$ is larger than the $(K-r+1)$th largest items probability among $I\setminus I_u$ under the poisoned graph.
We denote the top $(K-r+1)$ largest items by their probability among $I\setminus I_u$ as a set $\mathbf{I}_{kr}$.
Following \cite{jia2023pore}, instead of considering the $(K-r+1)$th item (the smallest one in $\mathbf{I}_{kr}$), jointly considering multiple items $H_c$ usually leads to a smaller upper bound, where $H_c$ is the subset of $\mathbf{I}_{kr}$ with size $c$. For the $c$ items, we know its summation of upper bound: $\overline{p}_{H_c}:=\sum_{j\in H_c} \overline{p_{u,j}}$. Because each of the system recommender $K'$ items, we have 
$\overline{p}_{H_c}\leq K'$ (i.e., $\frac{\overline{p}_{H_c}}{K'}\leq1$). 
Then we have the upper bound for $p'_{H_c}$:

\begin{align}
    \quad p'_{H_c}=\sum_{j\in H_c} p'_{u,j}\leq \hat{p}\cdot \overline{p}_{H_c}+ K'(1-\hat{p})(1-p_0).
\end{align}

Then, because the minimum value of a set is always smaller than the average value, we have:

\begin{align}
    \min_{j\in\mathbf{I}_{kr}} p'_{u,j}
    &\leq \min_{j\in H_c} p'_{u,j}
    \leq \frac{\sum_{j\in H_c} p'_{u,j}}{c} \nonumber\\
    &\leq (\hat{p}\cdot \overline{p}_{H_c}+ K'(1-\hat{p})(1-p_0))/c.
\end{align}
Finally, we have at least $r$ recommended items overlapped with ground truth items $I_u$ if: $\hat{p}\,\underline{p_r}-\min_{H_c}(\overline{p}_{H_c}+K'(1-\hat{p})(1-p_0))/c>0$.

\end{proof}

\section{Other Experimental Results}
In this section, we display the supplemental experimental results. Figure.~\ref{fig:clean_acc} shows the clean accuracy of the smoothed classifier under various smoothing parameters $p_e$ and $p_n$. For the evasion attack, the Multi-layer Perceptron (MLP) will never affected because it does not rely on the graph structure for prediction. As a result, a graph model with a lower accuracy is not meaningful. In the Cora-ML dataset, the clean accuracy of our smoothed classifiers is significantly higher than $0.691$, which are all effective. However, due to the smaller average degree of the Citeseer dataset, our smoothed classifiers might have a lower clean accuracy than MLP. We exclude the parameters (shadow) that lead to lower accuracy than the MLP model, which is $0.691$ on Cora-ML
and $0.660$ on the Citeseer dataset. 

Note that the MLP model is also subject to poisoning GIA since the malicious node feature can be crafted arbitrarily (all the results are effective). Notably, our smoothed model in the poisoning setting improves the clean accuracy.

\setcounter{figure}{0}
\renewcommand\thefigure{\Alph{section}.\arabic{figure}}

\label{Sec:Appendix_B}
\begin{figure}[hbt!]
\centering
    \subfigure[Evasion]{\includegraphics[width=0.222\textwidth,height=3.5cm]{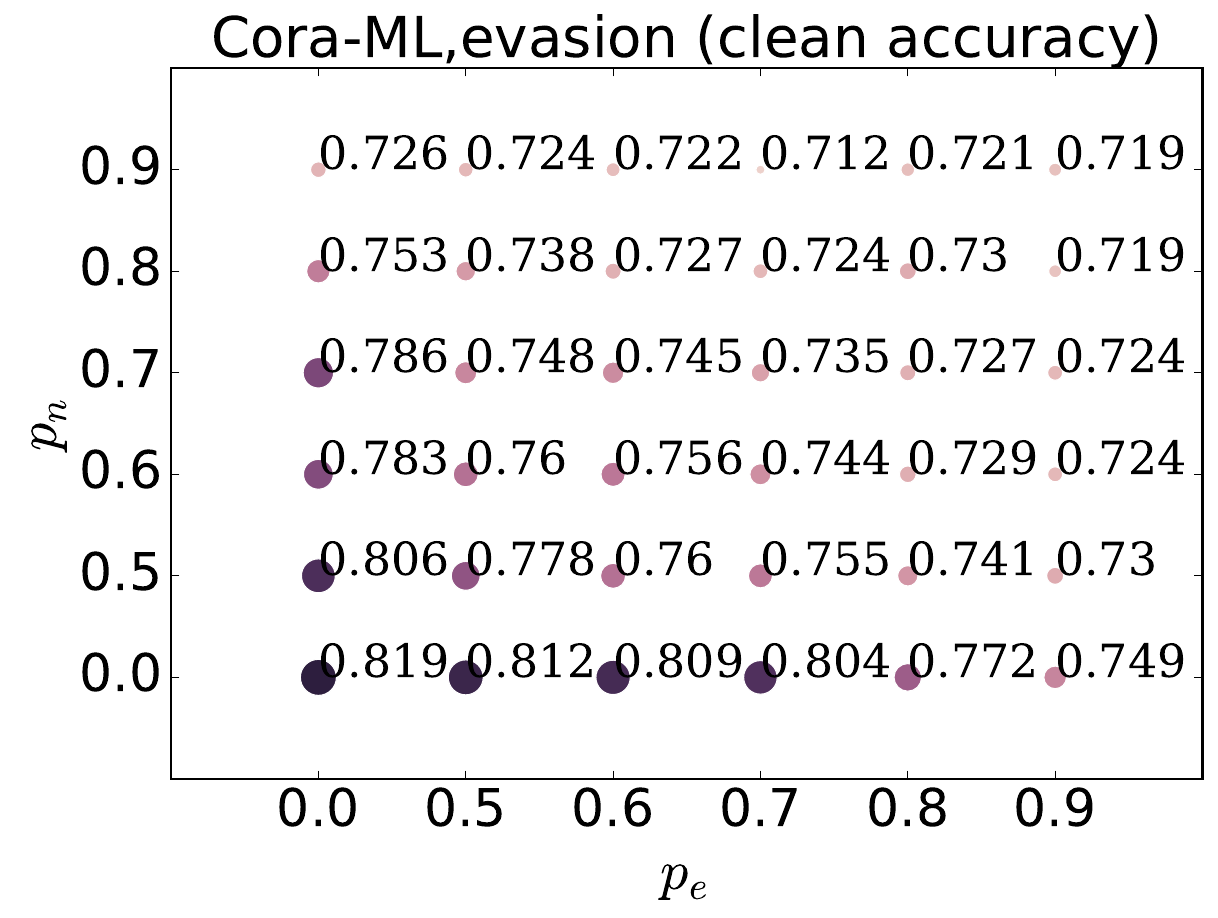}
    \label{fig_second_case}
    }
    \subfigure[Evasion]{\includegraphics[width=0.222\textwidth,height=3.5cm]{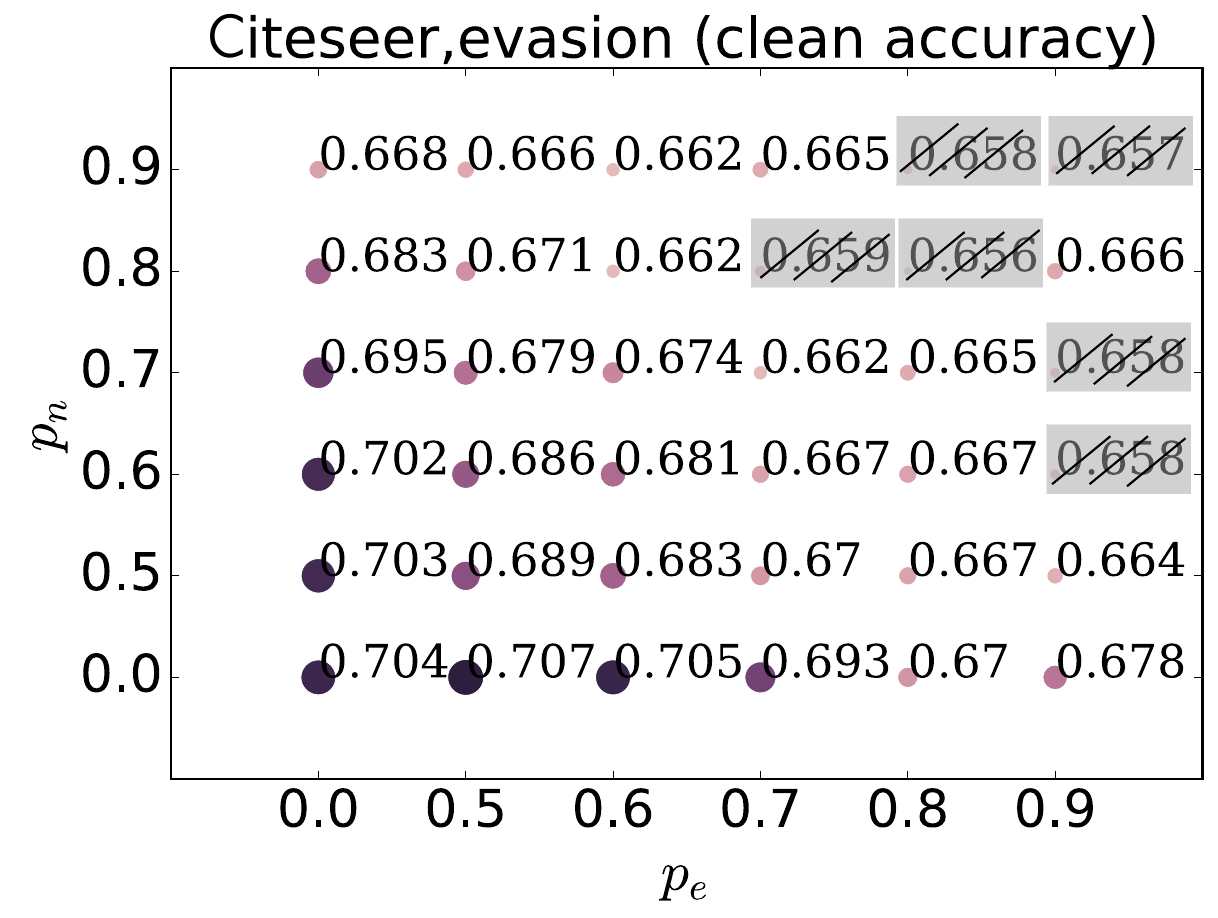}
    }
    \subfigure[Poisoning (include)]{\includegraphics[width=0.222\textwidth,height=3.5cm]{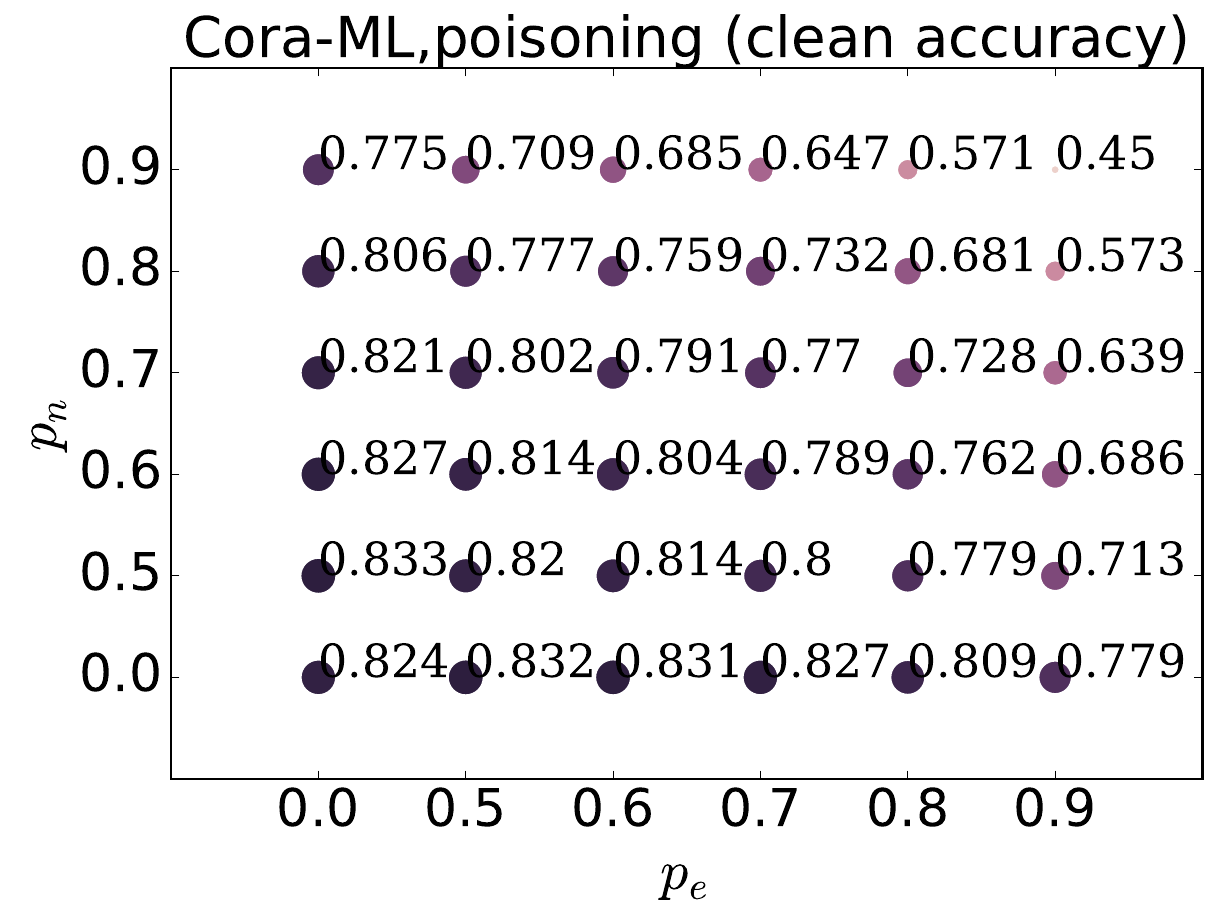}
    }
    \subfigure[Poisoning (include)]{\includegraphics[width=0.222\textwidth,height=3.5cm]{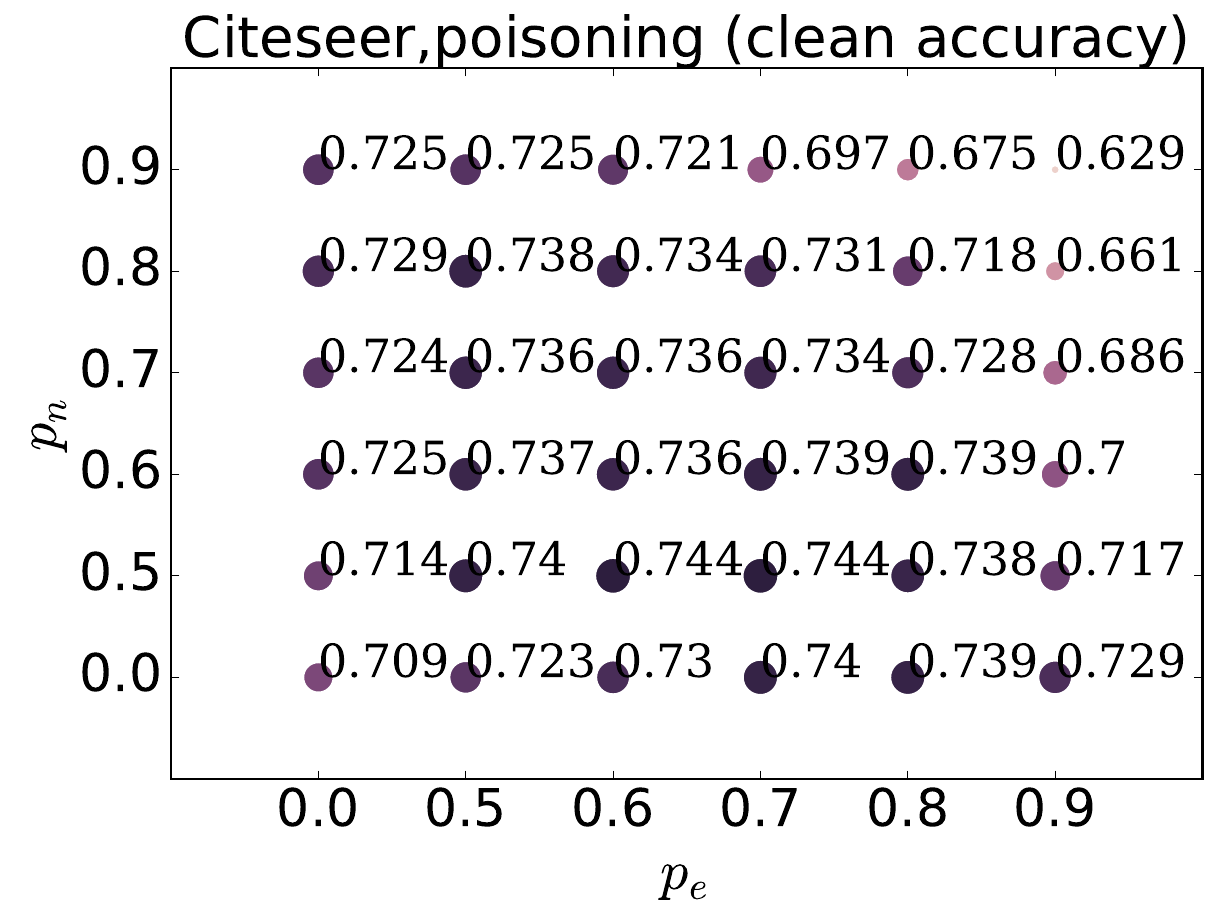}
    }
    \subfigure[Poisoning (exclude)]{\includegraphics[width=0.222\textwidth,height=3.5cm]{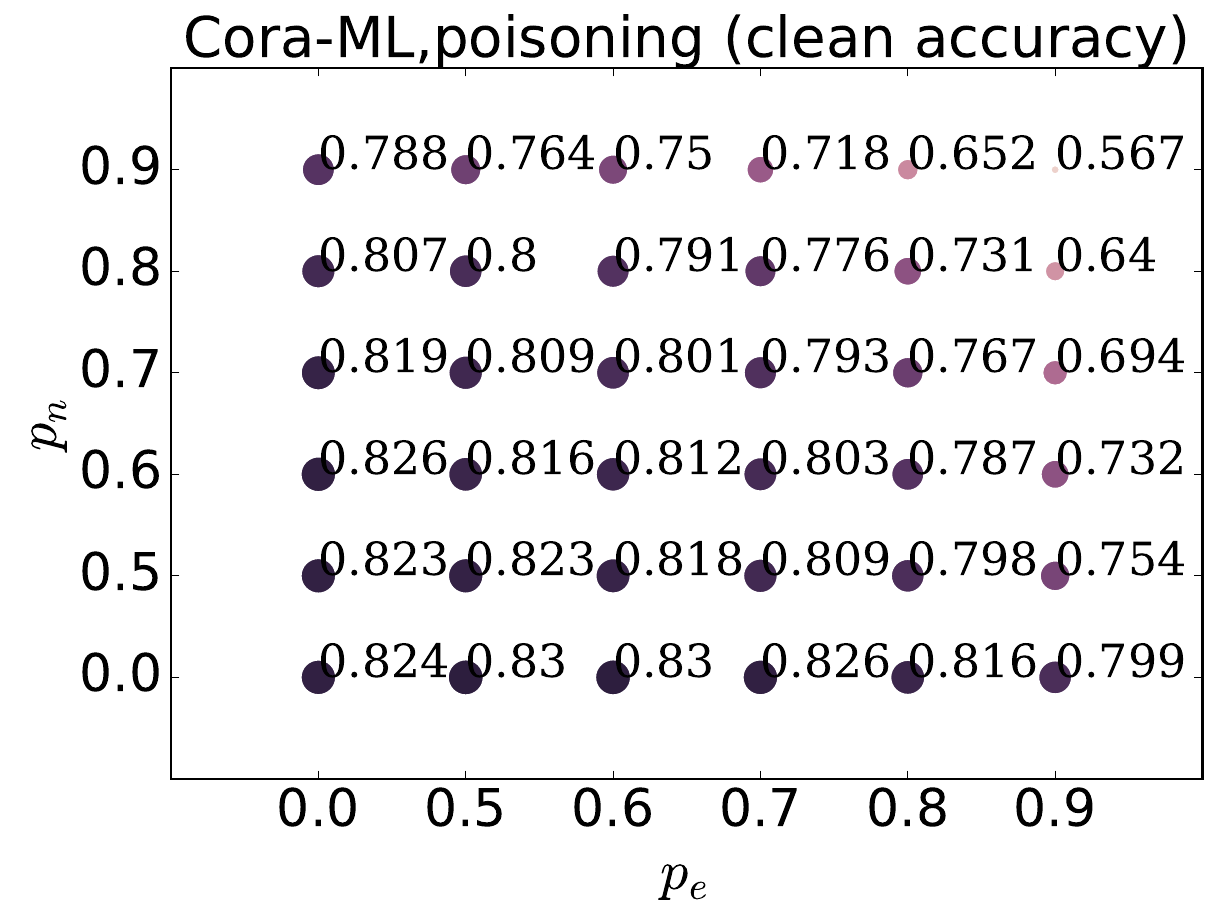}
    }
    \subfigure[Poisoning (exclude)]{\includegraphics[width=0.222\textwidth,height=3.5cm]{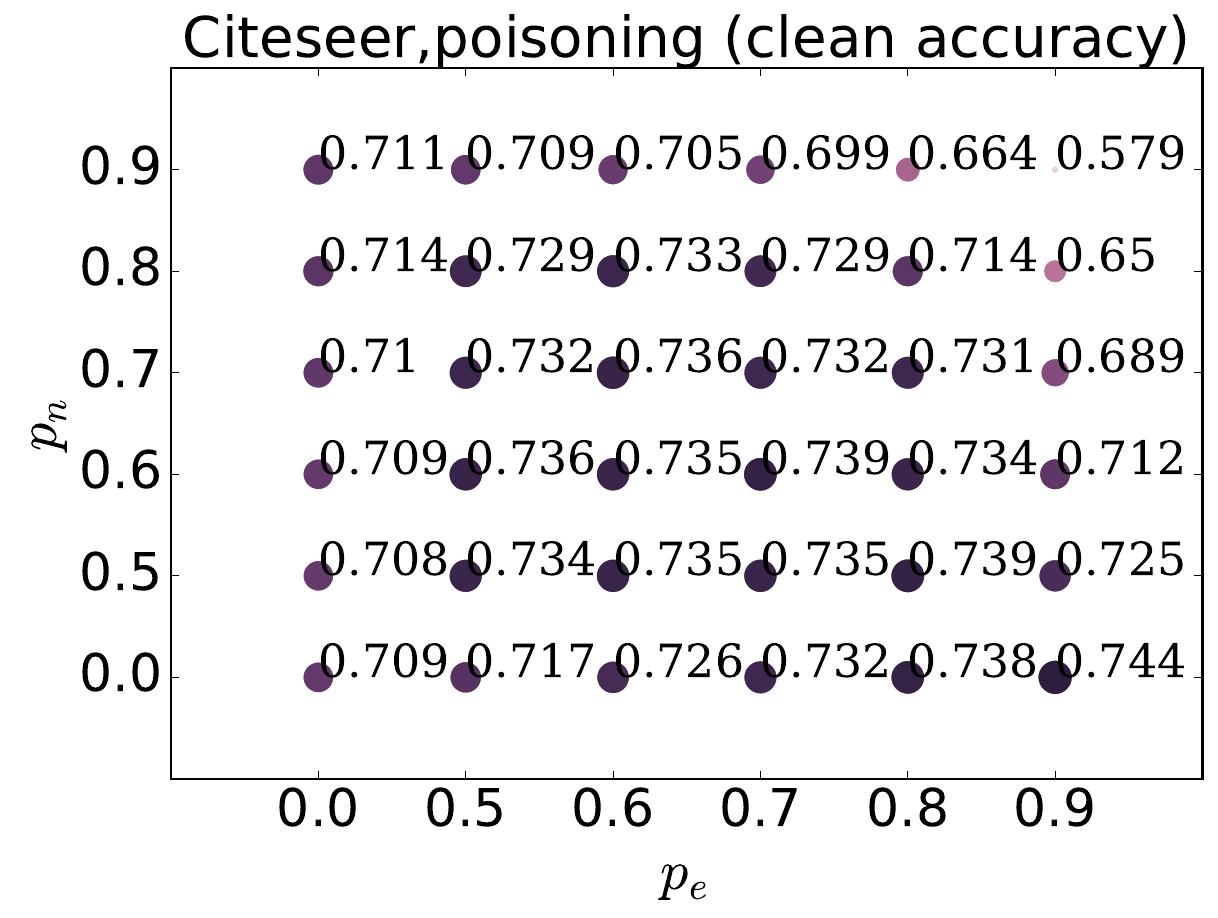}
    }
\caption{Clean accuracy of node-aware bi-smoothing classifiers with various parameters under evasion and poisoning setting.}
\label{fig:clean_acc}
\end{figure}

\end{document}